\documentclass[10pt]{article}

\usepackage{amsmath,amsthm,verbatim,amssymb,amsfonts,amscd, graphicx, enumitem, times}
\usepackage{graphics}
\usepackage{centernot}
\usepackage{authblk}
\theoremstyle{plain}
\newtheorem{theorem}{Theorem}
\newtheorem{corollary}{Corollary}

\newtheorem*{remark}{Remark}
\newtheorem{proposition}{Proposition}

\newtheorem{definition}{Definition}
\newtheorem{assumption}{Assumption} 
\newtheorem{property}{Property}

\usepackage{booktabs}
\usepackage{longtable}
\usepackage{array}

\usepackage[colorlinks,linkcolor=blue,citecolor=blue]{hyperref}

\usepackage[bottom]{footmisc}
\usepackage{caption}
\usepackage{subcaption}
\usepackage{helvet}  %Required
\usepackage{courier}  %Required
\usepackage{url}  %Required
\usepackage{graphicx}  %Required
\usepackage{multirow}
\usepackage{amsthm}
\usepackage{color}
\usepackage{MnSymbol}
\usepackage{makecell}
\usepackage{arydshln}
\usepackage{amsmath}
\usepackage[dvipsnames]{xcolor}
\usepackage{caption} 
\usepackage{natbib}
\usepackage{bbm}

\usepackage{textcomp}
\usepackage{wrapfig}
\usepackage{algorithm}
\usepackage{algorithmic}

\usepackage{csquotes}

\newcommand{\Tr}{{\rm Tr}}

\newcommand{\rr}{\raggedright\let\\\tabularnewline}

\begin{document}
%\[ \fbox{$\Box$} \fbox{$\hat\Box$} \fbox{$\tilde\Box$} \]

\title{Neural Quadratic Forms:\\
A Unified Minimal Model for Sudden Learning and Scaling Laws}
\author{Liu Ziyin$^{1,}$\thanks{These two authors contributed equally.}\ \ , Yizhou Xu$^{2,*}$, Tomaso Poggio$^1$, Isaac Chuang$^1$\\
$^1$\textit{Massachusetts Institute of Technology}\\
$^2$\textit{École Polytechnique Fédérale de Lausanne}%\\
%$^3$\textit{NTT Research}
}
\maketitle

\begin{abstract}
Neural networks trained by gradient descent on a smooth cost function can nevertheless learn in steps: the cost holds on long plateaus and then drops abruptly. Meanwhile, training losses instead follow smooth power laws. Variants of both behaviors occur in architectures with very different microscopic structures, which is the signature of a few relevant collective variables. We show that a symmetry fixes what those variables are: a network layer is a sum over interchangeable units, so relabeling the units leaves it unchanged; given smoothness and the condition that a unit's gradient vanish at the origin, symmetry then enforces a universal leading form for the expansion about the near-zero weights present at the start of training, the quadratic $\Tr[WW^{\top}A(x)]$, in which every architectural detail is confined to a single ``structure matrix" $A(x)$ that we compute for each architecture. Perceptrons, attention layers, mixtures of experts, and convolutions become one model at different $A$. Its training dynamics then close on the ``order parameter" $M=WW^{\top}$ and, whenever the data matrices share an eigenbasis, reduce to a Lotka--Volterra equation whose modes switch on one after another. The smaller the initial weights, the further apart the switch-on times, and the plateaus appear as a singular limit of a smooth flow; when many modes are unresolved the same events merge into a power law in training time whose exponent the theory predicts. We confirm both numerically across training methods and architectures.
\end{abstract}

%\begin{enumerate}[noitemsep,topsep=0pt, parsep=0pt,partopsep=0pt, leftmargin=13pt]
%    \item ..
%\end{enumerate}

%\section{Introduction}

\section{Introduction}

Neural networks exhibit two apparently incompatible forms of regularity. Along a single training trajectory, the loss can remain nearly constant for long intervals and then fall abruptly as new features are acquired. Across model size, dataset size, or training compute, by contrast, losses often follow smooth power laws over many orders of magnitude, with weak sensitivity to architectural detail \cite{kaplan2020scaling,bahri2024explaining,maloney2022solvable}. These cross-scale laws are accurate enough to guide large-scale design decisions: compute-optimal scaling can favor a smaller model trained on substantially more data over a much larger, undertrained one \cite{hoffmann2022training}. Variants of both behaviors occur across architectures with very different microscopic structure, including multilayer perceptrons, convolutional networks, mixtures of experts, and transformers \cite{saxe2014exact,liu2022towards,zhang2025saddle,kunin2026alternating,huh2024platonic,kaushik2025universalweightsubspacehypothesis}. Sharp microscopic events together with macroscopic laws that are indifferent to microscopic detail are the familiar signature of a small number of relevant collective variables, and they invite the construction of a minimal model.

To a physicist, these behaviors are not unique to neural networks. Smooth deterministic dynamics can produce sharp, reproducible transitions when different modes cross instability thresholds one at a time, as at the onset of convection in a fluid. Long plateaus followed by rapid changes resemble the induction periods of autocatalytic reactions. Moreover, a transition may remain smooth for every finite value of a parameter but become sharp in a limiting regime, just as phase transitions emerge only in the thermodynamic limit. Finally, the weak dependence on microscopic details suggests universality: the large-scale behavior is governed by a small number of collective variables.

These physical phenomena can be understood by finding the collective variables and the equation they obey. Such a search naturally begins with a symmetry, and neural architectures supply an unusually clean one. Neural networks are canonically assembled from repeated, exchangeable components: hidden units in a perceptron, heads in an attention layer, experts in a mixture, channels in a convolution. Relabeling those components leaves the represented function unchanged, so a module of width $d$ is invariant under the symmetric group $S_d$ acting on its components, in the same way that a collection of identical particles is invariant under exchange. The analogy is more than verbal: the training of two-layer networks has been mapped directly onto the mean-field dynamics of interacting particles \cite{mei2019mean}, permutation symmetry has been proposed as the operational definition of a neuron \cite{ziyin2024symmetry}, and the breaking and restoration of parameter symmetries has been argued to organize the order in which features are acquired \cite{ziyin2024parameter,ziyin2025parameter}.

Symmetry of this kind is the standard point of departure for a Landau construction \cite{landau2013statistical,goldenfeld1992lectures,chaikin1995principles}: identify the symmetry group, identify a small parameter that holds the system near a distinguished point of configuration space, expand about that point, and retain only the lowest-order terms the symmetry permits. Here the group is $S_d$, the distinguished point is the origin of parameter space, and the small parameter is the initialization scale $\epsilon$. Microscopic detail should then survive only in the phenomenological coefficients of the expansion, which is the sense in which universality across architectures is to be expected and which is the logic underlying the statistical-mechanical tradition in learning theory \cite{seung1992statistical,engel2001statistical,zdeborova2016statistical,bahri2020statistical,carleo2019machine}.

\begin{figure}
    \centering
    \includegraphics[width=1\linewidth]{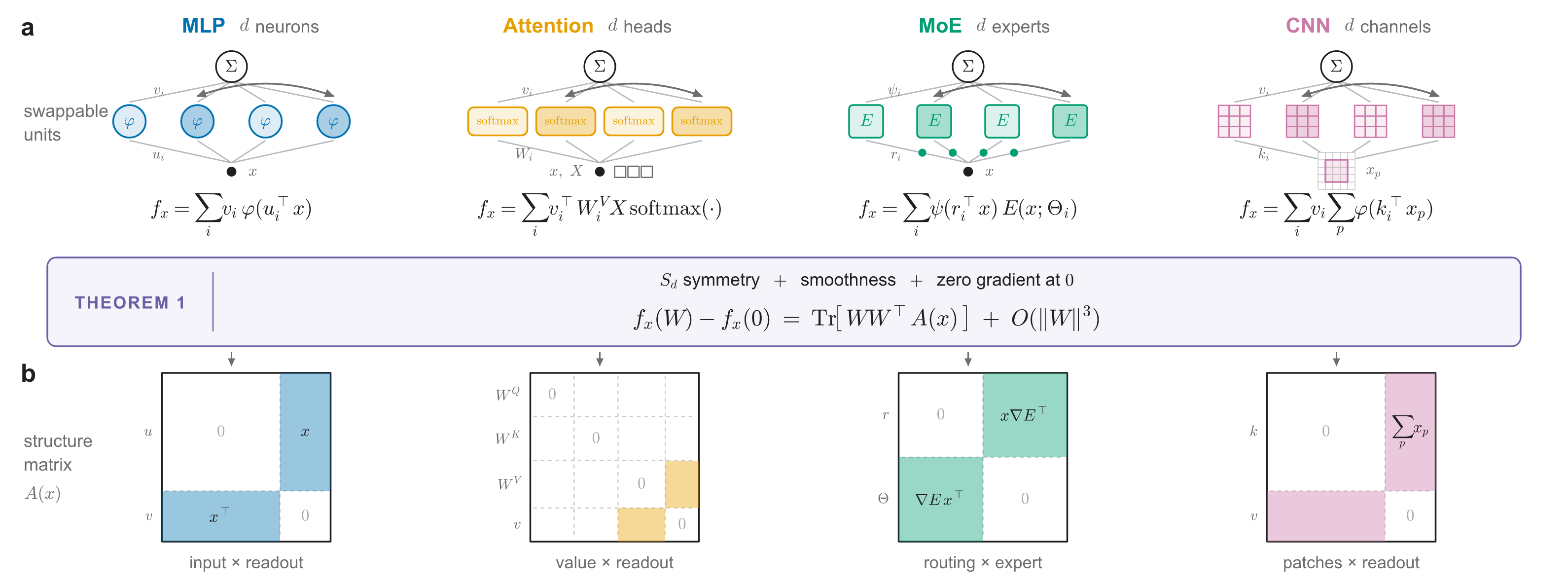}
    \caption{\textbf{Permutation symmetry of the units allows us to characterize the model with structure matrices $A$.} \textbf{(a)} Perceptrons, attention layers, mixtures of experts and convolutions are all sums over interchangeable components, so a permutation leaves the represented function unchanged.
  \textbf{(b)} Training begins with all weights near zero, so each module may be Taylor expanded about that point. Given smoothness and the condition that a component's gradient vanish whenever that component's own weights do, symmetry then enforces the same leading quadratic form in every case (Theorem~\ref{theo:main}), and the architecture survives only in the structure matrix $A(x)$. Here $x$ is a single input to the module.}
    \label{fig:overview}
\end{figure}

Carrying out this program produces a single expression, summarized in Figure~\ref{fig:overview}. We prove (Theorem~\ref{theo:main}) that any module $f_x(w_1,\dots,w_d)$ that is $S_d$-symmetric in its components [Figure~\ref{fig:overview}(a)], is three times continuously differentiable, and satisfies a gradient condition, stated precisely below, obeys
\begin{equation}
f_x(W)-f_x(0)= \mu^\top  g(x) + \Tr\!\left[WW^\top A(x)\right]+O\left (\|W\|^3 \right ).
\label{eq:nqf-intro}
\end{equation}
Here $x$ is a single input, $W=(w_1,\dots,w_d)\in\mathbb{R}^{p\times d}$ collects the $d$ components,  each carrying $p$ parameters, $\mu = \sum_{i=1}^d w_i$, and the symmetric matrix $A(x)\in\mathbb{R}^{p\times p}$ is fixed by the architecture and by $x$ alone; we call Eq.~\eqref{eq:nqf-intro} the neural quadratic form (NQF) and $A(x)$ the structure matrix. It turns out that the term $\mu^\top g(x)$ vanishes for most of the practical architectures, and so the NQF can often be written with only the second-moment term.

Everything architecture-specific now sits in $A(x)$, as a sparsity pattern and a set of couplings [Figure~\ref{fig:overview}(b)]. Attention is the striking case: only its value and readout blocks appear at this order, the query and key matrices entering first at quartic order. What controls the model is likewise not the individual components but the order parameter $M=WW^\top$, whose size is set by $p$.

%The reach of the result is limited in two ways, one by the hypotheses and one by the expansion. The gradient condition is stronger than requiring the gradient to vanish at the origin, and has to be: the weaker condition leaves behind a term coupling distinct components. The expansion, for its part, is local, and so is the theory built on it, which will need modification for biases, nonsmooth activations, modules that apply a nonlinearity after aggregating their components, or far-from-origin initialization. We verify Eq.~\eqref{eq:nqf-intro} numerically for perceptrons, convolutions and attention under three optimizers, and confirm that it fails, as it must, once the initialization is large.

Quadratic parameterizations are not themselves new as proxies for nonconvex learning \cite{gunasekar2017implicit,arora2019implicit,candes2015phase,berthier2023incremental,maillard2024bayes,defilippis2025scaling}. Matrix sensing, phase retrieval, diagonal networks and quadratic networks have each been analyzed on their own terms, to explain implicit low-rank bias, spectral initialization, or feature growth; exact analyses of deep linear networks display the same mode-wise plateaus under gradient descent \cite{saxe2014exact,advani2020high}; and recent work argues that quadratic models remain quantitatively accurate even for large language models \cite{meterez2026defensequadraticmodel}. What Eq.~\eqref{eq:nqf-intro} adds is that these are not a family of analogous models but one model, and that the coupling $A(x)$ is computed from the architecture [Figure~\ref{fig:overview}(a,b)]. Seventeen models previously studied as separate solvable proxies, including those above, are instances of Eq.~\eqref{eq:nqf-intro} at particular $A$. A practical consequence is that $A(x)$ can be evaluated for a proposed layer before that layer is ever trained.

Several distinct regularities travel under the name of a neural scaling law, and they are worth separating before any theory is compared. Power laws are reported in the loss against optimization time within a single run; in the trained loss against model or dataset size; and along a compute-optimal frontier trading the two. What follows is a law of the first kind, for excess loss against rescaled training time, controlled by the same spectral structure that enters theories of the other two. How the trained loss depends on model size, dataset size and compute, and how those resources should be traded against one another, is a separate question we do not take up.

Within that scope, the closest theory is that of Bahri \emph{et al.} \cite{bahri2024explaining}, who separate a variance-limited regime, where the loss falls as the inverse of model or dataset size, from a resolution-limited one whose exponent is set by the data manifold or the tangent-kernel spectrum. Theirs is a theory of trained loss against scale; ours is of loss against time. Kernel treatments reach power-law learning curves from spectral structure by a related route \cite{bordelon2020spectrum,canatar2021spectral,maloney2022solvable,bordelon2024dynamical}, but they linearize about initialization and freeze the representation \cite{jacot2018neural,chizat2019lazy}, suppressing the feature growth that produces abrupt transitions \cite{radhakrishnan2022mechanism,beaglehole2023mechanism}; sudden learning is then treated separately, as saddle-to-saddle transitions for particular targets \cite{abbe2023sgd,michaud2023quantization,nam2024exactly,arous2025learning}, or as the loss of stability of a trivial symmetric solution \cite{wu2019learnability}. Our question spans the two: can the spectrum behind a macroscopic law emerge from within a feature-learning dynamics, and can that same dynamics produce the abrupt acquisition of individual modes? Answering both questions with one dynamics is tractable only because that dynamics is reducible to a few collective variables.

That reduction is our first result. Under stochastic gradient descent, the dynamics of any NQF close on the pair $(M,\mu)=(\sum_i w_iw_i^\top,\sum_i w_i)$, no matter how many trainable parameters the model contains (Theorem~\ref{theo:master}). The entire parameter trajectory is therefore represented by these low-dimensional moments; and within the NQF description a module of width $d$ admits a compressed equivalent of width $k_{\mathcal V}+1$, set by the rank of the data, which reproduces its predictions on the training data at every step under a rescaling (Theorem~\ref{theo:compressibility}). This measures the low-dimensional training dynamics \cite{gur2018gradient,li2018measuring,aghajanyan2021intrinsic}.
The second result is that the closed dynamics is in many cases solvable. In the common eigenbasis of the data, the flow of $M$ becomes the generalized Lotka--Volterra equation of population ecology \cite{hofbauer1998evolutionary,may1972will,bunin2017ecological}, with the eigenvalues $z_k$ of $M$ in the role of species abundances, and we solve it in four regimes (Theorems~\ref{theo:proportion}--\ref{theo:isotropic_samples}), each exhibiting a combination of sudden learning and neural scaling law phenomena. %The phenomenology of the opening paragraph then follows: each mode suddenly starts to learn at a time $t_k^*\simeq(a\zeta_k)^{-1}\ln(1/\epsilon)$, so sudden learning times separate logarithmically as $\epsilon\to0$ and the loss becomes a staircase, and when the growth rates $\zeta_k$ and target strengths $V_k$ are power-law distributed with exponents $\alpha_2$ and $\alpha_1$, the staircase aggregates into a power law of exponent $(\alpha_1-1)/\alpha_2$ in rescaled training time. That exponent is predicted from the two spectra with no fitted parameter.

This work is organized as follows. Section~\ref{sec:theory} derives the normal form and Section~\ref{sec:example} computes $A(x)$ for standard architectures. Section~\ref{sec:dynamics} analyzes the dynamics, and Section~\ref{sec:solutions} solves the dynamics for special cases. Section~\ref{sec:scaling} derives the predictions for sudden learning and neural scaling laws. Section~\ref{sec:exp} tests them numerically, and Section~\ref{sec: discussion} discusses the limitations and the extension to higher orders.

\section{Neural Quadratic Forms}\label{sec:theory}
This sections presents the central result of this paper. We show that permutation symmetry alone gives us a universal norm quadratic expansion, which is the result of Theorem~\ref{theo:main}. In Section~\ref{sec:example} we will show how each neural architecture is different.

\paragraph{Notation.} Throughout, $x$ denotes a single input to the module: one training example drawn from a set $\mathcal{X}$. The module itself is written $f_x$, a scalar-valued function of the parameters at fixed $x$; the subscript records that $x$ is held fixed while the parameters vary. Its $d$ components carry $p$ parameters each, collected as $w_1,\dots,w_d\in\mathbb{R}^p$, so $f_x:\mathbb{R}^{p\times d}\to\mathbb{R}$. Vector-valued outputs are treated in Appendix~\ref{app:multi-dimension} and change nothing essential. All notations used in our paper are listed in Table \ref{tab:notation}.

The starting point of the theory is that there is a universally shared mathematical structure across different types of layers and architectural modules (see Section~\ref{sec:example}). 
\begin{definition}[Permutation symmetry]
\label{def:perm}
$f_x: \mathbb{R}^{p \times d} \to \mathbb{R}$ satisfies the permutation symmetry ($S_d$- symmetry) if $f_x(w_1, \dots, w_d) = f_x(w_{\sigma(1)}, \dots, w_{\sigma(d)})$ for any permutation $\sigma$.
\end{definition}
Leveraging the terminology from MLPs, each $w_i$ can be seen as the weights of a neuron of a hidden layer. It is thus natural to call $w_i$ a ``neuron," and we will stick to this terminology. In fact, permutation symmetries have been suggested as the ``right" way to precisely define a neuron in deep learning \cite{ziyin2024symmetry}. To be more precise, one could also call $w_i$ the ``coordinate" of the $i$-th neuron. Thus, in our work, any subset of the weights that is the minimal unit of permutation symmetry will be called a ``neuron." In this terminology, the weights of a self-attention head are also a neuron.

For a simplified presentation of the theory, we assume one further benign condition that is obeyed by almost any neural architecture in use (see Section~\ref{sec:example}).
\begin{property}
Zero gradient at zero (ZGZ): for any neuron $i$ and any  $w_{j \neq i}$, $
\nabla_{w_i} f_x(w_1, \dots, w_d) |_{w_i = 0} = 0$.
\end{property}

The following theorem shows that for models with permutation symmetry, their Taylor expansions take a highly universal form. In some sense, this result can be seen as a variant of the fundamental theorem of symmetric polynomials \cite{wang2026universal}.

\begin{theorem}[Neural Quadratic Forms]
\label{theo:main}
Let $f_x: \mathbb{R}^{p \times d} \to \mathbb{R}$ be a three times continuously differentiable model with respect to its neurons $W = (w_1, \dots, w_d)$, where $w_i \in \mathbb{R}^p$. Assume $f_x$ satisfies the permutation symmetry and the ZGZ conditions. Then, 
\begin{equation}\label{eq:NQF}
f_x(W) = f_x(0)+\sum_{i=1}^d \mathrm{Tr}[w_i w_i^\top  A(x)] + O(\|W\|^3),
\end{equation}
where $A(x) \in \mathbb{R}^{p \times p}$ is a symmetric matrix dependent only on $x$, and $f_x(0)$ is the model evaluated at $W=0$. Equivalently, 
\begin{equation}
    \lim_{\phi \to \infty}  \phi (f_x(\phi^{-1/2} W)-f_x(0)) =  \sum_{i=1}^d \mathrm{Tr}[w_i w_i^\top  A(x)].
\end{equation}
\end{theorem}

\begin{remark}
    We note that the ZGZ condition is nonessential. Removing it leads to essentially the same result but with more complicated notations. We present these results in Appendix~\ref{app:remove_ZGZ}. Moreover, in practice, the ZGZ assumption can be replaced by the stronger (but equally common) assumption of having a per-neuron $Z_2$ symmetry, meaning $f_x(\dots, w_i, \dots) = f_x(\dots, -w_i, \dots)$ for all $i$. This means the 1st- and 3rd-order tensors in the Taylor expansion are zero, naturally elevating the error bound from $O(\|W\|^3)$ to $O(\|W\|^4)$.
\end{remark}

Therefore, we will refer to any model in the form of Eq.\eqref{eq:NQF} as a neural quadratic form (NQF). For example,
\begin{equation}
    f_x(W)= c_0  + \Tr [WW^\top A(x)]
\end{equation}
will be a generic NQF with $d$ as its width and $A$ as its ``architecture" or ``structure." We will from now on refer to $A$ as the structure matrix. In Appendix \ref{app:multi-dimension} we extend the results to multidimensional outputs. Note also the dimension of $M$. Both $M$ and $A(x)$ are $p\times p$, where $p$ is the number of parameters of one neuron: the width $d$ having already been summed away. Since the neurons are exchangeable, they may be read as identical particles with a $p$-dimensional  state space, and $\sum_i w_i$ and $M=\sum_i w_iw_i^\top$ as the first two moments of their empirical distribution, $M$ playing the role of a density matrix. 

The next section shows how different forms of $A$ distinguish different architectures. Note that $f$ depends on the neurons only through their second moment $M = WW^\top$, a point we take up in Section~\ref{sec:dynamics}.

%Permutation symmetry admits only symmetric functions of the $w_i$, and truncation at quadratic order only these two moments. %Everything downstream is sized accordingly: the dynamics of Section ~\ref{sec:dynamics} concern $p$ modes, and the compressed width (Theorem~\ref{theo:compressibility}) cannot exceed $p$.

\begin{table*}[t]
\caption{Six representative models that are instances of the neural quadratic form
\eqref{eq:nqf-intro}, together with $A(x)$ that distinguishes them. There are two ways to read the table:
many generic architectures reduce approximately to an NQF, and many previously separate
models are special cases of the NQF. The complete list appears in Table~\ref{tab:NQF_summary}.}
\label{tab:NQF_compact}
\renewcommand{\arraystretch}{1.35}
\begin{tabular}{@{}p{3.0cm}p{4.0cm}p{7.4cm}l@{}}
\hline\hline
\textbf{Model} & \textbf{NQF notation} & \textbf{Structure matrix $A(x)$} & \textbf{Refs.}\\
\hline
\rr Two-layer MLP
& \rr $f_x(W)=\sum_{i=1}^d v_i\,\phi(u_i^\top x)$
& \rr $A(x)=\frac{\phi'(0)}{2}\begin{bmatrix}0_{k\times k} & x\\ x^\top & 0\end{bmatrix}$
& \rr \cite{fukumizu1996regularity,radhakrishnan2022mechanism}\\
\rr Single-layer CNN
& \rr $f_x(W)=\sum_{i=1}^d v_i\sum_{p=1}^P \phi(k_i^\top x_p)$
& \rr $A(X)=\frac{\phi'(0)}{2}\begin{bmatrix}0_{m\times m} & \sum_p x_p\\ \sum_p x_p^\top & 0\end{bmatrix}$
& \rr \cite{krizhevsky2012imagenet,du2018gradient}\\
\rr Multi-head attention
& \rr $f_x(W)=\sum_{i=1}^d v_i^\top W_i^V X\cdot\mathrm{softmax}(\cdot)$
& \rr $A(X)=\frac12\begin{bmatrix}0_{2d_kD} & 0 & 0\\ 0 & 0 & x_{\mathrm{avg}}\otimes I_{d_v}\\ 0 & x_{\mathrm{avg}}^\top\otimes I_{d_v} & 0\end{bmatrix}$
& \rr \cite{vaswani2017attention}\\
\rr Mixture of experts
& \rr $f_x(W)=\sum_{i=1}^d \psi(r_i^\top x)\,E(x;\Theta_i)$
& \rr $A(x)=\frac{\psi'(0)}{2}\begin{bmatrix}0 & x\,G(x)^\top\\ G(x)\,x^\top & 0\end{bmatrix}$, $G:=\nabla_{\Theta}E(x;0)$
& \rr \cite{shazeer2017outrageously,nguyen2024sigmoid}\\
\rr Phase retrieval
& \rr $\hat y_a = a^\top WW^\top a$
& \rr $A_a = aa^\top$
& \rr \cite{candes2015phase,maillard2020phase}\\
\rr Diagonal linear network
& \rr $f_x(u)=\sum_k x_k u_k^2/4$
& \rr $A_x = \frac14\,\mathrm{Diag}(x)$
& \rr \cite{berthier2023incremental,pesme2021implicit}\\
\hline\hline
\end{tabular}
\renewcommand{\arraystretch}{1}
\end{table*}

\section{NQF for Different Architectures}
\label{sec:example}

In fact, a point often unclear to practitioners is that the existence of permutation symmetries is so universal that essentially any architecture module that has a notion of ``width" automatically has the permutation symmetry property. This makes Theorem~\ref{theo:main} applicable to almost any neural module one encounters in practice. Thus, in the NQF perspective, it is the structure matrices that determine the learning dynamics of neural networks. See Table \ref{tab:NQF_compact} for common models that are reducible to an NQF. We now discuss a few examples in more detail.

\paragraph{Two-Layer MLP.} Consider a two-layer Multi-Layer Perceptron (MLP) with a scalar output $f_x: \mathbb{R}^{d \times (k+1)} \to \mathbb{R}$ without bias terms. Let the parameters associated with the $i$-th hidden neuron be denoted as a single vector $w_i = [u_i^\top , v_i]^\top $, where $u_i \in \mathbb{R}^k$ is the input weight vector and $v_i \in \mathbb{R}$ is the readout (output) weight. The model is given by:
\begin{equation}
f_x(w_1, \dots, w_d) = \sum_{i=1}^d v_i \phi(u_i^\top  x)
\label{eq:MLP}
\end{equation}
where $x \in \mathbb{R}^k$ is the input and $\phi: \mathbb{R} \to \mathbb{R}$ is the activation function.
\begin{proposition}
\label{prop:MLP}
If $\phi$ is three times continuously differentiable and satisfies $\phi(0) = 0$, then the model \eqref{eq:MLP} satisfies the assumptions in Theorem \ref{theo:main} with
\begin{equation}
    A(x) = \frac{\phi'(0)}{2} 
    \begin{bmatrix}
        0_{k \times k} & x \\
        x^\top  & 0
    \end{bmatrix}.
\end{equation}
\end{proposition}

\paragraph{Self-Attention.} Consider a Multi-Head Attention (MHA) model computing a scalar output from an input query token $x \in \mathbb{R}^D$ and a context matrix $X \in \mathbb{R}^{D \times N}$ (where $N$ is the sequence length).

Let the parameters associated with the $i$-th attention head be grouped into a single parameter vector $w_i = \mathrm{vec}(W_i^Q, W_i^K, W_i^V, v_i)$, where $W_i^Q, W_i^K \in \mathbb{R}^{d_k \times D}$ are the query and key matrices, $W_i^V \in \mathbb{R}^{d_v \times D}$ is the value matrix, and $v_i \in \mathbb{R}^{d_v}$ is the readout vector for the scalar output. Assuming no bias terms, the model is formulated as:
\begin{equation}
    f_x(w_1, \dots, w_d) = \sum_{i=1}^d v_i^\top  W_i^V X \cdot \mathrm{softmax} \left( \frac{X^\top  (W_i^K)^\top  W_i^Q x}{\sqrt{d_k}} \right)
    \label{eq:MHA}
\end{equation}
where $\mathrm{softmax}$ is applied over the $N$ dimension.
\begin{proposition}
\label{prop:attention}
The MHA model \eqref{eq:MHA} satisfies the assumption in Theorem \ref{theo:main}. The matrix $A(X)$ takes the following block form:
\begin{equation}
    A(X) = \frac{1}{2} 
    \begin{bmatrix}
        0 & 0 & 0 & 0 \\
        0 & 0 & 0 & 0 \\
        0 & 0 & 0 & x_{avg} \otimes I_{d_v} \\
        0 & 0 & x_{avg}^\top  \otimes I_{d_v} & 0
    \end{bmatrix}
\end{equation}
where $x_{avg} = \frac{1}{N} X \mathbf{1}_N \in \mathbb{R}^D$, $\mathbf{1}_N \in \mathbb{R}^N$ is a vector of all ones, and $\otimes$ denotes the Kronecker product.
\end{proposition}
This means that to leading order, what determines the behavior of a self-attention head are the value and output matrices, and the self-attention behaves like an averaging operator. In fact, the query and key weights appear only at quartic order.

\paragraph{Query-Key-Only Head.} Now, consider a variant of MHA where the readout vector and value matrix are fixed and merged into a constant vector $c \in \mathbb{R}^D$. While this is rarely the case in practice, it is actually a common theoretical model of the learning dynamics of attention (e.g., \cite{ataee2023max,maulen2026attention}, or in which the readout
and query-key matrices are trained separately \cite{wang2025learning}). The model evaluates a scalar output as:
\begin{equation}
f_x(w_1, \dots, w_d) = \sum_{i=1}^d c^\top  X \cdot \mathrm{softmax} \left( \frac{X^\top  (W_i^K)^\top  W_i^Q x}{\sqrt{d_k}} \right)
\label{eq:Query-Key-Only}
\end{equation}
Let the parameters for the $i$-th head be the query and key matrices $w_i = [\mathrm{vec}(W_i^Q)^\top , \mathrm{vec}(W_i^K)^\top ]^\top  \in \mathbb{R}^{2 d_k D}$.
\begin{proposition}
\label{prop:MHA_variant}
Model \eqref{eq:Query-Key-Only} satisfies all assumptions in Theorem \ref{theo:main}. Furthermore, $A(X)$ takes the block form:
\begin{equation}
A(X) = \frac{1}{2 \sqrt{d_k}}
\begin{bmatrix}
0 & x u^\top  \otimes I_{d_k} \\
u x^\top  \otimes I_{d_k} & 0
\end{bmatrix}
\end{equation}
where $u = \Sigma_X c \in \mathbb{R}^D$, and $\Sigma_X = \frac{1}{N} X X^\top  - (\frac{1}{N} X \mathbf{1}_N)(\frac{1}{N} X \mathbf{1}_N)^\top  \in \mathbb{R}^{D \times D}$ is the sample covariance matrix.
\end{proposition}

\paragraph{Single-head attention logits.}
Finally, we can also consider the single-head attention model, which can, interestingly, also be regarded as an independent NQF because the query and key weight matrices contain the permutation symmetry as a subgroup.
\begin{proposition}
\label{prop:SingleHead_Row}
By treating the $i$-th rows of $W^Q$ and $W^K$ as $w_i \in \mathbb{R}^{2D}$, the single-head attention model 
\begin{equation}
f_x(w_1, \dots, w_{d_k}) = c^\top  X \cdot \mathrm{softmax} \left( \frac{X^\top  (W^K)^\top  W^Q x}{\sqrt{d_k}} \right)
\end{equation}
satisfies all assumptions in Theorem \ref{theo:main}. Furthermore, $A(X) \in \mathbb{R}^{2D \times 2D}$ takes the block form:
\begin{equation}
A(X) = \frac{1}{2 \sqrt{d_k}}
\begin{bmatrix}0 & x u^\top  \\u x^\top  & 0
\end{bmatrix}
\end{equation}
where $u = \Sigma_X c \in \mathbb{R}^D$, and $\Sigma_X$ is the sample covariance matrix.
\end{proposition}

\paragraph{Mixture of Experts.} Consider a Mixture of Experts (MoE) model computing a scalar output $f_x$ from an input $x \in \mathbb{R}^k$. Let the parameters associated with the $i$-th expert branch be grouped into a single vector $w_i = [r_i^\top , \mathrm{vec}(\Theta_i)^\top ]^\top $, where $r_i \in \mathbb{R}^k$ is the routing (gating) weight vector, and $\Theta_i$ represents the internal parameters of the expert network $E(x; \Theta_i)$. We employ an independent gating mechanism \cite{nguyen2024sigmoid}:
\begin{equation}
f_x(w_1, \dots, w_d) = \sum_{i=1}^d \psi(r_i^\top  x) E(x; \Theta_i)
\label{eq:MoE}
\end{equation}
where $\psi: \mathbb{R} \to \mathbb{R}$ is a smooth gating activation function (e.g., Tanh or GeLU) satisfying $\psi(0) = 0$. Each expert $E(x; \Theta_i)$ is a smooth neural network satisfying $E(x; 0) = 0$.
\begin{proposition}
\label{prop:MoE}
Assuming the gating function $\psi$ and the expert networks $E$ are at least three times continuously differentiable, the MoE model \eqref{eq:MoE} satisfies the assumption in Theorem \ref{theo:main} with
\begin{equation}
A(x) = \frac{1}{2} \begin{bmatrix} 0 & \psi'(0) x \nabla_{\Theta_i} E(x; 0)^\top  \\ \psi'(0) \nabla_{\Theta_i} E(x; 0) x^\top  & 0 \end{bmatrix}.
\end{equation}
\end{proposition}

\paragraph{Convolutional Neural Networks.}
Consider a single-layer Convolutional Neural Network (CNN) computing a scalar output $f_x$ from an input feature map $X$. Let the parameters associated with the $i$-th convolutional channel be grouped into a single vector $w_i = [k_i^\top , v_i]^\top $, where $k_i \in \mathbb{R}^m$ represents the flattened weights of the convolutional filter, and $v_i \in \mathbb{R}$ is the readout weight applied after pooling.

Assuming the network uses global sum-pooling over the $P$ spatial patches of the input (denoted as $x_p \in \mathbb{R}^m$ for $p = 1, \dots, P$), the model is formulated as:
\begin{equation}
    f_x(w_1, \dots, w_d) = \sum_{i=1}^d v_i \sum_{p=1}^P \phi(k_i^\top  x_p),
\label{eq:CNN}
\end{equation}
where $\phi: \mathbb{R} \to \mathbb{R}$ is the activation function.

\begin{proposition}
If $\phi$ is three times continuously differentiable and satisfies $\phi(0) = 0$, then the CNN \eqref{eq:CNN} satisfies the assumption in Theorem \ref{theo:main} with
\begin{equation}
A(X) = \frac{\phi'(0)}{2} \begin{bmatrix} 0_{m \times m} & \sum_{p=1}^P x_p \\ \sum_{p=1}^P x_p^\top  & 0 \end{bmatrix}.
\end{equation}
\label{prop:CNN}
\end{proposition}

\paragraph{Notions of Layer and Representation.} While the form $\Tr[WAW^\top ]$ is universal, the matrix $A$ always takes a sparse and strongly off-diagonal form. In particular, for any neural layer above,
\begin{equation}
    A = \frac{1}{2}\begin{bmatrix}
        0 & C(x)\\
        C^\top (x) & 0 
    \end{bmatrix},
\end{equation}
where $C \in \mathbb{R}^{d_1 \times d_2}$ is an arbitrary rectangular matrix with suitable dimensions. This means that one can also divide $W$ into corresponding blocks:
$W = \begin{bmatrix} Z_1 \\ Z_2 \end{bmatrix}$, 
such that 
\begin{equation}
M = WW^\top = \begin{bmatrix} Z_1 Z_1^\top  & Z_1 Z_2^\top  \\ Z_2 Z_1^\top  & Z_2 Z_2^\top  \end{bmatrix}
\end{equation}
This implies that the output of the model only depends on $Z_2Z_1^\top $: $f_x = \Tr[Z_2 Z_1^\top  C]$. In the case of a vector $C=x$, $Z_1$ is also a vector and so $f_x = Z_2Z_1^\top  x$, a two-layer linear network. Therefore, one can abstractly think of $Z_1^\top$ as the first-layer weight and $Z_2 $ as the second-layer weight. The representation also becomes definable. The first-layer latent representation of $C(x)$ is thus
\begin{equation}
    h(x)= Z_1^\top C(x).
\end{equation}
For a structured model, such as above, it is possible to talk about a latent representation and define the notion of layers. Whenever the matrix $A$ can be written in this form, we refer to the model as a ``feedforward" model.

\section{Learning Dynamics of the NQF}\label{sec:dynamics}
Sections~\ref{sec:theory} and~\ref{sec:example} concern a model when expanded, and in this section, we study the learning dynamics of these models. A key question this section answers is how many numbers are needed to follow training. The answer is that the number does not grow with the model: whatever the width, the trajectory of an NQF is a flow on the pair $(M,\mu)$. This is stated as Theorem~\ref{theo:master}. One consequence follows immediately: learning is highly redundant, in that a model can be compressed with no change to its trajectory at all (Theorem~\ref{theo:compressibility}). 

We work throughout with the most general form of NQF without the ZGZ condition (See Appendix \ref{app:remove_ZGZ}). If we denote $M:=\sum_{i=1}^dw_iw_i^\top $ and $\mu:=\sum_{i=1}^dw_i$, we can write a general NQF as
\begin{equation}\label{eq:NQF-general}
f_x(W)=f_x(0)+g(x)^\top \mu+\mu^\top B(x)\mu+\Tr[MA(x)].
\end{equation}
Without loss of generality, $A$ and $B$ matrices can always be regarded as symmetric matrices.

\begin{theorem}[Master Theorem for NQF]
\label{theo:master}
Let $A(x), B(x)$ be symmetric. Under SGD
\begin{equation}
\Delta W=-\eta\sum_{x \in \mathcal{B}}\nabla_W\mathcal{L}(f_x(W))
\end{equation}
with identical data sampling, two neural quadratic models have $M_a(t) = M_b(t),\ \mu_a(t)=\mu_b(t)$ at any time $t$ if at initialization:
\begin{equation}
M_a(0) = M_b(0),\ \mu_a(0)=\mu_b(0).
\end{equation}
The learning dynamics are completely determined by $M$, $\mu$:
\begin{equation}
\Delta \mu = -\eta \big( d \cdot v + H \mu \big)
\end{equation}
\begin{equation}
\Delta M = -\eta \big( v \mu^\top  + \mu v^\top  + H M + M H \big) + \eta^2 \big( d \cdot v v^\top  + v \mu^\top  H + H \mu v^\top  + H M H \big)
\end{equation}
where $d$ is the number of neurons, $v := \sum_{x \in \mathcal{B}} \ell'_x \big(g(x) + 2B(x)\mu \big)$ and $H := \sum_{x \in \mathcal{B}} 2 \ell'_x A(x)$. Here we denote $\ell'_x := \frac{\partial \mathcal{L}}{\partial f_x}$ and $\mathcal{B}$ denotes a minibatch.
\end{theorem}
This theorem can be extended to a form that describes a broad class of optimization methods including vanilla SGD, SGD with weight decay, and Polyak Momentum. With weight decay, a simple decay $-\gamma M$ term appears in the equation, where $\gamma$ is the weight decay strength. See Appendix \ref{app:linear_update}. 

A deep neural network can often be viewed as a composition of multiple layers, each of which can be approximated by an NQF model. We theoretically and empirically study multilayer NQFs in Appendix \ref{app:multi-layer}.

\paragraph{Compressibility of NQF.}
A key implication of Theorem~\ref{theo:master} is that the learning process of neural networks is highly redundant. We will say that two neural networks $f_x(\theta)$ and $g_x(\theta')$ have the same learning dynamics if for any $x$,
\begin{equation}
    f_x(\theta_t) = g_x(\theta_t').
\end{equation}
Namely, the identity is defined on the functional side, and it is possible for two models to have the same dynamics even if they have different parameters.

Actually, there exist infinitely many parameter configurations whose learning dynamics are identical. This also implies that the learning process of a very large neural network is identical to and can be fully captured by a much smaller network -- a phenomenon that has been termed the ``dynamical lottery ticket hypothesis" \cite{frankle2018lottery,wang2026universal}.

\begin{theorem}
\label{theo:compressibility}
Let an NQF have $d$ neurons, defined by functions $g(x), B(x), A(x)$, learning rate $\eta$, and initial parameter statistics $\mu(0)$ and $M(0)$. Let $k_{\mathcal{V}} := \dim \left( \text{span} \bigcup_{x \in \mathcal{X}} \Big( \{g(x)\} \cup \text{Col}(A(x)) \cup \text{Col}(B(x)) \Big) \right) \leq p$ represent the dimension of the joint subspace spanned by the vectors $g(x)$ and all column vectors of $B(x)$ and $A(x)$ across all $x \in \mathcal{X}$, where $\mathcal{X}$ denotes the training set. If $d > k_{\mathcal{V}}+1$, then there exists a smaller NQF with $d' = k_{\mathcal{V}}+1$ neurons, characterized by $\tilde{g}(x) = g(x)$, $\tilde{B}(x) = B(x)$, $\tilde{A}(x) = \frac{d'}{d} A(x)$ and learning rate $\tilde{\eta} = \frac{d}{d'} \eta$, such that:
\begin{enumerate}[noitemsep,topsep=0pt, parsep=0pt,partopsep=0pt, leftmargin=13pt]
    \item $f_{d'} = f_d$ for all inputs $x\in\mathcal{X}$;
    \item $f_{d'} = f_d$ for all training steps $t \ge 0$ under SGD with identical data sampling.
\end{enumerate}
\end{theorem}

In fact, independent of the original width $d$, there exists a finite-size NQF whose learning dynamics is identical to that of the original. This partially explains the commonly observed low dimensionality of the training dynamics \cite{gur2018gradient,li2018measuring,aghajanyan2021intrinsic}. Moreover, when the data is low-rank, this theorem states that the dimension of the learning dynamics is at most that of the data, a direct explanation of the folklore belief that structures in the data make learning simple. In Appendix \ref{app:compression_error} we further characterize the compression error.

\section{Exactly Solvable Cases of Learning Dynamics}\label{sec:solutions}
The learning dynamics of NQF do not have a general solution. However, with special initialization or data distributions, solutions of the learning dynamics are obtainable. Here, we focus on training NQF with the gradient flow algorithm. We consider the empirical Mean Squared Error (MSE) loss over a dataset of $m$ samples $\{x_\mu, y_\mu\}_{\mu=1}^m$:
\begin{equation}
\mathcal{L}(W) = \frac{1}{m}\sum_{\mu=1}^m \left( \mathrm{Tr}[WW^\top  A(x_\mu)] - y_\mu \right)^2
\end{equation}
where $W = [w_1, \dots, w_d] \in \mathbb{R}^{p \times d}$ is the parameter matrix. 
Let $\Delta_\mu(t) = \mathrm{Tr}[W(t)W(t)^\top  A(x_\mu)] - y_\mu$ denote the residual for the $\mu$-th sample. The gradient flow $\dot{W} = -\nabla_W L(W)$ is given by:
\begin{equation}
\label{eq:W_dynamics}
\dot{W}(t) = -\frac{4}{m} \sum_{\mu=1}^m \Delta_\mu(t) A(x_\mu) W(t) = -H(t) W(t)
\end{equation}
where $H(t) := \frac{4}{m} \sum_{\mu=1}^m \Delta_\mu(t) A(x_\mu) \in \mathbb{R}^{p \times p}$.

Notice that the network's output is determined by the positive semi-definite matrix $M = WW^\top  \in \mathbb{R}^{p \times p}$. By differentiating $M(t)$, we observe that the dynamics naturally close in the space of $M$:
\begin{equation}
\begin{aligned}
\dot{M}(t) &= \dot{W}W^\top  + W\dot{W}^\top  = -H(t)M(t) - M(t)H(t)
\end{aligned}
\label{eq:M_dynamics}
\end{equation}
This kind of equation appears frequently in fluid dynamics and polymer physics, and this dynamical equation can be seen as a zero-dimensional homogeneous flow, and $M$ can be identified as the Reynolds stress \cite{pope2001turbulent}. 
The non-linear matrix differential equation \eqref{eq:M_dynamics} is generally impossible to solve. However, for some special cases it has exact solutions. 

We will assume the following commutativity condition -- this is a common assumption for solving the learning dynamics of exactly solvable models, e.g., \cite{saxe2014exact,gunasekar2017implicit}.
\begin{assumption}
The dataset consists of $m$ mutually commuting symmetric data matrices: $A(x_\mu) = P \Lambda_\mu P^\top $, where $P$ is an orthogonal matrix and $\Lambda_\mu = \mathrm{diag}(\lambda_{\mu, 1}, \dots, \lambda_{\mu, p})$.
\label{assume:commutation}
\end{assumption}
Under Assumption \ref{assume:commutation}, we can project the dynamics into the common eigenbasis defined by $P$. Let $\tilde{M}(t) = P^\top M(t) P$. The residual can be rewritten as:
\begin{equation}
\Delta_\mu(t) = \mathrm{Tr}[M(t) A(x_\mu)] - y_\mu = \mathrm{Tr}[\tilde{M}(t) \Lambda_\mu] - y_\mu = \sum_{k=1}^p \tilde{M}_{kk}(t) \lambda_{\mu, k} - y_\mu.
\label{eq:residual_diagonal}
\end{equation}
Equation \eqref{eq:residual_diagonal} reveals that the network's output and the loss gradient only depend on the diagonal elements of $\tilde{M}(t)$. We define $z_k(t) := \tilde{M}_{kk}(t) \ge 0$ as the $k$-th feature.

\begin{proposition}
\label{prop:decoupling}
Under Assumption \ref{assume:commutation}, the off-diagonal elements are slaved to the diagonal elements and do not influence the predictions or the loss. Their evolution is solved by:
\begin{equation}
\tilde{M}_{ij}(t) = \tilde{M}_{ij}(0) \exp\left( -\int_0^t  \left( \tilde{H}_{ii}(\tau) + \tilde{H}_{jj}(\tau) \right) d\tau \right).
\end{equation}
\end{proposition}

Substituting $\Delta_\mu(t) = \sum_{j=1}^p \lambda_{\mu,j} z_j(t) - y_\mu$ into \eqref{eq:M_dynamics}, we obtain the evolution of the $k$-th feature:
\begin{equation}
\dot{z}_k(t) = -\frac{8}{m} \sum_{\mu=1}^m \Delta_\mu(t) \lambda_{\mu,k} z_k(t).
\label{eq:GD-z}
\end{equation}
Expanding $\Delta_\mu(t)$ yields:
\begin{equation}
\label{eq:GLV}
\dot{z}_k(t) = \frac{8}{m} \left( \sum_{\mu=1}^m y_\mu \lambda_{\mu,k} - \sum_{j=1}^p \left( \sum_{\mu=1}^m \lambda_{\mu,k} \lambda_{\mu,j} \right) z_j(t) \right) z_k(t).
\end{equation}
\eqref{eq:GLV} is the Generalized Lotka-Volterra (GLV) equation from population ecology.
In general, it still does not have analytic solutions\footnote{It has implicit solutions as shown in \cite{gunasekar2017implicit}.}. For the following sections, we will consider several special cases where the dynamics are actually solvable. The connection to the GLV allows some direct interpretation of $z_k$ as the abundance of a species and $C_{k,j}:=\sum_\mu \lambda_{\mu,k} \lambda_{\mu,j}$ as the competition (when positive) and cooperation (when negative) between species \cite{hofbauer1998evolutionary}.

\paragraph{Proportional Samples.}
\begin{theorem}
\label{theo:proportion}
Suppose Assumption \ref{assume:commutation} holds, and further assume that the data matrices can be expressed as $A(x_\mu) = c_\mu A$ for a symmetric matrix $A$ and constants $c_\mu$.
Define $\alpha := \frac{4}{m} \sum_{\mu=1}^m c_\mu^2$ and $\beta := \frac{4}{m} \sum_{\mu=1}^m c_\mu y_\mu$. Then, the trajectory of $W(t)$ is given by:
\begin{equation}
W(t) = P \exp(-\xi(t) \Lambda) P^\top  W(0)
\end{equation}
where $\exp(-\xi(t) \Lambda) = \mathrm{diag}\left(\exp(-\lambda_1 \xi(t)), \dots, \exp(-\lambda_p \xi(t))\right)$, and the auxiliary scalar variable $\xi(t)$ satisfies the following implicit integral equation:
\begin{equation}
t = \int_0^{\xi(t)} \frac{ds}{\alpha \sum_{j=1}^p \lambda_j z_j(0) \exp(-2 \lambda_j s) - \beta}
\label{eq:implicit-xi}
\end{equation}
\end{theorem}
\begin{remark}
The exact solution in \cite{xu2025three} is a special case of Theorem \ref{theo:proportion}, where $\Lambda$ only has two eigenvalues and \eqref{eq:implicit-xi} can be integrated explicitly.
\end{remark}

\paragraph{Orthogonal Samples.}
\begin{theorem}
\label{theo:orthogonal_samples}
Suppose Assumption \ref{assume:commutation} holds, and further assume that $A(x_\mu) A(x_\nu) = 0$ for all $\mu \neq \nu$. Then, the trajectory of $W(t)$ is given by
\begin{equation}
W(t) = P \exp(-\Sigma(t)) P^\top  W(0),
\end{equation}
where $\Sigma(t) = \frac{4}{m} \sum_{\mu=1}^m \xi_\mu(t) \Lambda_\mu$ is a diagonal matrix, and the auxiliary scalar variable $\xi_\mu(t)$ for each $\mu \in \{1, \dots, m\}$ evolves independently according to the following implicit integral equation:
\begin{equation}
t = \int_0^{\xi_\mu(t)} \frac{ds}{\sum_{k=1}^p \lambda_{\mu, k} z_k(0) \exp\left(-\frac{8}{m} \lambda_{\mu, k} s\right) - y_\mu}.
\label{eq:implicit-solution}
\end{equation}
\end{theorem}

\paragraph{Orthogonal Features.}
\begin{theorem}
\label{theo:ortho_features}
Suppose Assumption \ref{assume:commutation} holds, and further assume that: $\sum_{\mu=1}^m \lambda_{\mu, k} \lambda_{\mu, j} = 0$ for all $k \neq j$.
Define
\begin{equation}
r_k := \frac{8}{m} \sum_{\mu=1}^m \lambda_{\mu, k} y_\mu, \quad C_{kk} := \frac{8}{m} \sum_{\mu=1}^m \lambda_{\mu, k}^2.
\label{eq:r-and-C}
\end{equation}
Then the trajectory of $W(t)$ is given by:
\begin{equation}
W(t) = P D(t) P^\top  W(0),
\end{equation}
where $D(t) = \mathrm{diag}(d_1(t), \dots, d_p(t))$ is a diagonal matrix with $d_k(t) = \sqrt{z_k(t) / z_k(0)}$ if $z_k(0)\neq0$ and $d_k(t)=0$ otherwise. The exact solution for $z_k(t)$ is:
\begin{equation}
z_k(t) = \frac{r_k z_k(0)}{C_{kk} z_k(0) + (r_k - C_{kk} z_k(0)) \exp(-r_k t)}
\end{equation}
if $r_k \neq 0$ and $z_k(t) = \frac{z_k(0)}{1 + C_{kk} z_k(0) t}$ otherwise.
\end{theorem}

\paragraph{Isotropic Samples.}

The previous solutions rely on the commutativity assumption (Assumption \ref{assume:commutation}). In this section, we explore a different solvable regime where the data matrices do not commute, but instead satisfy an isotropic condition.

\begin{assumption}
\label{assume:isotropic}
Assume the dataset consists of $m \ge p(p+1)/2$  symmetric data matrices $\{A(x_\mu)\}_{\mu=1}^m$ and
their second moment tensor satisfies:
\begin{equation}
\frac{1}{m} \sum_{\mu=1}^m A_{ij}(x_\mu) A_{kl}(x_\mu) = \frac{c}{2} (\delta_{ik}\delta_{jl} + \delta_{il}\delta_{jk})
\end{equation}
where $c>0$ is a constant, and $\delta$ is the Kronecker delta.
\end{assumption}

\begin{theorem}\label{theo:isotropic_samples}
Suppose Assumption \ref{assume:isotropic} holds. Let $Y = \frac{4}{m} \sum_{\mu=1}^m y_\mu A(x_\mu) \in \mathbb{R}^{p \times p}$. Then, the trajectory of $M(t) = W(t)W(t)^\top $ has a closed-form solution:
\begin{equation}
M(t) = \exp(Y t) M(0) \left[ I + 8c \Phi(t) M(0) \right]^{-1} \exp(Y t),
\end{equation}
where $I$ is the identity matrix, and $\Phi(t) = \int_0^t  \exp(2Y\tau) d\tau$. If $Y$ is non-singular, $\Phi(t) = \frac{1}{2}Y^{-1}(\exp(2Yt) - I)$.

Furthermore, if the initial state $M(0)$ and $Y$ are simultaneously diagonalizable, then, the eigenvalues of $M(t)$, denoted as $z_k(t)$, evolve according to
\begin{equation}
z_k(t) = \frac{\gamma_k z_k(0)}{4c z_k(0) + (\gamma_k - 4c z_k(0)) \exp(-2\gamma_k t)}
\end{equation}
if $\gamma_k\neq0$ and $z_k(t) = \frac{z_k(0)}{1 + 8c z_k(0) t}$ otherwise, where $\gamma_k$ is the $k$-th eigenvalue of $Y$.
\end{theorem}

In Appendix \ref{app:exact_solution} we futher propose a general solution which unifies Theorem \ref{theo:ortho_features} and
the simultaneously diagonalizable case of Theorem \ref{theo:isotropic_samples}, as well as the exact solutions of linear networks in \cite{saxe2014exact}.

\paragraph{Neural Tangent Kernel and Feature Learning.} While the main focus of our work is sudden learning at a small initialization, the NQF models can also be used to study lazy training and feature learning. We discuss this point in Appendix~\ref{app:ntk}.

\section{Sudden Learning and Neural Scaling Laws}\label{sec:scaling}
When learning of single features or data points is step-function-like, they can be naturally composed to give an arbitrary learning curve (as long as it is monotonically decreasing). For example, in this section, we show that when the data correlations follow a power-law structure, the learning curves naturally become power laws with predictable exponents. This leads to a precise prediction of the exponents of neural scaling laws. In this section, we only provide an informal statement. Please refer to Appendix \ref{app:feature-descent} for formal statements and the proofs.

\paragraph{Feature-wise Descent}
\label{sec:saddle-to-saddle}
Under the conditions of Theorem \ref{theo:ortho_features} (or Theorem~\ref{theo:isotropic_samples} with $M(0)$ and $Y$ simultaneously diagonalizable), define $\zeta_k = r_k$ (in Theorem \ref{theo:ortho_features}) or $\zeta_k=\gamma_k$ (in Theorem~\ref{theo:isotropic_samples}), which can be understood as the effective growth rate for the feature $z_k$. As the initialization scale $\epsilon\to0$, the characteristic time for $z_k$ to be activated is
\begin{equation}
t_k^*\sim\frac{1}{\zeta_k} \ln\frac{1}{\epsilon}.
\label{eq:ignition-time}
\end{equation}
Consequently, the gaps between different feature activations diverge in the limit of small initialization, which exhibits saddle-to-saddle dynamics: 
\begin{equation}
\lim_{\epsilon \to 0} |t_{k'}^* - t_k^*| = \infty
\end{equation}
if $\zeta_{k'}\neq\zeta_k$.

Moreover, define $V_k := \frac{r_k^2}{C_{kk}}$ (under Theorem \ref{theo:ortho_features}) or $V_k := \frac{\gamma_k^2}{c}$ (under Theorem \ref{theo:isotropic_samples}), which can be understood as the effective target strength. Assume that $\zeta_k$ and $V_k$ exhibit the following power-law decays:
\begin{align}
\zeta_k \propto k^{-\alpha_2}, \
V_k \propto k^{-\alpha_1}.
\end{align}
Then, under the infinite-width limit $p\to\infty$ and the small initialization limit $\epsilon\to0$, the limiting excess loss has a power-law decay
\begin{equation}
\mathcal{E}(\tau) = \Theta\left( \tau^{-\frac{\alpha_1 - 1}{\alpha_2}} \right).
\end{equation}

\paragraph{Sample-wise Descent}
Under the conditions of Theorem \ref{theo:orthogonal_samples}, suppose $\zeta_\mu := \frac{8}{m} \max_k \lambda_{\mu, k} y_\mu>0$, which can be understood as the effective growth rate of the $\mu$-th sample. When the initialization scale $\epsilon\to0$, the empirical MSE converges to a saddle-to-saddle trajectory
\begin{equation}
\lim_{\epsilon \to 0^+} L\left(\tau \ln\frac{1}{\epsilon}\right) = \frac{1}{m} \sum_{\mu=1}^m y_\mu^2 \cdot \mathbb{I}(\tau < \tau_\mu^*),
\label{eq:staircase-limit}
\end{equation}
where $\mathbb{I}(\cdot)$ denotes the indicator function and $\tau_\mu^* := \frac{1}{\zeta_\mu}$.

Moreover, suppose that the $\zeta_\mu$ and $y_\mu$ exhibit the following power-law decay:
\begin{align}
\zeta_\mu \propto\mu^{-\gamma_2},\
y_\mu\propto \mu^{-\gamma_1}
\end{align}
for $\gamma_1 > 1/2$ and $\gamma_2 > 0$. Then under the infinite sample limit $m\to\infty$ and the small initialization limit $\epsilon\to0$, the limiting MSE follows a power-law decay:
\begin{equation}
\mathcal{L}(\tau) = \Theta\left( \tau^{-\frac{2\gamma_1 - 1}{\gamma_2}} \right).
\end{equation}

The dynamics in Theorem \ref{theo:proportion} can be viewed as being mathematically equivalent to the solution in Theorem \ref{theo:orthogonal_samples} restricted to $m=1$, and thus there is only one plateau.

\section{Experiments}\label{sec:exp}
\begin{figure}[t]
\centering
\includegraphics[width=1.0\linewidth]{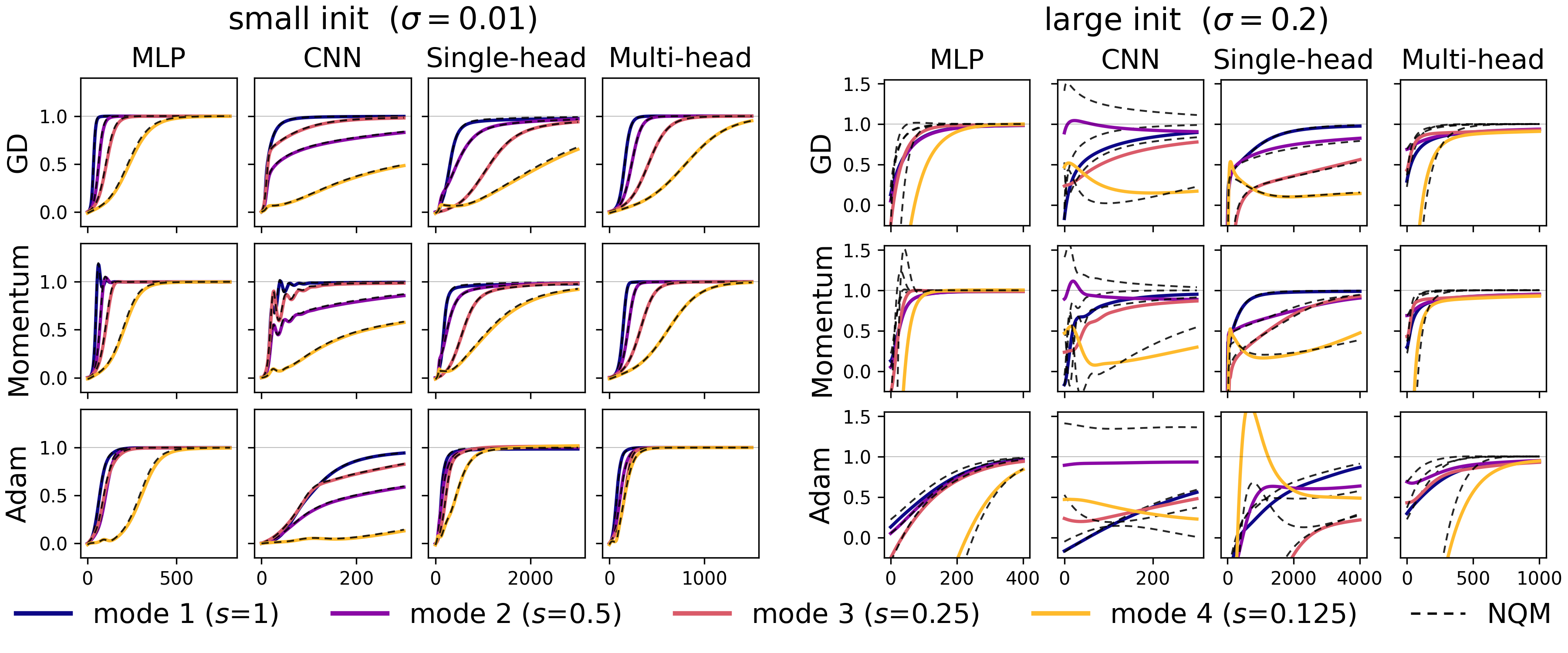}
\caption{NQF approximation across initialization, optimizer, and
architecture. Two initialization scales are juxtaposed---\emph{small init} ($\sigma=0.01$, left) and \emph{large init} ($\sigma=0.2$, right). Within each block, rows are optimizers and columns are architectures. Each module is trained together with its NQF approximation on an NQF teacher with singular values
$s=(1,\tfrac12,\tfrac14,\tfrac18)$. Curves show the recovered $k$-th singular value $\hat{s}_k(t)/s_k$ versus training step; \emph{colored solid} $=$ original module, \emph{black dashed} $=$ NQF.}
\label{fig:NQF_approx}
\end{figure}
\begin{figure}[t]
    \centering
    \includegraphics[width=\linewidth]{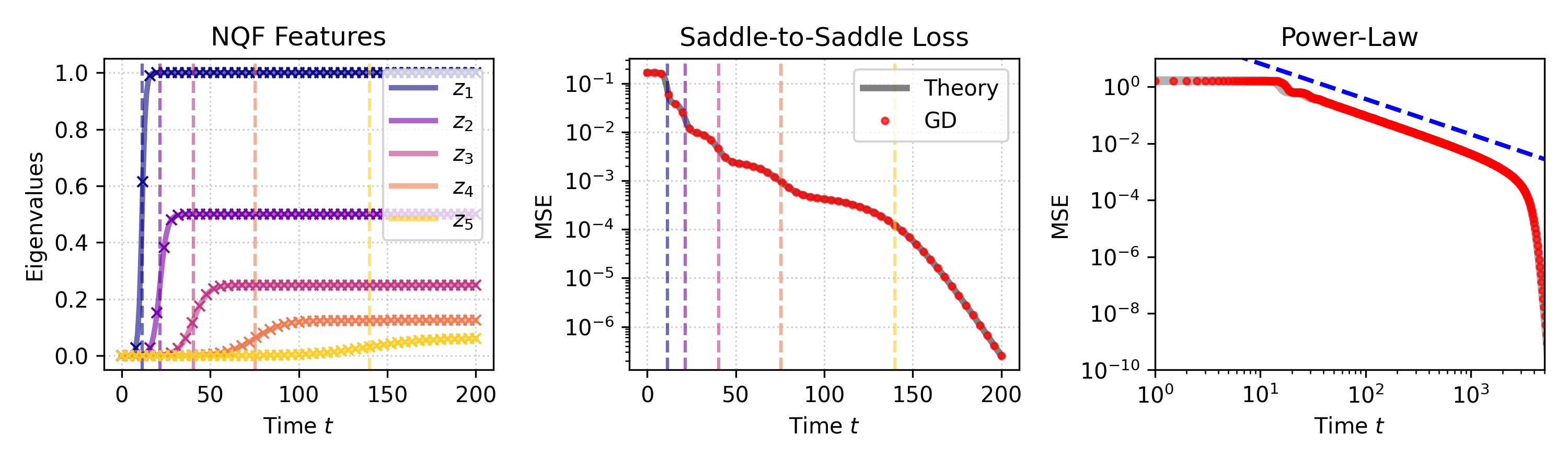}
    \caption{Sudden learning in a solvable NQF, the phenomenon the theory of this paper is built to explain. The experiment is under the feature-wise descent setting of Section \ref{sec:saddle-to-saddle}. Solid lines represent the theory in Theorem \ref{theo:ortho_features} and markers represent the results obtained by GD. Vertical dashed lines mark the characteristic timescales $t_k^* = \frac{1}{r_k} \ln \left( \frac{r_k}{C_{kk}\epsilon} \right)$ where each feature reaches half-saturation.  \textbf{Left}: Evolution of the eigenvalues of $WW^\top $. \textbf{Middle}: The excess MSE loss over time, which exhibits sequential plateaus. Each sharp drop in the loss aligns with $t_k^*$. \textbf{Right}: Under the feature-wise descent setting of Section \ref{sec:saddle-to-saddle}, the blue dashed line indicates the theoretical slope and confirms that our predicted slope matches the loss trajectory before the finite-size cutoff.
    }
    \label{fig:feature-descent}
\end{figure}
\begin{figure}[t]
    \centering
    \includegraphics[width=0.8\linewidth]{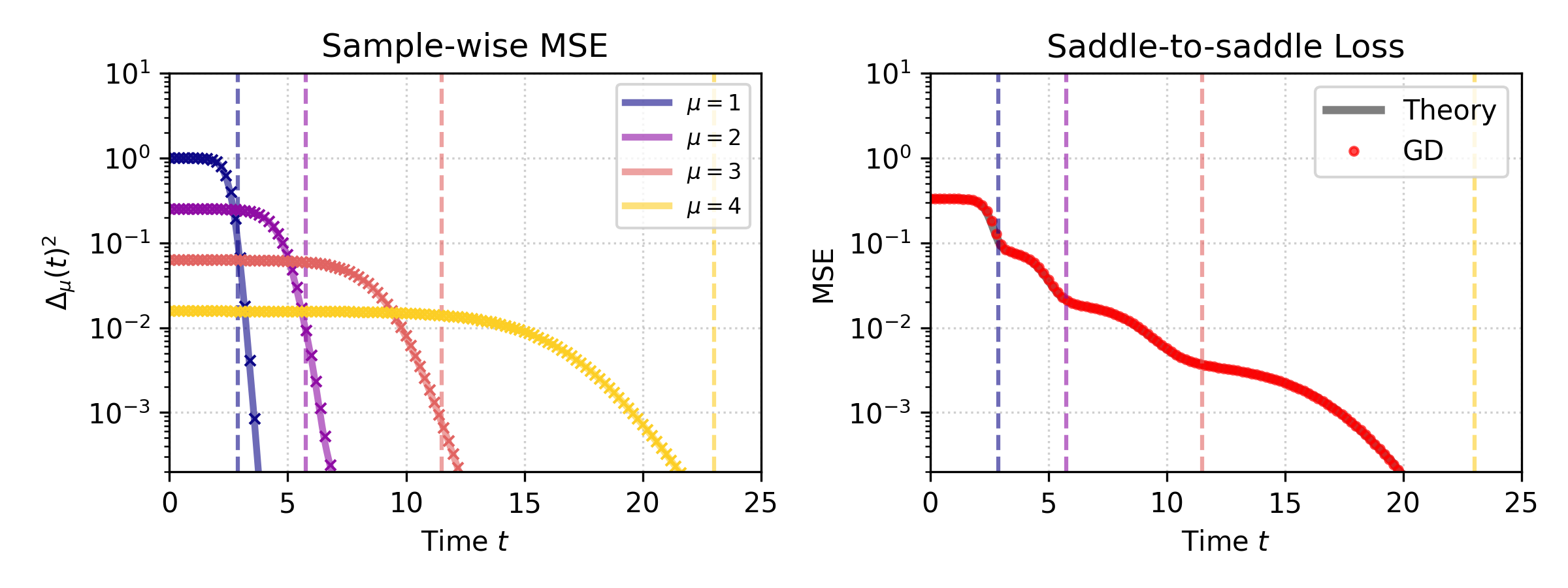}
    \caption{The experiment is under the setting of the sample-wise descent in Section~\ref{sec:saddle-to-saddle}. Solid lines represent the theory of Theorem~\ref{theo:orthogonal_samples} obtained by numerically solving Eq.~\protect\eqref{eq:implicit-solution}, and markers represent the empirical results obtained by full-batch GD. Vertical dashed lines mark the characteristic timescales $t_\mu^* = \zeta_\mu^{-1} \ln \left( 1/ \epsilon \right)$ according to \eqref{eq:staircase-limit}. \textbf{Left}: Evolution of the per-sample loss $\Delta_\mu(t)^2 = (\mathrm{Tr}[WW^\top  A(x_\mu)] - y_\mu)^2$, showing independent and sequential convergence. \textbf{Right}: The total empirical MSE loss over time, which exhibits $m$ sequential plateaus. Each discrete drop in the total loss matches the respective timescale $t_\mu^*$.}
    \label{fig:sample-descent}
\end{figure}

\begin{figure}[t!]
    \centering
    \includegraphics[width=0.75\linewidth]{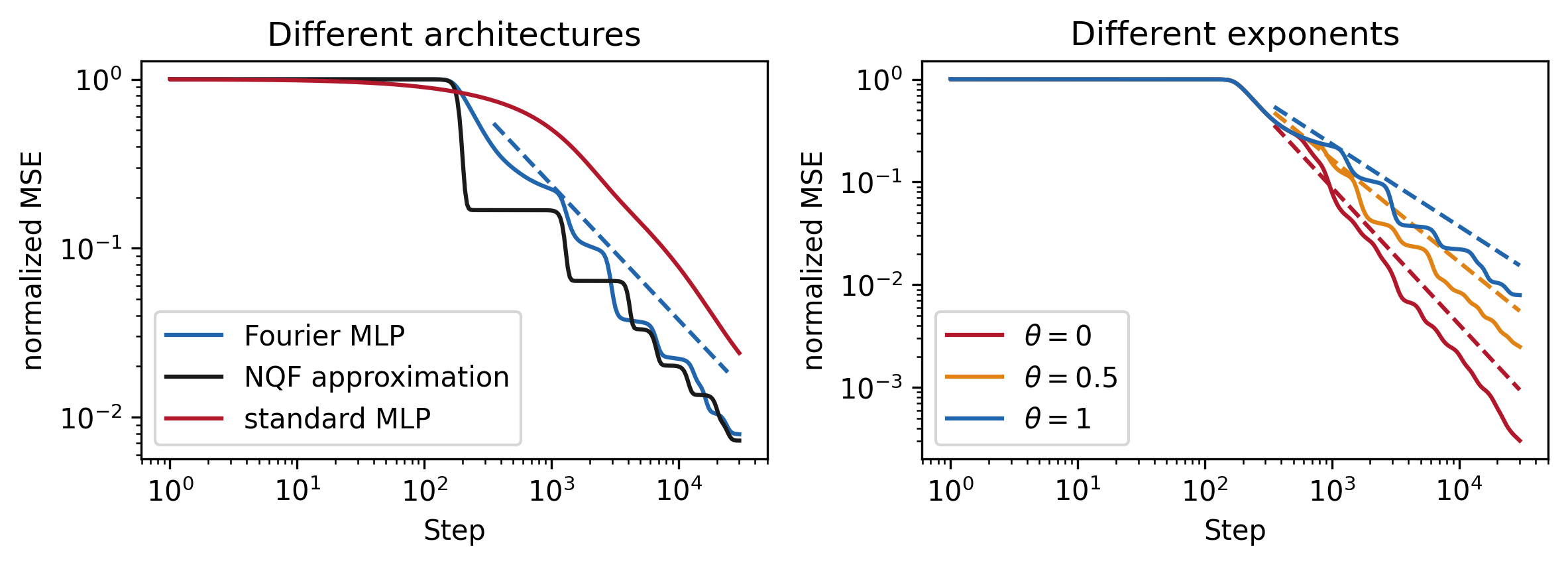}
    \caption{Curves show the normalized MSE $\mathcal{L}(t)/\mathcal{L}(0)$; dotted lines are the slopes predicted in Section \ref{sec:saddle-to-saddle}. \textbf{Left}: For $\theta = 1$, the Fourier MLP (blue) descends as the predicted power law $t^{-(2\beta-1)/(\theta+\beta)}$ (see Appendix \ref{app:exp-details}) and is tracked by its NQF approximation (black); a standard tanh MLP trained on the raw input $x$ (red), which does not exhibit the power law. \textbf{Right}: The power-law exponent of the Fourier MLP (solid lines) matches the theory (dashed lines) for $\theta \in \{0, 0.5, 1\}$.}
    \label{fig:power-law-MLP}
\end{figure}

All the experiments in this section are described in Appendix~\ref{app:exp-details}. The code that reproduces every figure is available at \url{https://github.com/xu-yz19/Neural-Quadratic-Forms}.

\paragraph{NQF approximation} 
To empirically validate the NQF approximation (Theorem \ref{theo:main}), we train a two-layer MLP (Proposition \ref{prop:MLP}), a single-layer CNN (Proposition \ref{prop:CNN}), a query-key-only attention head (Proposition \ref{prop:MHA_variant}) and a multi-head attention model (Proposition \ref{prop:attention}) in Figure \ref{fig:NQF_approx} with different optimizers (gradient descent, Polyak momentum $\beta=0.9$, and Adam). All models use Gaussian initialization with scale $\sigma$. The input is chosen to be Gaussian, and the label is generated by a low-rank NQF. The NQF tracks the original almost exactly across every architecture and optimizer. At large initialization, the NQF visibly departs from the original, showing that the NQF approximation is only valid in the small-initialization regime. In Appendix \ref{app:exp-details}, we further provide additional experiments concerning the NQF approximation for the teacher-student setting and real datasets.

\paragraph{Feature-wise saddle-to-saddle dynamics}
To validate the feature-wise descent in Section \ref{sec:saddle-to-saddle}, we choose $\{A(x_\mu)\}_{\mu=1}^m$ to be diagonal matrices satisfying the condition of Theorem \ref{theo:ortho_features} and train the model using full-batch GD from a small initialization. As shown in Figure \ref{fig:feature-descent}, the numerical trajectories match the predictions of Theorem \ref{theo:ortho_features}. Furthermore, the excess loss exhibits sequential plateaus at the predicted characteristic timescales $t_k^*$. On the right side of Figure \ref{fig:feature-descent}, we further choose $m=1000$, $r_k,V_k$ to follow power laws and confirm that the slope predicted by Section \ref{sec:saddle-to-saddle} matches the curve obtained by GD.

\paragraph{Sample-wise saddle-to-saddle dynamics}
To validate the sample-wise descent in Section \ref{sec:saddle-to-saddle}, we construct a synthetic dataset consisting of $m$ mutually orthogonal data matrices. The student model is trained via full-batch GD from a small initialization scale. As illustrated in Figure \ref{fig:sample-descent}, the numerical trajectories obtained via GD match the analytical curves in Theorem \ref{theo:orthogonal_samples}. Specifically, the left panel shows that the MSE for each individual sample remains stagnant at its initial plateau before undergoing a sudden, sharp decay toward zero near $t_\mu^*$. Consequently, as shown in the right panel, the total MSE loss transitions through $m$ sequential plateaus. 

\paragraph{Power law on MLP} We consider a Fourier MLP
\begin{equation}
f(x) = \frac{1}{\sqrt{P}} \sum_{k=1}^P W_k \tanh(s_k a_k \psi_k(x)),
\label{eq:Fourier-MLP}
\end{equation}
where the trainable parameters are $w_k = [a_k, W_k]^\top$. Here, $\psi_k(x)$ represents orthonormal Fourier features $\psi_k(x) = \sqrt{2} \cos(2\pi k x)$ sampled on a uniform 1D grid $x_\mu \in \{0, 1/M, \dots, (M-1)/M\}$. To induce the power-law structures required by Section \ref{sec:saddle-to-saddle}, we choose $s_k = k^{-\theta}$ and construct a target function $y(x) = \sum_{k=1}^P b_k \psi_k(x)$ with decaying coefficients $b_k = k^{-\beta}$. In Appendix \ref{app:exp-details} we will show that the NQF of this Fourier MLP satisfies the conditions in Section \ref{sec:saddle-to-saddle} (feature-wise descent).

We train the model using full-batch GD. We use the Gaussian initialization with a small scale. We compare this architecture against two baselines: (1) its NQF approximation (Proposition \ref{prop:MLP}); and (2) a standard MLP trained on the raw input $x$. 
The results are shown in Figure \ref{fig:power-law-MLP}, which shows an empirical power-law loss and matches the prediction in Section \ref{sec:saddle-to-saddle}. Meanwhile, the MLP trained on the raw input does not exhibit a power-law loss, which demonstrates the necessity of the structure required by Section \ref{sec:saddle-to-saddle} to attain the power laws. The small systematic steepening of the empirical power law relative to theory is the finite-initialization effect of Appendix \ref{app:exp-details}, which vanishes as $\epsilon \to 0$.

\section{Related Works}
\label{sec:related-works}
\paragraph{Quadratic and factorized model classes.}
Quadratic parameterizations have long served as solvable proxies for nonconvex learning, which includes low-rank matrix factorization \cite{gunasekar2017implicit,arora2019implicit}, phase retrieval \cite{candes2015phase}, diagonal linear networks \cite{berthier2023incremental}, quadratic networks \cite{maillard2024bayes,defilippis2025scaling}, and so on. These models have been used to explain implicit low-rank bias, spectral initialization, and feature learning. A recent work \cite{meterez2026defensequadraticmodel} suggests that quadratic models can even accurately approximate LLMs. NQFs differ from these model-specific analyses by deriving the quadratic features $A(x)$ from the architecture: matrix sensing, diagonal networks, MLPs, CNNs, attention heads, and MoE branches become instances of the same local normal form.

\paragraph{Learning dynamics and feature learning.}
Exact analyses of deep linear networks show rapid transitions between mode-wise plateaus under gradient descent \cite{saxe2014exact}. The neural tangent kernel and lazy-training literature describe the opposite regime, in which a network is well approximated by its linearization around initialization and training is governed by a nearly fixed kernel \cite{jacot2018neural,chizat2019lazy}. Feature-learning work has emphasized that realistic networks change their representation during training \cite{radhakrishnan2022mechanism,beaglehole2023mechanism}. NQFs sit between these perspectives: the model is still a local expansion near small initialization, but the retained quadratic term makes the kernel and representation evolve through the finite-dimensional order parameter $M$, giving closed nonlinear dynamics for feature growth and competition.

\paragraph{Sudden learning and scaling laws.}
Empirical neural scaling laws show that loss often follows power laws in data, model size, or compute \cite{kaplan2020scaling}. Several theoretical works explain such laws through spectral structure: high-eigenvalue modes are learned earlier, and power-law eigenspectra or task-model alignments induce power-law learning curves \cite{bordelon2020spectrum,canatar2021spectral,maloney2022solvable,bahri2024explaining,defilippis2025scaling}. Sudden learning has also been formalized in SGD analyses where target structure is acquired sequentially through saddle-to-saddle transitions \cite{abbe2023sgd}. The NQF perspective connects these two viewpoints: individual features or samples undergo sharp activation events at initialization-dependent times, while a power-law distribution of feature alignments, variances, or sample difficulties aggregates these transitions into a smooth scaling law \cite{michaud2023quantization,nam2024exactly,arous2025learning}.

\section{Discussion}\label{sec: discussion}

We have derived a normal form for the training of a smooth nonlinear neural networks from one structural assumption: that the repeated components of a layer are exchangeable. This property is present in essentially every architecture that has a notion of width. Within that scope, the resulting neural quadratic form is architecture independent, and what distinguishes one architecture from another is the sparsity pattern and the strength of its structure matrix $A$. The learning dynamics of any NQF close on the finite-dimensional order parameters $M$ and $\mu$, they reduce in solvable regimes to generalized Lotka--Volterra flows, and they generate saddle-to-saddle transitions whose composition produces neural scaling laws with predictable exponents.

Read as a Landau construction, the calculation says something specific about universality in deep learning. The dictionary is worth setting down carefully, because two of its entries are not the ones a first reading suggests. The symmetry group is $S_d$; the expansion is truncated at the lowest symmetry-allowed nonconstant order; and the architecture enters only through the coupling $A(x)$, in the same way that microscopic detail enters a Landau functional only through phenomenological coefficients. That much is the standard construction, and it is what makes universality across architectures the expected outcome.

The remaining two entries need more care. The first is the order parameter: $M=WW^\top$. It is $S_d$-invariant by construction, and invariant as well under the larger redundancy $W\mapsto WO$, $O\in\mathrm{O}(d)$, that the quadratic form does not see; it therefore does not transform in a nontrivial representation of the symmetry group, as the order parameter of a textbook Landau construction does. What it shares with the Edwards--Anderson overlap of spin-glass theory, and with the overlaps $w_i^\top w_j$ that serve as the order parameters of online learning in committee machines \cite{engel2001statistical,seung1992statistical}, is that it is an invariant which nonetheless vanishes in one phase and not in the other: $z_k=0$ is the unlearned state and $z_k>0$ the learned one. The symmetry broken at ignition is the freedom of $W$ itself, which the flow on $M$ has quotiented out, and $M$ is the invariant that registers the breaking.

The second is the control parameter $\epsilon$, which is not the small parameter of the expansion. The initialization scale $\epsilon$ does not decide whether a feature is learned, since every mode with $r_k>0$ will be learned at every $\epsilon>0$; what tunes the transition is $r_k$, or $r_k-\gamma$ once weight decay is included, whose sign decides the stability of $z_k=0$ in the way that a reduced temperature decides the stability of the disordered state -- a point we will explore in a future work. What $\epsilon$ controls instead is the sharpness of the crossover, and its natural counterpart is system size: the gaps between successive ignitions grow as $\ln(1/\epsilon)$ while the transitions themselves do not, so $\ln(1/\epsilon)$ plays the role of $N$, and $\epsilon\to0$ is the analogue of the thermodynamic limit. Sudden learning is then the ignition of individual modes of $M$, and the $\epsilon\to0$ limit is a genuine singular limit in which a smooth flow acquires sharp transitions; they are sharp for the same reason that a phase transition is sharp only in an infinite system.

\paragraph{Limitations.}%\label{sec:limitations}
Being a minimal model of training, the model certainly has many limitations, which are worth stating plainly.

\begin{enumerate}[leftmargin=1.4em,itemsep=2pt,topsep=2pt]
\item \textbf{Small initialization.} The NQF is the leading nonconstant term of a local expansion, and it is accurate only when the initialization scale is small compared with the feature length scale set by the data. Figure~\ref{fig:NQF_approx} shows the approximation failing at large initialization, which is the behavior the theory predicts for itself. Appendix~\ref{app:ntk} partially relaxes this by treating the lazy and feature-learning regimes within the same framework.
\item \textbf{Smoothness.} Theorem~\ref{theo:main} assumes three-times differentiability. Rectified activations at the origin, bias terms, normalization layers, and modules that apply a nonlinearity after aggregating their components each violate one of these and require additional treatment.
%\item \textbf{Expressivity is bounded by the neuron dimension.} Theorem~\ref{theo:master} is strong enough to be a limitation: because the dynamics close on a $p\times p$ object, the expressivity of an NQF is set by the parameter dimension of a single component and not by the width.
\item \textbf{Other plateaus.} Our experiments show plateaus that the NQF does not account for, which we attribute to cubic and higher terms. Extending to neural cubic and quartic forms (Appendix~\ref{app:remove_ZGZ}) is the natural next step, as is a systematic treatment of multilayer NQFs, begun in Appendix~\ref{app:multi-layer}.
\item \textbf{Power-law spectra are an assumption, not a prediction.} Deriving $A(x)$ from the architecture fixes the operator whose spectrum controls the scaling exponent. It does not, however, predict that the spectrum has a power-law tail. That tail is a hypothesis of the scaling theorem of Section~\ref{sec:scaling}, and our Fourier-feature experiment constructs it deliberately. Whether it arises generically, for realistic data and architectures, is the question that would have to be settled in future works.
\item \textbf{Adaptive optimizers.} The linear-update corollary of Appendix~\ref{app:linear_update} covers gradient descent, weight decay, and Polyak momentum, but not Adam, which is nonetheless tracked accurately by the NQF in our experiments (Figure~\ref{fig:NQF_approx}).
\end{enumerate}

\paragraph{Outlook.} The correspondences to physics suggest that the technical machinery developed for interacting populations, hydrodynamic instabilities, and polymer flow can be brought to bear on learning dynamics, and conversely that trained networks provide an unusually well-instrumented laboratory for these classical models: every mode amplitude is directly observable, the coupling matrix is known exactly, and the control parameter can be tuned over many decades. Whether the mapping survives beyond quadratic order, and whether the ecological picture of competition between features predicts anything about generalization rather than optimization, are the questions we regard as most worth pursuing.

\section*{acknowledgments}
We gratefully acknowledge the generous support of NTT Research, Inc. I.L.C. acknowledges support in part from the Institute for Artificial Intelligence and Fundamental Interactions (IAIFI) through NSF Grant No.~PHY-2019786. This work was also supported by the Center for Brains, Minds and Machines (CBMM), funded by NSF STC award No.~CCF-1231216.

\bibliographystyle{plain}
\bibliography{ref}

@article{fukumizu1996regularity,
  author  = {Fukumizu, Kenji},
  title   = {A Regularity Condition of the Information Matrix of a Multilayer Perceptron Network},
  journal = {Neural Networks},
  volume  = {9},
  number  = {5},
  pages   = {871--879},
  year    = {1996},
  doi     = {10.1016/0893-6080(95)00152-2}
}

@inproceedings{saxe2014exact,
  author    = {Saxe, Andrew M. and McClelland, James L. and Ganguli, Surya},
  title     = {Exact Solutions to the Nonlinear Dynamics of Learning in Deep Linear Neural Networks},
  booktitle = {International Conference on Learning Representations},
  year      = {2014},
  eprint    = {1312.6120},
  archivePrefix = {arXiv},
  primaryClass  = {cs.NE}
}

@article{candes2015phase,
  author  = {Cand{\`e}s, Emmanuel J. and Li, Xiaodong and Soltanolkotabi, Mahdi},
  title   = {Phase Retrieval via Wirtinger Flow: Theory and Algorithms},
  journal = {IEEE Transactions on Information Theory},
  volume  = {61},
  number  = {4},
  pages   = {1985--2007},
  year    = {2015},
  doi     = {10.1109/TIT.2015.2399924}
}

@inproceedings{bordelon2020spectrum,
  author    = {Bordelon, Blake and Canatar, Abdulkadir and Pehlevan, Cengiz},
  title     = {Spectrum Dependent Learning Curves in Kernel Regression and Wide Neural Networks},
  booktitle = {Proceedings of the 37th International Conference on Machine Learning},
  series    = {Proceedings of Machine Learning Research},
  volume    = {119},
  pages     = {1024--1034},
  year      = {2020},
  publisher = {PMLR}
}

@article{canatar2021spectral,
  author  = {Canatar, Abdulkadir and Bordelon, Blake and Pehlevan, Cengiz},
  title   = {Spectral Bias and Task-Model Alignment Explain Generalization in Kernel Regression and Infinitely Wide Neural Networks},
  journal = {Nature Communications},
  volume  = {12},
  number  = {2914},
  year    = {2021},
  doi     = {10.1038/s41467-021-23103-1}
}

@misc{kaushik2025universalweightsubspacehypothesis,
      title={The Universal Weight Subspace Hypothesis}, 
      author={Prakhar Kaushik and Shravan Chaudhari and Ankit Vaidya and Rama Chellappa and Alan Yuille},
      year={2025},
      eprint={2512.05117},
      archivePrefix={arXiv},
      primaryClass={cs.LG},
      url={https://arxiv.org/abs/2512.05117}, 
}

@article{maloney2022solvable,
  author        = {Maloney, Alexander and Roberts, Daniel A. and Sully, James},
  title         = {A Solvable Model of Neural Scaling Laws},
  journal       = {arXiv preprint arXiv:2210.16859},
  year          = {2022},
  eprint        = {2210.16859},
  archivePrefix = {arXiv},
  primaryClass  = {cs.LG},
  url = {https://arxiv.org/abs/2210.16859}
}

@inproceedings{abbe2023sgd,
  author    = {Abbe, Emmanuel and Boix Adser{\`a}, Enric and Misiakiewicz, Theodor},
  title     = {SGD Learning on Neural Networks: Leap Complexity and Saddle-to-Saddle Dynamics},
  booktitle = {Proceedings of the Thirty Sixth Conference on Learning Theory},
  series    = {Proceedings of Machine Learning Research},
  volume    = {195},
  pages     = {2552--2623},
  year      = {2023},
  publisher = {PMLR}
}

@inproceedings{arora2019implicit,
  author    = {Arora, Sanjeev and Cohen, Nadav and Hu, Wei and Luo, Yuping},
  title     = {Implicit Regularization in Deep Matrix Factorization},
  booktitle = {Advances in Neural Information Processing Systems},
  volume    = {32},
  year      = {2019}
}

@article{berthier2023incremental,
  author  = {Berthier, Rapha{\"e}l},
  title   = {Incremental Learning in Diagonal Linear Networks},
  journal = {Journal of Machine Learning Research},
  volume  = {24},
  number  = {171},
  pages   = {1--26},
  year    = {2023}
}

@inproceedings{zhu2024quadratic,
  author    = {Zhu, Libin and Liu, Chaoyue and Radhakrishnan, Adityanarayanan and Belkin, Mikhail},
  title     = {Quadratic Models for Understanding Catapult Dynamics of Neural Networks},
  booktitle = {International Conference on Learning Representations},
  year      = {2024},
  url       = {https://openreview.net/forum?id=PvJnX3dwsD}
}

@inproceedings{jacot2018neural,
  author    = {Jacot, Arthur and Gabriel, Franck and Hongler, Cl{\'e}ment},
  title     = {Neural Tangent Kernel: Convergence and Generalization in Neural Networks},
  booktitle = {Advances in Neural Information Processing Systems},
  volume    = {31},
  year      = {2018}
}

@inproceedings{chizat2019lazy,
  author    = {Chizat, L{\'e}na{\"i}c and Oyallon, Edouard and Bach, Francis},
  title     = {On Lazy Training in Differentiable Programming},
  booktitle = {Advances in Neural Information Processing Systems},
  volume    = {32},
  year      = {2019}
}

@article{radhakrishnan2022mechanism,
  author       = {Radhakrishnan, Adityanarayanan and Beaglehole, Daniel and Pandit, Parthe and Belkin, Mikhail},
  title        = {Mechanism for feature learning in neural networks and backpropagation-free machine learning models},
  journal      = {Science},
  volume       = {383},
  number       = {6690},
  pages        = {1461--1467},
  year         = {2024},
  doi          = {10.1126/science.adi5639},
  eprint       = {2212.13881},
  archivePrefix= {arXiv},
  primaryClass = {cs.LG}
}

@inproceedings{shazeer2017outrageously,
  author    = {Shazeer, Noam and Mirhoseini, Azalia and Maziarz, Krzysztof and Davis, Andy and Le, Quoc V. and Hinton, Geoffrey E. and Dean, Jeff},
  title     = {Outrageously Large Neural Networks: The Sparsely-Gated Mixture-of-Experts Layer},
  booktitle = {International Conference on Learning Representations},
  year      = {2017}
}

@inproceedings{gunasekar2017implicit,
  title     = {Implicit Regularization in Matrix Factorization},
  author    = {Gunasekar, Suriya and Woodworth, Blake and Bhojanapalli, Srinadh and Neyshabur, Behnam and Srebro, Nathan},
  booktitle = {Advances in Neural Information Processing Systems},
  volume    = {30},
  pages     = {6151--6159},
  year      = {2017},
  publisher = {Curran Associates, Inc.},
  url       = {https://proceedings.neurips.cc/paper_files/paper/2017/hash/58191d2a914c6dae66371c9dcdc91b41-Abstract.html},
  eprint    = {1705.09280},
  archivePrefix = {arXiv},
  primaryClass  = {stat.ML}
}

@inproceedings{ziyin2024parameter,
  title={Parameter symmetry and noise equilibrium of stochastic gradient descent},
  author={Ziyin, Liu and Wang, Mingze and Li, Hongchao and Wu, Lei},
  booktitle={The Thirty-eighth Annual Conference on Neural Information Processing Systems},
  year={2024}
}

@inproceedings{ziyin2024symmetry,
  title={Symmetry Induces Structure and Constraint of Learning},
  author={Ziyin, Liu},
  booktitle={Forty-first International Conference on Machine Learning},
  year={2024}
}

@article{liu2022towards,
  title={Towards understanding grokking: An effective theory of representation learning},
  author={Liu, Ziming and Kitouni, Ouail and Nolte, Niklas S and Michaud, Eric and Tegmark, Max and Williams, Mike},
  journal={Advances in Neural Information Processing Systems},
  volume={35},
  pages={34651--34663},
  year={2022}
}

@article{kaplan2020scaling,
  title={Scaling laws for neural language models},
  author={Kaplan, Jared and McCandlish, Sam and Henighan, Tom and Brown, Tom B and Chess, Benjamin and Child, Rewon and Gray, Scott and Radford, Alec and Wu, Jeffrey and Amodei, Dario},
  journal={arXiv preprint arXiv:2001.08361},
  year={2020},
  url = {https://arxiv.org/abs/2001.08361}
}

@article{bahri2024explaining,
  title={Explaining neural scaling laws},
  author={Bahri, Yasaman and Dyer, Ethan and Kaplan, Jared and Lee, Jaehoon and Sharma, Utkarsh},
  journal={Proceedings of the National Academy of Sciences},
  volume={121},
  number={27},
  pages={e2311878121},
  year={2024},
  publisher={National Academy of Sciences}
}

@inproceedings{huh2024platonic,
  title     = {Position: The Platonic Representation Hypothesis},
  author    = {Huh, Minyoung and Cheung, Brian and Wang, Tongzhou and Isola, Phillip},
  booktitle = {Proceedings of the 41st International Conference on Machine Learning},
  series    = {Proceedings of Machine Learning Research},
  volume    = {235},
  pages     = {20617--20642},
  year      = {2024},
  publisher = {PMLR},
  url       = {https://proceedings.mlr.press/v235/huh24a.html},
  eprint    = {2405.07987},
  archivePrefix = {arXiv},
  primaryClass  = {cs.LG}
}

@article{beaglehole2023mechanism,
  title={Mechanism of feature learning in convolutional neural networks},
  author={Beaglehole, Daniel and Radhakrishnan, Adityanarayanan and Pandit, Parthe and Belkin, Mikhail},
  journal={arXiv preprint arXiv:2309.00570},
  year={2023},
  url = {https://arxiv.org/abs/2309.00570}
}

@inproceedings{bardes2021vicreg,
  title     = {{VICReg}: Variance-Invariance-Covariance Regularization for Self-Supervised Learning},
  author    = {Bardes, Adrien and Ponce, Jean and LeCun, Yann},
  booktitle = {International Conference on Learning Representations},
  year      = {2022},
  eprint    = {2105.04906},
  archivePrefix = {arXiv},
  primaryClass  = {cs.CV}
}

@article{srebro2004maximum,
  title={Maximum-margin matrix factorization},
  author={Srebro, Nathan and Rennie, Jason and Jaakkola, Tommi},
  journal={Advances in neural information processing systems},
  volume={17},
  year={2004}
}

@article{von2021self,
  title={Self-supervised learning with data augmentations provably isolates content from style},
  author={Von K{\"u}gelgen, Julius and Sharma, Yash and Gresele, Luigi and Brendel, Wieland and Sch{\"o}lkopf, Bernhard and Besserve, Michel and Locatello, Francesco},
  journal={Advances in neural information processing systems},
  volume={34},
  pages={16451--16467},
  year={2021}
}

@article{ge2016matrix,
  title={Matrix completion has no spurious local minimum},
  author={Ge, Rong and Lee, Jason D and Ma, Tengyu},
  journal={Advances in neural information processing systems},
  volume={29},
  year={2016}
}

@article{stoger2021small,
  title={Small random initialization is akin to spectral learning: Optimization and generalization guarantees for overparameterized low-rank matrix reconstruction},
  author={St{\"o}ger, Dominik and Soltanolkotabi, Mahdi},
  journal={Advances in Neural Information Processing Systems},
  volume={34},
  pages={23831--23843},
  year={2021}
}

@article{vaswani2017attention,
  title={Attention is all you need},
  author={Vaswani, Ashish and Shazeer, Noam and Parmar, Niki and Uszkoreit, Jakob and Jones, Llion and Gomez, Aidan N and Kaiser, {\L}ukasz and Polosukhin, Illia},
  journal={Advances in neural information processing systems},
  volume={30},
  year={2017}
}

@article{krizhevsky2012imagenet,
  title={Imagenet classification with deep convolutional neural networks},
  author={Krizhevsky, Alex and Sutskever, Ilya and Hinton, Geoffrey E},
  journal={Advances in neural information processing systems},
  volume={25},
  year={2012}
}

@article{koren2009matrix,
  title={Matrix factorization techniques for recommender systems},
  author={Koren, Yehuda and Bell, Robert and Volinsky, Chris},
  journal={Computer},
  volume={42},
  number={8},
  pages={30--37},
  year={2009},
  publisher={IEEE}
}

@book{landau2013statistical,
  title={Statistical Physics: Volume 5},
  author={Landau, Lev Davidovich and Lifshitz, Evgenii Mikhailovich},
  volume={5},
  year={2013},
  publisher={Elsevier}
}

@article{bahri2020statistical,
  title={Statistical mechanics of deep learning},
  author={Bahri, Yasaman and Kadmon, Jonathan and Pennington, Jeffrey and Schoenholz, Sam S and Sohl-Dickstein, Jascha and Ganguli, Surya},
  journal={Annual Review of Condensed Matter Physics},
  volume={11},
  pages={501--528},
  year={2020},
  publisher={Annual Reviews}
}

@article{zdeborova2016statistical,
  title={Statistical physics of inference: Thresholds and algorithms},
  author={Zdeborov{\'a}, Lenka and Krzakala, Florent},
  journal={Advances in Physics},
  volume={65},
  number={5},
  pages={453--552},
  year={2016},
  publisher={Taylor \& Francis}
}

@inproceedings{mei2019mean,
  title={Mean-field theory of two-layers neural networks: dimension-free bounds and kernel limit},
  author={Mei, Song and Misiakiewicz, Theodor and Montanari, Andrea},
  booktitle={Conference on learning theory},
  pages={2388--2464},
  year={2019},
  organization={PMLR}
}

@article{pesme2021implicit,
  title={Implicit bias of sgd for diagonal linear networks: a provable benefit of stochasticity},
  author={Pesme, Scott and Pillaud-Vivien, Loucas and Flammarion, Nicolas},
  journal={Advances in Neural Information Processing Systems},
  volume={34},
  pages={29218--29230},
  year={2021}
}

@inproceedings{even2023s,
  title     = {{(S)GD} over Diagonal Linear Networks: Implicit bias, Large Stepsizes and Edge of Stability},
  author    = {Even, Mathieu and Pesme, Scott and Gunasekar, Suriya and Flammarion, Nicolas},
  booktitle = {Advances in Neural Information Processing Systems},
  volume    = {36},
  year      = {2023},
  url       = {https://proceedings.neurips.cc/paper_files/paper/2023/hash/5da6ce80e97671b70c01a2e703b868b3-Abstract-Conference.html},
  eprint    = {2302.08982},
  archivePrefix = {arXiv},
  primaryClass  = {cs.LG}
}

@inproceedings{frankle2018lottery,
  title     = {The Lottery Ticket Hypothesis: Finding Sparse, Trainable Neural Networks},
  author    = {Frankle, Jonathan and Carbin, Michael},
  booktitle = {International Conference on Learning Representations},
  year      = {2019},
  url       = {https://openreview.net/forum?id=rJl-b3RcF7},
  eprint    = {1803.03635},
  archivePrefix = {arXiv},
  primaryClass  = {cs.LG}
}

@inproceedings{ziyin2022shapes,
  title     = {What shapes the loss landscape of self-supervised learning?},
  author    = {Ziyin, Liu and Lubana, Ekdeep Singh and Ueda, Masahito and Tanaka, Hidenori},
  booktitle = {International Conference on Learning Representations},
  year      = {2023},
  eprint    = {2210.00638},
  archivePrefix = {arXiv},
  primaryClass  = {cs.LG}
}

@inproceedings{mikolov2013distributed,
  title={Distributed Representations of Words and Phrases and their Compositionality},
  author={Mikolov, Tomas and Sutskever, Ilya and Chen, Kai and Corrado, Greg S. and Dean, Jeff},
  booktitle={Advances in Neural Information Processing Systems},
  volume={26},
  year={2013}
}

@misc{goldberg2014word2vec,
  title={word2vec Explained: Deriving Mikolov et al.'s Negative-Sampling Word-Embedding Method},
  author={Goldberg, Yoav and Levy, Omer},
  year={2014},
  eprint={1402.3722},
  archivePrefix={arXiv},
  primaryClass={cs.CL},
  url = {https://arxiv.org/abs/1402.3722}
}

@inproceedings{pan2024dissecting,
  title     = {Dissecting Query-Key Interaction in Vision Transformers},
  author    = {Pan, Xu and Philip, Aaron and Xie, Ziqian and Schwartz, Odelia},
  booktitle = {Advances in Neural Information Processing Systems},
  volume    = {37},
  year      = {2024},
  url       = {https://openreview.net/forum?id=CsF3PwBN6N},
  eprint    = {2405.14880},
  archivePrefix = {arXiv},
  primaryClass  = {cs.CV}
}

@article{maillard2024bayes,
  title={Bayes-optimal learning of an extensive-width neural network from quadratically many samples},
  author={Maillard, Antoine and Troiani, Emanuele and Martin, Simon and Zdeborov{\'a}, Lenka and Krzakala, Florent},
  journal={Advances in Neural Information Processing Systems},
  volume={37},
  pages={82085--82132},
  year={2024}
}

@article{erba2025bilinear,
  title={Bilinear sequence regression: A model for learning from long sequences of high-dimensional tokens},
  author={Erba, Vittorio and Troiani, Emanuele and Biggio, Luca and Maillard, Antoine and Zdeborov{\'a}, Lenka},
  journal={Physical Review X},
  volume={15},
  number={2},
  pages={021092},
  year={2025},
  publisher={APS}
}

@inproceedings{boncoraglio2025single,
  title     = {Single-Head Attention in High Dimensions: A Theory of Generalization, Weights Spectra, and Scaling Laws},
  author    = {Boncoraglio, Fabrizio and Erba, Vittorio and Troiani, Emanuele and Xu, Yizhou and Krzakala, Florent and Zdeborov{\'a}, Lenka},
  booktitle = {Proceedings of the 43rd International Conference on Machine Learning},
  year      = {2026},
  url       = {https://openreview.net/forum?id=3qan4Zg9rA},
  eprint    = {2509.24914},
  archivePrefix = {arXiv},
  primaryClass  = {cs.LG}
}

@inproceedings{defilippis2025scaling,
  title     = {Scaling Laws and Spectra of Shallow Neural Networks in the Feature Learning Regime},
  author    = {Defilippis, Leonardo and Xu, Yizhou and Girardin, Julius and Erba, Vittorio and Troiani, Emanuele and Zdeborov{\'a}, Lenka and Loureiro, Bruno and Krzakala, Florent},
  booktitle = {The Fourteenth International Conference on Learning Representations},
  year      = {2026},
  url       = {https://openreview.net/forum?id=Q3yLIIkt7z},
  eprint    = {2509.24882},
  archivePrefix = {arXiv},
  primaryClass  = {cs.LG}
}

@inproceedings{xu2025fundamental,
  title     = {Fundamental Limits of Matrix Sensing: Exact Asymptotics, Universality, and Applications},
  author    = {Xu, Yizhou and Maillard, Antoine and Zdeborov{\'a}, Lenka and Krzakala, Florent},
  booktitle = {Proceedings of Thirty Eighth Conference on Learning Theory},
  series    = {Proceedings of Machine Learning Research},
  volume    = {291},
  pages     = {5757--5823},
  year      = {2025},
  publisher = {PMLR},
  url       = {https://proceedings.mlr.press/v291/xu25a.html},
  eprint    = {2503.14121},
  archivePrefix = {arXiv},
  primaryClass  = {cs.LG}
}

@article{maillard2020phase,
  title={Phase retrieval in high dimensions: Statistical and computational phase transitions},
  author={Maillard, Antoine and Loureiro, Bruno and Krzakala, Florent and Zdeborov{\'a}, Lenka},
  journal={Advances in Neural Information Processing Systems},
  volume={33},
  pages={11071--11082},
  year={2020}
}

@inproceedings{wang2026universal,
  title={A universal compression theory for lottery ticket hypothesis and neural scaling laws},
  author={Wang, Hong-Yi and Luo, Di and Poggio, Tomaso and Chuang, Isaac L and Ziyin, Liu},
  booktitle={The Fourteenth International Conference on Learning Representations},
  year={2026}
}

@article{ataee2023max,
  title={Max-margin token selection in attention mechanism},
  author={Ataee Tarzanagh, Davoud and Li, Yingcong and Zhang, Xuechen and Oymak, Samet},
  journal={Advances in neural information processing systems},
  volume={36},
  pages={48314--48362},
  year={2023}
}

@inproceedings{wang2025learning,
  title     = {Learning Compositional Functions with Transformers from Easy-to-Hard Data},
  author    = {Wang, Zixuan and Nichani, Eshaan and Bietti, Alberto and Damian, Alex and Hsu, Daniel and Lee, Jason D. and Wu, Denny},
  booktitle = {Proceedings of Thirty Eighth Conference on Learning Theory},
  series    = {Proceedings of Machine Learning Research},
  volume    = {291},
  pages     = {5632--5711},
  year      = {2025},
  publisher = {PMLR},
  url       = {https://proceedings.mlr.press/v291/wang25a.html},
  eprint    = {2505.23683},
  archivePrefix = {arXiv},
  primaryClass  = {cs.LG}
}

@article{maulen2026attention,
  title={Attention-based clustering},
  author={Maulen Soto, Rodrigo and Marion, Pierre and Boyer, Claire},
  journal={Advances in Neural Information Processing Systems},
  volume={38},
  pages={66455--66506},
  year={2025}
}

@article{gur2018gradient,
  title={Gradient descent happens in a tiny subspace},
  author={Gur-Ari, Guy and Roberts, Daniel A and Dyer, Ethan},
  journal={arXiv preprint arXiv:1812.04754},
  year={2018},
  url = {https://arxiv.org/abs/1812.04754}
}

@inproceedings{li2018measuring,
  title     = {Measuring the Intrinsic Dimension of Objective Landscapes},
  author    = {Li, Chunyuan and Farkhoor, Heerad and Liu, Rosanne and Yosinski, Jason},
  booktitle = {International Conference on Learning Representations},
  year      = {2018},
  url       = {https://openreview.net/forum?id=ryup8-WCW},
  eprint    = {1804.08838},
  archivePrefix = {arXiv},
  primaryClass  = {cs.LG}
}

@inproceedings{aghajanyan2021intrinsic,
  title={Intrinsic dimensionality explains the effectiveness of language model fine-tuning},
  author={Aghajanyan, Armen and Gupta, Sonal and Zettlemoyer, Luke},
  booktitle={Proceedings of the 59th annual meeting of the association for computational linguistics and the 11th international joint conference on natural language processing (volume 1: long papers)},
  pages={7319--7328},
  year={2021}
}

@book{pope2001turbulent,
  title     = {Turbulent Flows},
  author    = {Pope, Stephen B.},
  year      = {2000},
  publisher = {Cambridge University Press},
  address   = {Cambridge},
  doi       = {10.1017/CBO9780511840531}
}

@book{hofbauer1998evolutionary,
  title={Evolutionary games and population dynamics},
  author={Hofbauer, Josef and Sigmund, Karl},
  year={1998},
  publisher={Cambridge university press}
}

@book{abou2012matrix,
  title={Matrix Riccati equations in control and systems theory},
  author={Abou-Kandil, Hisham and Freiling, Gerhard and Ionescu, Vlad and Jank, Gerhard},
  year={2012},
  publisher={Birkh{\"a}user}
}

@inproceedings{arous2025learning,
  title     = {Learning quadratic neural networks in high dimensions: {SGD} dynamics and scaling laws},
  author    = {Ben Arous, G{\'e}rard and Erdogdu, Murat A. and Vural, Nuri Mert and Wu, Denny},
  booktitle = {Advances in Neural Information Processing Systems},
  volume    = {38},
  year      = {2025},
  url       = {https://proceedings.neurips.cc/paper_files/paper/2025/file/d7ce06e9293c3d8e6cb3f80b4157f875-Paper-Conference.pdf},
  eprint    = {2508.03688},
  archivePrefix = {arXiv},
  primaryClass  = {cs.LG}
}

@inproceedings{du2018gradient,
  title={Gradient descent learns one-hidden-layer cnn: Don’t be afraid of spurious local minima},
  author={Du, Simon and Lee, Jason and Tian, Yuandong and Singh, Aarti and Poczos, Barnabas},
  booktitle={International Conference on Machine Learning},
  pages={1339--1348},
  year={2018},
  organization={PMLR}
}

@article{cao2019tight,
  title={Tight sample complexity of learning one-hidden-layer convolutional neural networks},
  author={Cao, Yuan and Gu, Quanquan},
  journal={Advances in Neural Information Processing Systems},
  volume={32},
  year={2019}
}

@article{nguyen2024sigmoid,
  title={Sigmoid gating is more sample efficient than softmax gating in mixture of experts},
  author={Nguyen, Huy and Ho, Nhat and Rinaldo, Alessandro},
  journal={Advances in Neural Information Processing Systems},
  volume={37},
  pages={118357--118388},
  year={2024}
}

@article{michaud2023quantization,
  title={The quantization model of neural scaling},
  author={Michaud, Eric and Liu, Ziming and Girit, Uzay and Tegmark, Max},
  journal={Advances in Neural Information Processing Systems},
  volume={36},
  pages={28699--28722},
  year={2023}
}

@article{nam2024exactly,
  title={An exactly solvable model for emergence and scaling laws in the multitask sparse parity problem},
  author={Nam, Yoonsoo and Fonseca, Nayara and Lee, Seok H and Mingard, Chris and Louis, Ard A},
  journal={Advances in Neural Information Processing Systems},
  volume={37},
  pages={39632--39693},
  year={2024}
}

@inproceedings{zhang2025saddle,
  title     = {Saddle-to-Saddle Dynamics Explains a Simplicity Bias Across Neural Network Architectures},
  author    = {Zhang, Yedi and Saxe, Andrew and Latham, Peter E.},
  booktitle = {The Fourteenth International Conference on Learning Representations},
  year      = {2026},
  url       = {https://openreview.net/forum?id=Vit5M0G5Gb},
  eprint    = {2512.20607},
  archivePrefix = {arXiv},
  primaryClass  = {cs.LG}
}

@article{kunin2026alternating,
  title={Alternating gradient flows: A theory of feature learning in two-layer neural networks},
  author={Kunin, Daniel and Marchetti, Giovanni Luca and Chen, Feng and Karkada, Dhruva and Simon, James and Deweese, Michael and Ganguli, Surya and Miolane, Nina},
  journal={Advances in Neural Information Processing Systems},
  volume={38},
  pages={4377--4424},
  year={2025}
}

@inproceedings{xu2025three,
  title     = {Three Mechanisms of Feature Learning in a Linear Network},
  author    = {Xu, Yizhou and Ziyin, Liu},
  booktitle = {International Conference on Learning Representations},
  year      = {2025},
  url       = {https://openreview.net/forum?id=Wh4SE2S7Mo},
  eprint    = {2401.07085},
  archivePrefix = {arXiv},
  primaryClass  = {cs.LG}
}

@misc{meterez2026defensequadraticmodel,
      title={A Defense of the Quadratic Model}, 
      author={Alexandru Meterez and Pranav Ajit Nair and Depen Morwani and Cengiz Pehlevan and Sham Kakade and Alex Damian},
      year={2026},
      eprint={2607.21716},
      archivePrefix={arXiv},
      primaryClass={cs.LG},
      url={https://arxiv.org/abs/2607.21716}, 
}

@book{goldenfeld1992lectures,
  title     = {Lectures on Phase Transitions and the Renormalization Group},
  author    = {Goldenfeld, Nigel},
  year      = {1992},
  publisher = {Addison-Wesley},
  address   = {Reading, MA}
}

@book{chaikin1995principles,
  title     = {Principles of Condensed Matter Physics},
  author    = {Chaikin, Paul M. and Lubensky, Tom C.},
  year      = {1995},
  publisher = {Cambridge University Press},
  address   = {Cambridge}
}

@article{carleo2019machine,
  title   = {Machine learning and the physical sciences},
  author  = {Carleo, Giuseppe and Cirac, Ignacio and Cranmer, Kyle and Daudet, Laurent and Schuld, Maria and Tishby, Naftali and Vogt-Maranto, Leslie and Zdeborov\'a, Lenka},
  journal = {Rev. Mod. Phys.},
  volume  = {91},
  pages   = {045002},
  year    = {2019},
  doi     = {10.1103/RevModPhys.91.045002}
}

@article{seung1992statistical,
  title   = {Statistical mechanics of learning from examples},
  author  = {Seung, H. S. and Sompolinsky, H. and Tishby, N.},
  journal = {Phys. Rev. A},
  volume  = {45},
  pages   = {6056--6091},
  year    = {1992},
  doi     = {10.1103/PhysRevA.45.6056}
}

@book{engel2001statistical,
  title     = {Statistical Mechanics of Learning},
  author    = {Engel, Andreas and Van den Broeck, Christian},
  year      = {2001},
  publisher = {Cambridge University Press},
  address   = {Cambridge}
}

@article{advani2020high,
  title   = {High-dimensional dynamics of generalization error in neural networks},
  author  = {Advani, Madhu S. and Saxe, Andrew M. and Sompolinsky, Haim},
  journal = {Neural Networks},
  volume  = {132},
  pages   = {428--446},
  year    = {2020},
  doi     = {10.1016/j.neunet.2020.08.022}
}

@article{may1972will,
  title   = {Will a large complex system be stable?},
  author  = {May, Robert M.},
  journal = {Nature},
  volume  = {238},
  pages   = {413--414},
  year    = {1972},
  doi     = {10.1038/238413a0}
}

@article{bunin2017ecological,
  title   = {Ecological communities with {L}otka--{V}olterra dynamics},
  author  = {Bunin, Guy},
  journal = {Phys. Rev. E},
  volume  = {95},
  pages   = {042414},
  year    = {2017},
  doi     = {10.1103/PhysRevE.95.042414}
}

@inproceedings{bordelon2024dynamical,
  title     = {A dynamical model of neural scaling laws},
  author    = {Bordelon, Blake and Atanasov, Alexander and Pehlevan, Cengiz},
  booktitle = {Proceedings of the 41st International Conference on Machine Learning},
  series    = {PMLR},
  volume    = {235},
  year      = {2024}
}

@misc{ziyin2025parameter,
  title         = {Parameter symmetry potentially unifies deep learning theory},
  author        = {Ziyin, Liu and Xu, Yizhou and Poggio, Tomaso and Chuang, Isaac},
  year          = {2025},
  eprint        = {2502.05300},
  archivePrefix = {arXiv},
  primaryClass  = {cs.LG},
  url = {https://arxiv.org/abs/2502.05300}
}

@article{wu2019learnability,
  title   = {Learnability for the information bottleneck},
  author  = {Wu, Tailin and Fischer, Ian and Chuang, Isaac L. and Tegmark, Max},
  journal = {Entropy},
  volume  = {21},
  number  = {10},
  pages   = {924},
  year    = {2019},
  doi     = {10.3390/e21100924}
}

@inproceedings{hoffmann2022training,
  title     = {Training compute-optimal large language models},
  author    = {Hoffmann, Jordan and Borgeaud, Sebastian and Mensch, Arthur and Buchatskaya, Elena and Cai, Trevor and Rutherford, Eliza and de Las Casas, Diego and Hendricks, Lisa Anne and Welbl, Johannes and Clark, Aidan and Hennigan, Tom and Noland, Eric and Millican, Katie and van den Driessche, George and Damoc, Bogdan and Guy, Aurelia and Osindero, Simon and Simonyan, Karen and Elsen, Erich and Rae, Jack W. and Vinyals, Oriol and Sifre, Laurent},
  booktitle = {Advances in Neural Information Processing Systems},
  volume    = {35},
  year      = {2022},
  eprint    = {2203.15556},
  archivePrefix = {arXiv}
}
\clearpage

\appendix

\tableofcontents

\begin{table}[htbp]
\caption{Notation used throughout the paper. ``Defined'' gives the section or equation in which the symbol is introduced.}
\label{tab:notation}
\begin{tabular}{@{}llll@{}}
\hline\hline
Symbol & Meaning & Space & Defined\\
\hline
\multicolumn{4}{@{}l}{\emph{Architecture and the normal form}}\\
$d$ & width: number of exchangeable components & $\mathbb{N}$ & Section~\ref{sec:theory}\\
$p$ & parameter dimension of one component & $\mathbb{N}$ & Section~\ref{sec:theory}\\
$w_i$ & the $i$-th component (``neuron'') & $\mathbb{R}^{p}$ & Section~\ref{sec:theory}\\
$W$ & all components, $W=(w_1,\dots,w_d)$ & $\mathbb{R}^{p\times d}$ & Section~\ref{sec:theory}\\
$f_x(W)$ & module output on input $x$ & $\mathbb{R}$ & Section~\ref{sec:theory}\\
$S_d$ & symmetric group acting on the $d$ components & --- & Definition~\ref{def:perm}\\
$A(x)$ & structure matrix & $\mathbb{R}^{p\times p}$ sym. & Eq.~\eqref{eq:NQF}\\
$B(x),\,g(x)$ & extra couplings when ZGZ is dropped & $\mathbb{R}^{p\times p},\mathbb{R}^{p}$ & Eq.~\eqref{eq:NQF-general}\\
$C(x)$ & off-diagonal block of $A$ in a feedforward module & $\mathbb{R}^{d_1\times d_2}$ & Section~\ref{sec:example}\\
$\epsilon$ & initialization scale (small parameter of the expansion) & $\mathbb{R}_{>0}$ & Section~\ref{sec:scaling}\\
\hline
\multicolumn{4}{@{}l}{\emph{Order parameters and dynamics}}\\
$M$ & order parameter $WW^{\top}$, positive semidefinite & $\mathbb{R}^{p\times p}$ & Eq.~\eqref{eq:M_dynamics}\\
$\mu$ & first-moment order parameter $\sum_i w_i$ & $\mathbb{R}^{p}$ & Section~\ref{sec:dynamics}\\
$H$ & residual-weighted structure operator $\frac{4}{m}\sum_\mu\Delta_\mu A(x_\mu)$ & $\mathbb{R}^{p\times p}$ & Eq.~\eqref{eq:W_dynamics}\\
$\eta$ & learning rate & $\mathbb{R}_{>0}$ & Thm.~\ref{theo:master}\\
$\gamma$ & weight-decay strength & $\mathbb{R}_{\ge0}$ & Section~\ref{sec:dynamics}\\
$k_{\mathcal V}$ & dimension of the joint data subspace & $\mathbb{N}$ & Thm.~\ref{theo:compressibility}\\
$K(x,x')$ & empirical neural tangent kernel & $\mathbb{R}$ & App.~\ref{app:ntk}\\
\hline
\multicolumn{4}{@{}l}{\emph{Data and spectra}}\\
$m$ & number of training samples & $\mathbb{N}$ & Section~\ref{sec:solutions}\\
$x_\mu,\,y_\mu$ & $\mu$-th input and target & --- & Section~\ref{sec:solutions}\\
$\Delta_\mu$ & residual $\Tr[MA(x_\mu)]-y_\mu$ & $\mathbb{R}$ & Section~\ref{sec:solutions}\\
$\mathcal{L}$ & sample-averaged MSE loss & $\mathbb{R}_{\ge0}$ & Section~\ref{sec:solutions}\\
$\tilde{\mathcal{L}}$ & unnormalized squared error $\sum_\mu\Delta_\mu^2$ & $\mathbb{R}_{\ge0}$ & Section~\ref{sec:scaling}\\
$\mathcal{E}$ & excess loss $\mathcal{L}(t)-\mathcal{L}(\infty)$ & $\mathbb{R}_{\ge0}$ & Section~\ref{sec:scaling}\\
$P,\,\Lambda_\mu$ & common eigenbasis and eigenvalue matrix of $A(x_\mu)$ & $\mathbb{R}^{p\times p}$ & Asm.~\ref{assume:commutation}\\
$\lambda_{\mu,k}$ & $k$-th eigenvalue of $A(x_\mu)$ & $\mathbb{R}$ & Asm.~\ref{assume:commutation}\\
$z_k$ & $k$-th feature, the $k$-th eigenvalue of $M$ & $\mathbb{R}_{\ge0}$ & Section~\ref{sec:solutions}\\
$r_k$ & growth rate of feature $k$ & $\mathbb{R}$ & Eq.~\eqref{eq:r-and-C}\\
$C_{kj}$ & feature interaction matrix & $\mathbb{R}^{p\times p}$ & Eq.~\eqref{eq:r-and-C}\\
$Y,\,\gamma_k$ & $\frac{4}{m}\sum_\mu y_\mu A(x_\mu)$ and its $k$-th eigenvalue & $\mathbb{R}^{p\times p},\mathbb{R}$ & Thm.~\ref{theo:isotropic_samples}\\
$c$ & isotropy constant of the data second moment & $\mathbb{R}_{>0}$ & Asm.~\ref{assume:isotropic}\\
\hline
\multicolumn{4}{@{}l}{\emph{Sudden learning and scaling laws}}\\
$\zeta_k$ & effective growth rate ($r_k$, or $\gamma_k$ if isotropic) & $\mathbb{R}_{>0}$ & Section~\ref{sec:scaling}\\
$a$ & $1$ under Thm.~\ref{theo:ortho_features}, $2$ under Thm.~\ref{theo:isotropic_samples} & $\{1,2\}$ & Eq.~\eqref{eq:ignition-time}\\
$V_k$ & effective target strength of feature $k$ & $\mathbb{R}_{\ge0}$ & Section~\ref{sec:scaling}\\
$t_k^*$ & ignition time of feature $k$ & $\mathbb{R}_{>0}$ & Eq.~\eqref{eq:ignition-time}\\
$\tau$ & rescaled time $t/\ln(1/\epsilon)$ & $\mathbb{R}_{>0}$ & Section~\ref{sec:scaling}\\
$\alpha_1,\alpha_2$ & decay exponents of $V_k$ and $\zeta_k$ & $\mathbb{R}_{>0}$ & Section~\ref{sec:scaling}\\
$\gamma_1,\gamma_2$ & decay exponents of $y_\mu$ and $\zeta_\mu$ & $\mathbb{R}_{>0}$ & Section~\ref{sec:scaling}\\
$\beta_0$ & decay exponent of the initialization profile & $\mathbb{R}$ & App.~\ref{app:feature-descent}\\
$\beta$ & decay exponent of the target coefficients $b_k$ & $\mathbb{R}$ & Section~\ref{sec:exp}\\
$\theta$ & spectral exponent of the Fourier features & $\mathbb{R}$ & Section~\ref{sec:exp}\\
$\alpha_{\mathrm{eff}}$ & effective exponent at finite $\epsilon$ & $\mathbb{R}$ & App.~\ref{app:exp-details}\\
\hline\hline
\end{tabular}
\end{table}

\section{Theory}
\subsection{List of Notations}
Table \ref{tab:notation} lists all notations used throughout the paper.

\subsection{List of NQF Examples}
\label{app:table}

Table~\ref{tab:NQF_summary} gives the full list of models that reduce to the neural quadratic form, extending Table~\ref{tab:NQF_compact}.

\begin{table*}[t!]
\caption{\small Complete list of models that are special cases of the neural quadratic
form, together with their structure matrices $A(x)$. A six-row excerpt appears as
Table~\ref{tab:NQF_compact} in the main text. Two conventions vary across the rows. The
neurons are the columns of $W$ in most entries, so that $M=WW^\top$; in the tied query--key
and linear SSL rows they are the rows of $W$, so that $M=W^\top W$. 
% And the entry marked
% $\dagger$ carries a structure matrix that depends on the neuron index $i$: it is a
% \emph{generalized} NQF $\sum_i\Tr[w_iw_i^\top A_i(x)]$, for which permutation symmetry is
% broken and the closure of Theorem~\ref{theo:master} on $M$ alone does not hold. It may be
% read instead as a single NQF of width $d=1$ and parameter dimension $\sum_i p_i$ whose
% structure matrix is the block-diagonal assembly of the $A_i$, exactly as the Fourier MLP of
% Appendix~\ref{app:exp-details} is.
}
\label{tab:NQF_summary}
\renewcommand{\arraystretch}{1.2}
{\footnotesize
\begin{tabular}{@{}p{3.3cm}p{3.8cm}p{7.4cm}p{1.3cm}@{}}
\hline\hline
\textbf{Model} & \textbf{NQF notation} & \textbf{Structure matrix $A(x)$} & \textbf{Refs.} \\
\hline
\rr PSD matrix sensing / \newline PSD factorization
& \rr $\hat{y}_a = \langle S_a, WW^\top \rangle$
& \rr $A_a = \mathrm{Sym}(S_a)$; if $S_a = S_a^\top$, $A_a = S_a$.
& \rr \cite{stoger2021small,gunasekar2017implicit,xu2025fundamental} \\
\rr Rectangular matrix sensing / \newline matrix factorization
& \rr $\hat{y}_a = \langle S_a, UV^\top \rangle = \text{Tr}(U^\top S_a V)$
& \rr $A_a = \frac{1}{2} \begin{bmatrix} 0 & S_a \\ S_a^\top & 0 \end{bmatrix}$
& \rr \cite{gunasekar2017implicit,arora2019implicit,xu2025fundamental} \\
\rr Matrix completion / \newline recommender entry
& \rr $\hat{R}_{ij} = u_i^\top v_j$
& \rr $A_{ij} = \frac{1}{2} \begin{bmatrix} 0 & e_i e_j^\top \\ e_j e_i^\top & 0 \end{bmatrix}$
& \rr \cite{koren2009matrix,ge2016matrix} \\
\rr Word-context embedding score
& \rr $s_{wc} = u_w^\top v_c$
& \rr $A_{wc} = \frac{1}{2} \begin{bmatrix} 0 & e_w e_c^\top \\ e_c e_w^\top & 0 \end{bmatrix}$
& \rr \cite{mikolov2013distributed,goldberg2014word2vec} \\
\rr Quadratic network
& \rr $f_W(x) = \sum_r (w_r^\top x)^2$
& \rr $A_x = xx^\top$
& \rr \cite{radhakrishnan2022mechanism,zhu2024quadratic,defilippis2025scaling} \\
\rr Phase retrieval
& \rr $\hat{y}_a = a^\top WW^\top a$
& \rr $A_a = aa^\top$
& \rr \cite{candes2015phase,maillard2020phase} \\
\rr Diagonal linear network, \newline square parametrization
& \rr $f_x(u) = \sum_k x_k u_k^2 / 4$
& \rr $A_x = \frac{1}{4} \mathrm{Diag}(x)$
& \rr \cite{berthier2023incremental,pesme2021implicit} \\
\rr Diagonal linear network, \newline two-factor signed form
& \rr $f_x(u, v) = u^\top \mathrm{Diag}(x) v$
& \rr $A_x = \frac{1}{2} \begin{bmatrix} 0 & \mathrm{Diag}(x) \\ \mathrm{Diag}(x) & 0 \end{bmatrix}$
& \rr \cite{even2023s,pesme2021implicit} \\
\rr Two-layer MLP
& \rr $f_x(W) = \sum_{i=1}^d v_i \phi(u_i^\top x)$
& \rr $A(x) = \frac{\phi'(0)}{2} \begin{bmatrix} 0_{k \times k} & x \\ x^\top & 0 \end{bmatrix}$
& \rr \cite{fukumizu1996regularity,radhakrishnan2022mechanism} \\
\rr Single-layer CNN
& \rr $f_x(W) = \sum_{i=1}^d v_i \sum_{p=1}^P \phi(k_i^\top x_p)$
& \rr $A(X) = \frac{\phi'(0)}{2} \begin{bmatrix} 0_{m \times m} & \sum_{p=1}^P x_p \\ \sum_{p=1}^P x_p^\top & 0 \end{bmatrix}$
& \rr \cite{krizhevsky2012imagenet,du2018gradient,cao2019tight} \\
\rr Mixture of Experts (MoE)
& \rr $f_x(W) = \sum_{i=1}^d \psi(r_i^\top x) E(x; \Theta_i)$
& \rr $A(x) = \frac{\psi'(0)}{2} \begin{bmatrix} 0 & x\,G(x)^\top \\ G(x)\,x^\top & 0 \end{bmatrix}$, with $G(x) := \nabla_{\Theta_i} E(x; 0)$
& \rr \cite{shazeer2017outrageously,nguyen2024sigmoid} \\
\rr Multi-head attention (MHA)
& \rr $f_x(W) = \sum_{i=1}^d v_i^\top W_i^V X \dots$
& \rr $A(X) = \frac{1}{2} \begin{bmatrix} 0 & 0 & 0 & 0 \\ 0 & 0 & 0 & 0 \\ 0 & 0 & 0 & x_{\mathrm{avg}} \otimes I_{d_v} \\ 0 & 0 & x_{\mathrm{avg}}^\top \otimes I_{d_v} & 0 \end{bmatrix}$
& \rr \cite{vaswani2017attention} \\
\rr Query-key-only attention
& \rr $f_x = \sum_{i=1}^d c^\top X \cdot \text{softmax}(\dots)$
& \rr $A(X) = \frac{1}{2\sqrt{d_k}} \begin{bmatrix} 0 & x u^\top \otimes I_{d_k} \\ u x^\top \otimes I_{d_k} & 0 \end{bmatrix}$, \newline where $u = \Sigma_X c$.
& \rr \cite{ataee2023max,pan2024dissecting,maulen2026attention} \\
\rr Tied query-key attention
& \rr $\frac{1}{\sqrt{d_k}} u^\top W^\top W x$
& \rr $A_{x,X,c}^{\text{tied}} = \frac{1}{2\sqrt{d_k}} (x u^\top + u x^\top)$
& \rr \cite{pan2024dissecting,boncoraglio2025single} \\
\rr Bilinear sequence regression
& \rr $x_s^\top B x_t$, with $B = UV^\top$
& \rr $A_{s,t} = \frac{1}{2} \begin{bmatrix} 0 & x_s x_t^\top \\ x_t x_s^\top & 0 \end{bmatrix}$
& \rr \cite{srebro2004maximum,gunasekar2017implicit,erba2025bilinear} \\
% \rr ReLU catapult NQF$^\dagger$ %branch \newline (Zhu et al.)
% & \rr $\frac{1}{\sqrt{md}} \delta v_i \delta u_i^\top x \mathbf{1}\{u_{0,i}^\top x \ge 0\}$
% & \rr $A_{i,x} = \frac{\mathbf{1}\{u_{0,i}^\top x \ge 0\}}{2\sqrt{md}} \begin{bmatrix} 0 & x \\ x^\top & 0 \end{bmatrix}$
% & \rr \cite{lewkowycz2020large,zhu2024quadratic} \\
\rr Linear SSL augmentation
& \rr $\|W\Delta\|^2 = \text{Tr}(W^\top W \Delta \Delta^\top)$
& \rr $A_\Delta = \Delta \Delta^\top$, with $\Delta = x - \chi$
& \rr \cite{von2021self,bardes2021vicreg,ziyin2022shapes} \\
\hline\hline
\end{tabular}}
\renewcommand{\arraystretch}{1}
\end{table*}

\subsection{Proof of Theorem \ref{theo:main}}
Because $f_x$ is three times continuously differentiable, we can write its multi-variable Taylor expansion around the origin $W = (0, \dots, 0)$:
\begin{equation}
f_x(W) = f_x(0) + \sum_{i=1}^d (\nabla_{w_i} f_x(0))^\top  w_i + \frac{1}{2} \sum_{i=1}^d \sum_{j=1}^d w_i^\top  H_{ij} w_j + O(\|W\|^3),
\end{equation}
where $H_{ij} = \frac{\partial^2 f_x}{\partial w_i \partial w_j} \bigg|_{W=0} \in \mathbb{R}^{p \times p}$ represents the blocks of the Hessian matrix.

By the ZGZ condition, we know that
\begin{equation}
\nabla_{w_i} f_x(0) = 0 \quad \text{for all } i.
\end{equation}
Thus, the first-order term $\sum (\nabla_{w_i} f_x(0))^\top  w_i$ vanishes.

Consider the off-diagonal Hessian blocks $H_{ij}$ where $i \neq j$. By definition:
\begin{equation}
H_{ij} = \frac{\partial}{\partial w_j} \left( \nabla_{w_i} f_x(W) \right) \bigg|_{W=0}.
\end{equation}
By the assumption that $\nabla_{w_i} f_x = 0$ as long as $w_i = 0$ irrespective of $w_j$, taking the derivative with respect to $w_j$ yields
\begin{equation}
H_{ij} = 0 \quad \text{for all } i \neq j.
\end{equation}
The cross-terms in the second-order expansion are eliminated.

We are now left with only the diagonal terms of the second-order expansion:
\begin{equation}
f_x(W) = f_x(0) + \frac{1}{2} \sum_{i=1}^d w_i^\top  H_{ii} w_i + O(\|W\|^3)
\end{equation}
By the Permutation Symmetry condition, all neurons are interchangeable. This implies that the curvature of the function with respect to any single neuron $i$, evaluated at the symmetric origin $W=0$, must be identical for all neurons. Therefore:
\begin{equation}
H_{11} = H_{22} = \dots = H_{dd}.
\end{equation}
Let us define a new matrix $A(x) = \frac{1}{2} H_{ii}$. Substituting this into our expansion yields:
\begin{equation}
\begin{aligned}
f_x(W) &= f_x(0) + \sum_{i=1}^d w_i^\top  A(x) w_i + O(\|W\|^3)= f_x(0) + \sum_{i=1}^d \mathrm{Tr}[w_i w_i^\top  A(x)] + O(\|W\|^3).
\end{aligned}
\end{equation}
This finishes the proof.

\subsection{Proof of Proposition \ref{prop:MLP}}
We verify Theorem \ref{theo:main}'s prerequisites sequentially.

1. Smoothness: 
Since the activation function $\phi \in C^3$, $f_x$ is at least three times continuously differentiable.

2. Permutation Symmetry:
Let $\sigma$ be an arbitrary permutation of the indices $\{1, \dots, d\}$.
\begin{equation}
    f_x(w_{\sigma(1)}, \dots, w_{\sigma(d)}) = \sum_{i=1}^d v_{\sigma(i)} \phi(u_{\sigma(i)}^\top  x) = \sum_{i=1}^d v_i \phi(u_i^\top  x) = f_x(w_1, \dots, w_d).
\end{equation}

3. ZGZ:
The gradient with respect to the combined neuron parameter $w_i$ consists of
\begin{align}
    \nabla_{u_i} f_x = v_i \phi'(u_i^\top  x) x,\ 
    \nabla_{v_i} f_x = \phi(u_i^\top  x).
\end{align}
Setting $u_i = w_i=0$ gives
\begin{align}
    \nabla_{u_i} f_x \big|_{u_i=0, v_i=0} = 0 \cdot \phi'(0) x = 0,\ \
    \nabla_{v_i} f_x \big|_{u_i=0, v_i=0} = \phi(0) = 0.
\end{align}
Therefore, the two-layer MLP satisfies all assumptions in Theorem \ref{theo:main}.

Finally let us compute $A(x) = \frac{1}{2} H_{ii}$, where $H_{ii} = \nabla_{w_i}^2 f_x \big|_{w_i=0}$ is the Hessian matrix with respect to the $i$-th neuron's parameters $w_i = [u_i^\top , v_i]^\top $. The Hessian $H_{ii}$ can be written as:
\begin{equation}
    H_{ii} = 
    \begin{bmatrix}
        \nabla_{u_i}^2 f_x & \nabla_{u_i} \nabla_{v_i} f_x \\
        \nabla_{v_i} \nabla_{u_i} f_x & \nabla_{v_i}^2 f_x
    \end{bmatrix}
    \Bigg|_{u_i=0, v_i=0}.
\end{equation}
Top-left block ($\nabla_{u_i}^2 f_x$):
\begin{equation}
    \nabla_{u_i} (\nabla_{u_i} f_x) = \nabla_{u_i} (v_i \phi'(u_i^\top  x) x) = v_i \phi''(u_i^\top  x) x x^\top 
\end{equation}
Evaluating at $w_i=0$ (which sets $v_i=0$), this term vanishes: $\nabla_{u_i}^2 f_x \big|_{0} = 0_{k \times k}$.

Bottom-right block ($\nabla_{v_i}^2 f_x$):
\begin{equation}
    \nabla_{v_i} (\nabla_{v_i} f_x) = \nabla_{v_i} (\phi(u_i^\top  x)) = 0
\end{equation}

Off-diagonal blocks ($\nabla_{v_i} \nabla_{u_i} f_x$):
Taking the derivative of $\nabla_{u_i} f_x$ with respect to $v_i$:
\begin{equation}
    \nabla_{v_i} (v_i \phi'(u_i^\top  x) x) = \phi'(u_i^\top  x) x
\end{equation}
Evaluating at $w_i=0$ (which sets $u_i=0$), we obtain $\phi'(0) x$. 
By the symmetry of second derivatives, the top-right block is its transpose $\phi'(0) x^\top $.

Assembling the blocks yields the Hessian evaluated at the origin:
\begin{equation}
    H_{ii} = 
    \begin{bmatrix}
        0_{k \times k} & \phi'(0) x \\
        \phi'(0) x^\top  & 0
    \end{bmatrix},
\end{equation}
which finishes the proof.

\subsection{Proof of Proposition \ref{prop:attention}}

1. Smoothness: 
The $\mathrm{softmax}$ function is infinitely differentiable everywhere in its domain. Thus, $f_x$ is $C^\infty$.

2. Permutation Symmetry:
Let $\sigma$ be an arbitrary permutation of the head indices $\{1, \dots, d\}$. 
\begin{equation}
    f_x(w_{\sigma(1)}, \dots, w_{\sigma(d)}) = \sum_{i=1}^d \text{Head}_{\sigma(i)}(x, X)=\sum_{i=1}^d \text{Head}_i(x, X) = f_x(w_1, \dots, w_d)
\end{equation}

3. ZGZ:
Let $A_i \in \mathbb{R}^N$ denote the softmax attention weights.
Evaluating the gradients at $w_i = 0$ implies $v_i=0$, $W_i^V=0$, $W_i^Q=0$, and $W_i^K=0$:
\begin{align}
    \nabla_{v_i} f_x \big|_{w_i=0} &= W_i^V X A_i \big|_{W_i^V=0} = 0 \\
    \nabla_{W_i^V} f_x \big|_{w_i=0} &= v_i A_i^\top  X^\top  \big|_{v_i=0} = 0
\end{align}
For the query and key matrices, the gradients involve the derivative of the softmax function, which we denote as a Jacobian matrix $J_{A_i}$. By the chain rule:
\begin{align}
    \nabla_{W_i^Q} f_x = \left( \dots J_{A_i} \dots \right) \times (v_i^\top  W_i^V X) ,\
    \nabla_{W_i^K} f_x = \left( \dots J_{A_i} \dots \right) \times (v_i^\top  W_i^V X).
\end{align}
Because evaluating at $v_i = W_i^V = 0$, $(v_i^\top  W_i^V X)$ becomes zero. Consequently, regardless of the Jacobian of the softmax, the overall gradients for the query and key matrices vanish:
\begin{equation} 
    \nabla_{W_i^Q} f_x \big|_{w_i=0} = 0, \quad \nabla_{W_i^K} f_x \big|_{w_i=0} = 0
\end{equation}
We conclude that $\nabla_{w_i} f_x |_{w_i=0} = 0$. All assumptions in Theorem \ref{theo:main} are satisfied.

Finally let us compute $A(X) = \frac{1}{2} H_{ii}$, where $H_{ii} = \nabla_{w_i}^2 f_x \big|_{w_i=0}$. The output of the $i$-th head is:
\begin{equation}
    h_i(w_i) = v_i^\top  W_i^V X \cdot \mathrm{softmax} \left( \frac{X^\top  (W_i^K)^\top  W_i^Q x}{\sqrt{d_k}} \right)
\end{equation}

Let the pre-softmax logits be $L(w_i) = \frac{X^\top  (W_i^K)^\top  W_i^Q x}{\sqrt{d_k}} \in \mathbb{R}^N$. Notice that $L(w_i)$ is a homogeneous polynomial of degree 2 with respect to $w_i$. The Taylor expansion of the $\mathrm{softmax}$ function around $L = 0$ is:
\begin{equation}
    \mathrm{softmax}(L) = \frac{1}{N} \mathbf{1}_N + J_S(0) L + O(\|L\|^2) = \frac{1}{N} \mathbf{1}_N + O(\|w_i\|^2)
\end{equation}
Now, substitute this back into the head formulation $h_i(w_i)$, we get:
\begin{align}
    h_i(w_i) = v_i^\top  W_i^V X \left( \frac{1}{N} \mathbf{1}_N + O(\|w_i\|^2) \right) = \frac{1}{N} v_i^\top  W_i^V X \mathbf{1}_N + O(\|w_i\|^4).
\end{align}
The only second-order term in the expansion comes from $
    v_i^\top  W_i^V \left( \frac{1}{N} X \mathbf{1}_N \right) = v_i^\top  W_i^V x_{avg}$,
where $x_{avg} = \frac{1}{N} X \mathbf{1}_N$. To express this as a bilinear quadratic form $w_i^\top  A(X) w_i$, we can vectorize the term as:
\begin{equation}
    v_i^\top  W_i^V x_{avg} = v_i^\top  \mathrm{vec}(I_{d_v} W_i^V x_{avg}) = v_i^\top  (x_{avg}^\top  \otimes I_{d_v}) \mathrm{vec}(W_i^V)
\end{equation}
 Placing this in the $4 \times 4$ block matrix corresponding to $Q, K, V, v$ yields the required matrix $A(X)$ and finishes the proof.

\subsection{Proof of Proposition \ref{prop:MHA_variant}}
Similarly to Proposition \ref{prop:attention}, smoothness and permutation symmetry are naturally satisfied.

Let $L = \frac{1}{\sqrt{d_k}} X^\top  (W_i^K)^\top  W_i^Q x \in \mathbb{R}^N$. Evaluating the gradient with respect to $W_i^Q$ at the origin:
\begin{equation}
\nabla_{W_i^Q} f_x \big|_{w_i=0} \propto \frac{\partial L}{\partial W_i^Q} \big|_{W_i^K=0} = 0
\end{equation}
By symmetry, $\nabla_{W_i^K} f_x \big|_{w_i=0} = 0$. Then the ZGZ condition holds.

We compute the second-order term of the Taylor expansion of the $i$-th head around $w_i=0$. The expansion of the softmax function around $L=0$ is:
\begin{equation}
\mathrm{softmax}(L) = \frac{1}{N} \mathbf{1}_N + J_S(0) L + O(|L|^2),
\end{equation}
where $J_S(0) = \frac{1}{N} I_N - \frac{1}{N^2} \mathbf{1}_N \mathbf{1}_N^\top  \in \mathbb{R}^{N \times N}$ is the Jacobian of the softmax function at the origin. Substituting $L$, the head's output becomes:
\begin{equation}
h_i(w_i) = c^\top  X \frac{1}{N} \mathbf{1}_N + c^\top  X J_S(0) \frac{X^\top  (W_i^K)^\top  W_i^Q x}{\sqrt{d_k}} + O(|w_i|^4)
\end{equation}
The first term is a constant term. The second-order term is:
\begin{equation}
E_2 := \frac{1}{\sqrt{d_k}} c^\top  X J_S(0) X^\top  (W_i^K)^\top  W_i^Q x
\end{equation}
Notice that the term $X J_S(0) X^\top $ inside the expression is the sample covariance matrix:
\begin{equation}
X J_S(0) X^\top  = \frac{1}{N} X X^\top  - \left(\frac{1}{N} X \mathbf{1}_N \right) \left(\frac{1}{N} \mathbf{1}_N^\top  X^\top  \right) = \Sigma_X
\end{equation}
Define $u = \Sigma_X c \in \mathbb{R}^D$. We can rewrite the scalar $E_2$ as:
\begin{equation}
E_2 = \frac{1}{\sqrt{d_k}} u^\top  (W_i^K)^\top  W_i^Q x
\end{equation}
Using the Kronecker product property $\mathrm{vec}(A B C) = (C^\top  \otimes A) \mathrm{vec}(B)$, we rewrite $E_2$:
\begin{equation}
\begin{aligned}
E_2 &= \frac{1}{\sqrt{d_k}} (W_i^K u)^\top  (W_i^Q x) = \frac{1}{\sqrt{d_k}} \mathrm{vec}(W_i^K)^\top  (u x^\top  \otimes I_{d_k}) \mathrm{vec}(W_i^Q)
\end{aligned}
\end{equation}
Thus:
\begin{equation}
A(X) = \frac{1}{2 \sqrt{d_k}}
\begin{bmatrix}
0 & x u^\top  \otimes I_{d_k} \\
u x^\top  \otimes I_{d_k} & 0
\end{bmatrix},
\end{equation}
which finishes the proof.

\subsection{Proof of Proposition \ref{prop:SingleHead_Row}}
Smoothness:
The function is $C^\infty$ as it is a composition of linear transformations and the softmax function.

Permutation Symmetry:
Notice that the pre-softmax projection depends on the product $(W^K)^\top  W^Q = \sum_{i=1}^{d_k} w_i^K (w_i^Q)^\top $. Any permutation $\sigma$ of the neuron indices $\{1, \dots, d_k\}$ leaves the matrix product unchanged. Thus the model is invariant to the permutation of its hidden projection dimensions.

ZGZ:
As derived in the last section, the relevant part of the model is:
\begin{equation}
E(W) = \frac{1}{\sqrt{d_k}} u^\top  (W^K)^\top  W^Q x = \frac{1}{\sqrt{d_k}} \sum_{i=1}^{d_k} (u^\top  w_i^K) (x^\top  w_i^Q)
\end{equation}
where $u = \Sigma_X c$. Taking the gradient with respect to the components of the $i$-th neuron $w_i = [w_i^Q, w_i^K]^\top $:
\begin{align}
\nabla_{w_i^Q} f_x &\propto (u^\top  w_i^K) x \\
\nabla_{w_i^K} f_x &\propto (x^\top  w_i^Q) u
\end{align}
Evaluating these gradients at $w_i = 0$ yields zero vectors regardless of the states of $w_j$ ($j \neq i$). The ZGZ condition holds.

Deriving $A(X)$:
We focus on the second-order term for the $i$-th neuron:
\begin{equation}
\frac{1}{\sqrt{d_k}} (u^\top  w_i^K) (x^\top  w_i^Q) = \frac{1}{\sqrt{d_k}} (w_i^K)^\top  u x^\top  w_i^Q
\end{equation}
We need to express this in the quadratic form $ w_i^\top  A(X) w_i$, Which gives:
\begin{equation}
A(X) = \frac{1}{2 \sqrt{d_k}}
\begin{bmatrix}
0 & x u^\top  \\
u x^\top  & 0
\end{bmatrix}
\end{equation}
and finishes the proof.

\subsection{Proof of Proposition \ref{prop:MoE}}
Smoothness:
Assuming $\psi$ and $E$ are at least three times continuously differentiable, their product and sum ensure that the model $f_x$ is $C^3$ with respect to $W$.

Permutation Symmetry:
Let $\sigma$ be an arbitrary permutation of the expert indices $\{1, \dots, d\}$. Evaluating the model with permuted parameters yields:
\begin{equation}
    f_x(w_{\sigma(1)}, \dots, w_{\sigma(d)}) = \sum_{i=1}^d \psi(r_{\sigma(i)}^\top  x) E(x; \Theta_{\sigma(i)}) = \sum_{i=1}^d \psi(r_i^\top  x) E(x; \Theta_i) = f_x(w_1, \dots, w_d).
\end{equation}

ZGZ:
We evaluate the partial gradients with respect to the components of $w_i$ at $w_i = 0$ (which implies $r_i = 0$ and $\Theta_i = 0$).
Taking the partial derivative with respect to $r_i$ yields:
\begin{equation}
    \nabla_{r_i} f_x = \psi'(r_i^\top  x) E(x; \Theta_i) x
\end{equation}
Since $E(x; 0) = 0$, we obtain:
\begin{equation}
    \nabla_{r_i} f_x \Big|_{w_i=0} = \psi'(0) \cdot 0 \cdot x = 0
\end{equation}
Taking the partial derivative with respect to $\Theta_i$ gives:
\begin{equation}
    \nabla_{\Theta_i} f_x = \psi(r_i^\top  x) \nabla_{\Theta_i} E(x; \Theta_i)
\end{equation}
Since $\psi(0) = 0$, we have:
\begin{equation}
    \nabla_{\Theta_i} f_x \Big|_{w_i=0} = \psi(0) \cdot \nabla_{\Theta_i} E(x; \Theta_i) = 0 \cdot \nabla_{\Theta_i} E(x; \Theta_i) = 0
\end{equation}
Thus, the ZGZ condition holds.

The computation of $A(x)$ is similar to that in Proposition \ref{prop:MLP}.

\subsection{Proof of Proposition \ref{prop:CNN}}
The proof is identical to that of Proposition \ref{prop:MLP}.

\subsection{Proof of Theorem \ref{theo:master}}
Let us concatenate the weight vectors $w_i \in \mathbb{R}^p$ into $W = [w_1, w_2, \dots, w_d] \in \mathbb{R}^{p \times d}$.
Let $\mathbf{1} \in \mathbb{R}^d$ be the vector of all ones. With this notation, we can rewrite the summary statistics as $\mu = \sum_{i=1}^d w_i = W \mathbf{1}$ and $M = \sum_{i=1}^d w_i w_i^\top  = W W^\top $. Accordingly, the NQF formulation \eqref{eq:NQF-general} can be rewritten as:
\begin{equation}
f_x(W) = f_x(0) + g(x)^\top  W \mathbf{1} + \mathbf{1}^\top  W^\top  B(x) W \mathbf{1} + \Tr[W^\top  A(x) W].
\end{equation}
Taking the derivative of $f_x(W)$ with respect to $W$, we have: $\nabla_W (g(x)^\top  W \mathbf{1}) = g(x) \mathbf{1}^\top $, $\nabla_W (\mathbf{1}^\top  W^\top  B(x) W \mathbf{1}) = 2 B(x) W \mathbf{1} \mathbf{1}^\top  = 2 B(x) \mu \mathbf{1}^\top $ and $\nabla_W \Tr[W^\top  A(x) W] = 2 A(x) W$. Summing these terms yields the gradient:
\begin{equation}
\nabla_W f_x(W) = \big( g(x) + 2B(x)\mu \big) \mathbf{1}^\top  + 2A(x)W.
\label{eq:sample-gradient}
\end{equation}
By the chain rule, the gradient of the loss $\mathcal{L}$ with respect to $W$ over a mini-batch $\mathcal{B}$ is:
\begin{equation}
\nabla_W \mathcal{L} = \sum_{x \in \mathcal{B}} \ell'_x \nabla_W f_x(W).
\end{equation}
Substituting \eqref{eq:sample-gradient} and using the definitions of $v$ and $H$:
\begin{equation}
\nabla_W \mathcal{L} = v \mathbf{1}^\top  + H W.
\end{equation}
Under SGD, the weight matrix is updated as:
\begin{equation}
\Delta W = -\eta \nabla_W \mathcal{L} = -\eta (v \mathbf{1}^\top  + H W_t).
\end{equation}
We obtain the dynamics for $\mu$:
\begin{equation}
\Delta \mu = \Delta W \mathbf{1} = -\eta (v \mathbf{1}^\top  \mathbf{1} + H W_t \mathbf{1})= -\eta \big( d \cdot v + H \mu_t \big).
\end{equation}
The update $\Delta M$ is computed as:
\begin{equation}
\Delta M = (W + \Delta W)(W + \Delta W)^\top-WW^\top  = W \Delta W^\top  + \Delta W W^\top  + \Delta W \Delta W^\top .
\end{equation}
We evaluate these three components separately: $\Delta W W^\top  = -\eta (v \mathbf{1}^\top  + H W) W^\top  = -\eta (v \mu^\top  + H M)$. $W \Delta W^\top  = (\Delta W W^\top )^\top  = -\eta (\mu v^\top  + M H)$ since $H$ is symmetric. For the second-order term:
\begin{equation}
\Delta W \Delta W^\top  = \eta^2 (v \mathbf{1}^\top  + H W)(\mathbf{1} v^\top  + W^\top  H)= \eta^2 \big( d \cdot v v^\top  + v \mu^\top  H + H \mu v^\top  + H M H \big).
\end{equation}
Summing all three components, we establish:
\begin{equation}
\Delta M = -\eta \big( v \mu^\top  + \mu v^\top  + H M + M H \big) + \eta^2 \big( d \cdot v v^\top  + v \mu_t^\top  H + H \mu v^\top  + H M H \big).
\end{equation}
Therefore, if two NQFs satisfy $M_a(0) = M_b(0)$ and $\mu_a(0) = \mu_b(0)$ at initialization, we have $M_a(t) = M_b(t)$ and $\mu_a(t) = \mu_b(t)$. This completes the proof.

\subsection{Proof of Theorem \ref{theo:compressibility}}
Let us define the subspace $\mathcal{V} \subseteq \mathbb{R}^p$ spanned by the data and $g(x)$, $A(x)$, $B(x)$:
\begin{equation}
\mathcal{V} = \text{span} \bigcup_{x \in \mathcal{X}} \Big( \{g(x)\} \cup \text{Col}(A(x)) \cup \text{Col}(B(x)) \Big).
\end{equation}
Let $k_{\mathcal{V}} = \dim(\mathcal{V}) \le p$, and let $V \in \mathbb{R}^{p \times k}$ be a matrix whose columns form an orthonormal basis for $\mathcal{V}$, where we recall that $p:=\dim w_i$. By definition, for any $x$, we have $V V^\top  g(x) = g(x)$, $V V^\top  A(x) V V^\top  = A(x)$, and $V V^\top  B(x) V V^\top  = B(x)$.

Thus, to match the initial outputs and dynamics, we construct the smaller model within $\mathcal{V}$ such that its projected statistics match the original model:
\begin{equation}
\tilde{\mu}(0) = V V^\top  \mu(0),\ \tilde{M}(0) = \frac{d}{d'} V V^\top  M(0) V V^\top.
\label{eq:match_moments}
\end{equation}

We construct each weight vector of the smaller model as $w'_i = V \left( \frac{1}{d'} V^\top  \mu(0) + u_i \right)$, where $u_1, \dots, u_{d'} \in \mathbb{R}^{k_{\mathcal{V}}}$ satisfies:
\begin{equation}
\sum_{i=1}^{d'} u_i = 0, \quad \text{and} \quad \sum_{i=1}^{d'} u_i u_i^\top  = \frac{d}{d'} \left( V^\top  M(0) V - \frac{1}{d} V^\top  \mu(0)\mu(0)^\top  V \right) =: C_{\mathcal{V}},
\end{equation}
which leads to \eqref{eq:match_moments}. $C_{\mathcal{V}} \in \mathbb{R}^{k \times k}$ is inherently positive semi-definite (PSD) and can be factorized as $C_{\mathcal{V}} = L L^\top $ for some $L \in \mathbb{R}^{k \times k}$.

To find $u_1, \dots, u_{d'}$, consider the all-ones vector $\mathbf{1}_{k_{\mathcal{V}}+1} \in \mathbb{R}^{k_{\mathcal{V}}+1}$. Its orthogonal complement is a $k_{\mathcal{V}}$-dimensional subspace in $\mathbb{R}^{k_{\mathcal{V}}+1}$. Let the rows of a matrix $Z \in \mathbb{R}^{k_{\mathcal{V}} \times (k_{\mathcal{V}}+1)}$ form an orthonormal basis for this subspace, satisfying $Z \mathbf{1}_{k+1} = 0$ and $Z Z^\top  = I_{k_{\mathcal{V}}}$.

We can set $U := [u_1, \dots, u_{k+1}] := L Z$. It follows that:
\begin{equation}
\sum_{i=1}^{k_{\mathcal{V}}+1} u_i = U \mathbf{1}_{k+1} = L Z \mathbf{1}_{k_{\mathcal{V}}+1} = 0, \quad \text{and} \quad U U^\top  = L Z Z^\top  L^\top  = L I_{k_{\mathcal{V}}} L^\top  = C_{\mathcal{V}}.
\end{equation}
Thus, $d' = k_{\mathcal{V}}+1$ neurons are sufficient to match the effective initial statistics.

Assume that at any step $t$, the statistics of the smaller model remain within the subspace $\mathcal{V}$ and match the projected original statistics:
\begin{equation}
\tilde{\mu}_t = V V^\top  \mu_t,\ \tilde{M}_t = \frac{d}{d'} V V^\top  M_t V V^\top 
\label{eq:match_moments_t}
\end{equation}
Due to the subspace definitions (e.g., $V V^\top  g(x) = g(x)$), the output of the smaller model is:
\begin{align}
\tilde{f}_x &= f_x(0) + g(x)^\top  \tilde{\mu}_t + \tilde{\mu}_t^\top  B(x) \tilde{\mu}_t + \Tr[\tilde{M}_t A(x)] \nonumber \\
&= f_x(0) + g(x)^\top  V V^\top  \mu_t + \mu_t^\top  V V^\top  B(x) V V^\top  \mu_t + \Tr\left[ \left(\frac{d}{d'} V V^\top  M_t V V^\top  \right) \left(\frac{d'}{d} A(x) \right) \right] \nonumber \\
&= f_x(0) + g(x)^\top  \mu_t + \mu_t^\top  B(x) \mu_t + \Tr[M_t A(x)] \nonumber \\
&= f_x.
\end{align}
Thus, as long as the statistics maintain this proportional relationship, the outputs are identical. This satisfies the equivalence at initialization ($t=0$).

Because $\tilde{f}_x = f_x$, the loss derivative with respect to the output is identical: $\tilde{\ell}'_x = \ell'_x$. According to Theorem \ref{theo:master}, the gradient variables for the smaller model evaluate to:
\begin{align}
\tilde{v} &= \sum_{x \in \mathcal{B}} \tilde{\ell}'_x \big(\tilde{g}(x) + 2\tilde{B}(x)\tilde{\mu}_t\big) = \sum_{x \in \mathcal{B}} \ell'_x \big(g(x) + 2B(x)\mu_t\big) = v, \\
\tilde{H} &= \sum_{x \in \mathcal{B}} 2\tilde{\ell}'_x \tilde{A}(x) = \frac{d'}{d} \sum_{x \in \mathcal{B}} 2\ell'_x A(x) = \frac{d'}{d} H,
\end{align}
where we use $\tilde{B}(x)\tilde{\mu}_t = B(x)VV^\top \mu_t = B(x)\mu_t$ because $B(x)VV^\top = B(x)$ by definition of $\mathcal{V}$.
since $v$ is a linear combination of $g(x)$ and the columns of $B(x)$, we have $v = V V^\top v$. Similarly, the symmetric matrix $H$ is formed by $A(x)$, meaning $H = V V^\top H V V^\top$.

Applying Theorem \ref{theo:master} to the smaller model, we compute the update $\Delta \tilde{\mu}$:
\begin{align}
\Delta \tilde{\mu}_t &= -\tilde{\eta} \big( d' \cdot \tilde{v} + \tilde{H} \tilde{\mu}_t \big) \nonumber \\
&= -\left(\frac{d}{d'}\eta\right) \left( d' \cdot V V^\top v + \left(\frac{d'}{d} V V^\top H V V^\top \right) V V^\top \mu_t \right) \nonumber \\
&= - \eta V V^\top \big( d \cdot v + H \mu_t \big) = V V^\top \Delta \mu_t.
\end{align}
For the second-moment statistic, we have:
\begin{align}
\Delta \tilde{M}_t &= -\tilde{\eta} \big( \tilde{v} \tilde{\mu}_t^\top + \tilde{\mu}_t \tilde{v}^\top + \tilde{H} \tilde{M}_t + \tilde{M}_t \tilde{H} \big) + \tilde{\eta}^2 \big( d' \cdot \tilde{v} \tilde{v}^\top + \tilde{v} \tilde{\mu}_t^\top \tilde{H} + \tilde{H} \tilde{\mu}_t \tilde{v}^\top + \tilde{H} \tilde{M}_t \tilde{H} \big).
\end{align}
By using \eqref{eq:match_moments_t}, the first-order term becomes:
\begin{equation}
-\left(\frac{d}{d'}\eta\right) V V^\top \big( v \mu_t^\top + \mu_t v^\top + H M_t + M_t H \big) V V^\top. \nonumber
\end{equation}
The second term becomes
\begin{equation}
\left(\frac{d}{d'}\eta\right)^2 \frac{d'}{d} V V^\top \big( d \cdot v v^\top + v \mu_t^\top H + H \mu_t v^\top + H M_t H \big) V V^\top. \nonumber
\end{equation}
Summing these up, we obtain:
\begin{equation}
\Delta \tilde{M}_t = \frac{d}{d'} V V^\top \Delta M_t V V^\top.
\end{equation}
Thus, at step $t+1$, we have $\tilde{\mu}_{t+1} = \tilde{\mu}_t + \Delta \tilde{\mu}_t = V V^\top \mu_{t+1}$ and $\tilde{M}_{t+1} = \frac{d}{d'} V V^\top M_{t+1} V V^\top$. By mathematical induction, these projection relations hold for all training steps $t \ge 0$, and we conclude that $f_{d'} = f_d$ holds throughout training. This completes the proof.

\subsection{Proof of Proposition \ref{prop:decoupling}}
By \eqref{eq:residual_diagonal} and $\dot{\tilde{M}} = -\tilde{H}\tilde{M} - \tilde{M}\tilde{H}$, we have:
\begin{equation}
\frac{d}{dt}\tilde{M}_{ij}= -(\tilde{H}_{ii} + \tilde{H}_{jj}) \tilde{M}_{ij}.
\end{equation}
Because $\tilde{H}_{ii}$ and $\tilde{H}_{jj}$ do not depend on the off-diagonal elements, it can be integrated directly to yield
\begin{equation}
\tilde{M}_{ij}(t) = \tilde{M}_{ij}(0) \exp\left( -\int_0^t  \left( \tilde{H}_{ii}(\tau) + \tilde{H}_{jj}(\tau) \right) d\tau \right).
\end{equation}

\subsection{Proof of Theorem \ref{theo:proportion}}
Recall that the residual for the $\mu$-th sample can be written as:
\begin{equation}
\Delta_\mu(t) = \mathrm{Tr}[W(t)W(t)^\top  c_\mu A] - y_\mu = c_\mu \mathrm{Tr}[M(t)A] - y_\mu
\end{equation}
and
\begin{equation}
H(t) = \frac{4}{m} \sum_{\mu=1}^m (c_\mu \mathrm{Tr}[M(t)A] - y_\mu) c_\mu A = (\alpha \mathrm{Tr}[M(t)A] - \beta) A,
\end{equation}
where we use $\alpha = \frac{4}{m} \sum_{\mu=1}^m c_\mu^2$ and $\beta = \frac{4}{m} \sum_{\mu=1}^m c_\mu y_\mu$. 
According to \eqref{eq:W_dynamics}, the gradient flow for $W(t)$ is:
\begin{equation}
\dot{W}(t) = -H(t)W(t) = - \left( \alpha \mathrm{Tr}[M(t)A] - \beta \right) A W(t)
\end{equation}
To solve this equation, we introduce an auxiliary scalar variable $\xi(t)$ defined as the integrated effective residual:
\begin{equation}
\xi(t) = \int_0^t  \left( \alpha \mathrm{Tr}[M(\tau)A] - \beta \right) d\tau.
\end{equation}
By definition, $\xi(0) = 0$ and $\dot{\xi}(t) = \alpha \mathrm{Tr}[M(t)A] - \beta$. Projecting the parameter matrix into the shared eigenbasis by defining $\tilde{W}(t) = P^\top  W(t)$, we can rewrite the dynamics of $\tilde{W}(t)$ as:
\begin{equation}
\dot{\tilde{W}}(t) = P^\top  \dot{W}(t) = -\dot{\xi}(t) P^\top  A W(t) = -\dot{\xi}(t) \Lambda \tilde{W}(t).
\end{equation}
Since $\Lambda$ is a diagonal matrix, this system of differential equations is decoupled row-by-row. Integrating with respect to time from $0$ to $t$ yields the solution in the projected space:
\begin{equation}
\tilde{W}(t) = \exp(-\xi(t)\Lambda)\tilde{W}(0).
\end{equation}
Multiplying by $P$ on the left, we obtain the closed-form solution for $W$:
\begin{equation}
W(t) = P \exp(-\xi(t) \Lambda) P^\top  W(0)
\end{equation}

Next, we evaluate $M(t) = W(t)W(t)^\top $ to close the system for $\xi(t)$. Projecting $M(t)$ into the eigenbasis yields:
\begin{equation}
\tilde{M}(t) = P^\top  M(t) P = \tilde{W}(t)\tilde{W}(t)^\top  = \exp(-\xi(t)\Lambda) \tilde{M}(0) \exp(-\xi(t)\Lambda),
\end{equation}
Since $\exp(-\xi(t)\Lambda)$ is a diagonal matrix, the diagonal entries of $\tilde{M}(t)$ decouple from the off-diagonal entries and depend only on the diagonal elements of $\tilde{M}(0)$. They evolve as:
\begin{equation}
z_k(t) = \tilde{M}_{kk}(t) = \left( \exp(-\xi(t)\Lambda) \tilde{M}(0) \exp(-\xi(t)\Lambda) \right)_{kk} = z_k(0) \exp(-2 \lambda_k \xi(t))
\end{equation}
The trace term can then be simplified via $\mathrm{Tr}[M(t)A] = \mathrm{Tr}[\tilde{M}(t)\Lambda] = \sum_{j=1}^p \lambda_j z_j(t)$. Substituting the expression of $z_j(t)$ into the derivative of $\xi(t)$ results in a self-contained ODE:
\begin{equation}
\frac{d\xi}{dt} = \alpha \sum_{j=1}^p \lambda_j z_j(0) \exp(-2 \lambda_j \xi) - \beta.
\end{equation}
Separating variables and integrating from $0$ to $\xi(t)$ leads to the implicit equation:
\begin{equation}
t = \int_0^{\xi(t)} \frac{ds}{\alpha \sum_{j=1}^p \lambda_j z_j(0) \exp(-2 \lambda_j s) - \beta}.
\end{equation}
This completes the proof.

\subsection{Proof of Theorem \ref{theo:orthogonal_samples}}
Recall that $\Delta_\mu(t) = \mathrm{Tr}[M(t) A(x_\mu)] - y_\mu$ and
\begin{equation}
H(t) = \frac{4}{m} \sum_{\mu=1}^m \Delta_\mu(t) A(x_\mu) = P \left( \frac{4}{m} \sum_{\mu=1}^m \Delta_\mu(t) \Lambda_\mu \right) P^\top .
\end{equation}
To solve the dynamics $\dot{W}(t) = -H(t) W(t)$, we introduce a set of $m$ independent auxiliary scalar variables, $\xi_\mu(t)$, defined as the integrated effective residual for each sample:
\begin{equation}
\xi_\mu(t) = \int_0^t  \Delta_\mu(\tau) d\tau.
\end{equation}
By definition, $\xi_\mu(0) = 0$ and $\dot{\xi}_\mu(t) = \Delta_\mu(t)$. Projecting the parameter matrix into the shared eigenbasis by defining $\tilde{W}(t) = P^\top  W(t)$, the dynamics decouple into:
\begin{equation}
\dot{\tilde{W}}(t) = -P^\top  H(t) P \tilde{W}(t) = -\left( \frac{4}{m} \sum_{\mu=1}^m \dot{\xi}_\mu(t) \Lambda_\mu \right) \tilde{W}(t).
\end{equation}
By defining the diagonal matrix $\Sigma(t) = \frac{4}{m} \sum_{\mu=1}^m \xi_\mu(t) \Lambda_\mu$, we can integrate this ODE to obtain
\begin{equation}
\tilde{W}(t) = \exp(-\Sigma(t)) \tilde{W}(0).
\end{equation}
Multiplying by $P$ on the left yields the closed-form representation for $W(t)$.

Next, we track the evolution of $M(t) = W(t)W(t)^\top $. Projecting $M(t)$ into the eigenbasis yields:
\begin{equation}
\tilde{M}(t) = P^\top  M(t) P = \exp(-\Sigma(t)) \tilde{M}(0) \exp(-\Sigma(t)).
\end{equation}
Because $\Sigma(t)$ is a diagonal matrix, the diagonal entries of $\tilde{M}(t)$ evolve independently of the off-diagonal elements. The feature magnitudes evolve as:
\begin{equation}
z_k(t) = \tilde{M}_{kk}(t) = z_k(0) \exp\left( -2 \Sigma_{kk}(t) \right) = z_k(0) \exp\left( -\frac{8}{m} \sum_{\nu=1}^m \xi_\nu(t) \lambda_{\nu, k} \right).
\end{equation}
Now we utilize the data orthogonality condition $A(x_\mu)A(x_\nu) = 0$. This implies $\Lambda_\mu \Lambda_\nu = 0$, meaning that for any specific feature index $k$, at most one dataset matrix can have a non-zero eigenvalue. Therefore, if $\lambda_{\mu, k} \neq 0$ for a specific sample $\mu$, then $\lambda_{\nu, k} = 0$ for all $\nu \neq \mu$.

This orthogonality decouples the sum in the exponent. When substituting $z_k(t)$ back into the derivative of the residual $\dot{\xi}_\mu(t)$, we only need to consider the non-zero elements $\lambda_{\mu, k}$:
\begin{equation}
\dot{\xi}_\mu(t) = \mathrm{Tr}[\tilde{M}(t)\Lambda_\mu] - y_\mu = \sum_{k=1}^p \lambda_{\mu, k} z_k(t) - y_\mu.
\end{equation}
Applying the orthogonality property to the exponent of $z_k(t)$, we get:
\begin{equation}
\frac{d\xi_\mu}{dt} = \sum_{k=1}^p \lambda_{\mu, k} z_k(0) \exp\left(-\frac{8}{m} \lambda_{\mu, k} \xi_\mu(t)\right) - y_\mu.
\end{equation}
Thus, the $O(pd)$-dimensional gradient flow reduces to $m$ one-dimensional ODEs.
Separating variables and integrating from $0$ to $\xi_\mu(t)$ provides the implicit solution for each $\mu$ independently:
\begin{equation}
t = \int_0^{\xi_\mu(t)} \frac{ds}{\sum_{k=1}^p \lambda_{\mu, k} z_k(0) \exp\left(-\frac{8}{m} \lambda_{\mu, k} s\right) - y_\mu}.
\end{equation}
This completes the proof.

\subsection{Proof of Theorem \ref{theo:ortho_features}}
We start with \eqref{eq:GLV}:
\begin{equation}
\dot{z}_k(t) = z_k(t) \left[ \left( \frac{8}{m} \sum_{\mu=1}^m \lambda_{\mu, k} y_\mu \right) - \sum_{j=1}^p \left( \frac{8}{m} \sum_{\mu=1}^m \lambda_{\mu, k} \lambda_{\mu, j} \right) z_j(t) \right].
\end{equation}
Using the definitions of $r_k$ and $C_{kj} := \frac{8}{m} \sum_{\mu=1}^m \lambda_{\mu, k} \lambda_{\mu, j}$, the dynamics can be written as:
\begin{equation}
\dot{z}_k(t) = z_k(t) \left( r_k - \sum_{j=1}^p C_{kj} z_j(t) \right).
\end{equation}
Applying the orthogonal features condition, we have $C_{kj} = 0$ for all $k \neq j$. This condition decouples the feature interactions. The ODE simplifies to the classic logistic equation:
\begin{equation}
\dot{z}_k(t) = z_k(t) (r_k - C_{kk} z_k(t)).
\end{equation}
We have $z_k(t)=0$ if $z_k(0)=0$. With the initial condition $z_k(0) > 0$, integrating from $0$ to $t$ yields the analytic solution:
\begin{equation}
z_k(t) = \frac{r_k z_k(0)}{C_{kk} z_k(0) + (r_k - C_{kk} z_k(0)) \exp(-r_k t)}
\end{equation}
if $r_k\neq0$. If $r_k = 0$, the solution is $z_k(t) = z_k(0) / (1 + C_{kk} z_k(0) t)$.

Finally, we reconstruct the parameter matrix $W(t)$. Recall that $\tilde{H}(t) = \frac{4}{m} \sum_{\mu=1}^m \Delta_\mu(t) \Lambda_\mu$ is a diagonal matrix for all $t$. Let $\Sigma(t) = \int_0^t  \tilde{H}(\tau) d\tau$, which is also diagonal. The dynamics in the projected space is $\dot{\tilde{W}}(t) = -\tilde{H}(t)\tilde{W}(t)$, yielding $\tilde{W}(t) = \exp(-\Sigma(t)) \tilde{W}(0)$. Since $\Sigma(t)$ is diagonal, the $k$-th row of $\tilde{W}(t)$ is scaled by the factor $\exp(-\Sigma_{kk}(t))$. We also know from the diagonal elements of $\tilde{M}(t) = \tilde{W}(t)\tilde{W}(t)^\top$ that $z_k(t) = z_k(0) \exp(-2\Sigma_{kk}(t))$. For any feature with $z_k(0) > 0$, we have $\exp(-\Sigma_{kk}(t)) = \sqrt{z_k(t) / z_k(0)}$. For a feature with $z_k(0) = 0$, the $k$-th row of $\tilde{W}(0)$ is a zero vector, which ensures that the $k$-th row of $\tilde{W}(t)$ remains zero for all $t \ge 0$. Thus, defining $d_k(t) = \sqrt{z_k(t) / z_k(0)}$ for $z_k(0) \neq 0$ and $d_k(t) = 0$ otherwise consistently captures the row-wise scaling: $\tilde{W}(t) = D(t) \tilde{W}(0)$. Multiplying by $P$ on the left, we arrive at the trajectory for $W$:
\begin{equation}
W(t) = P D(t) P^\top  W(0).
\end{equation}
This completes the proof.

\subsection{Proof of Theorem \ref{theo:isotropic_samples}}
Recall that the gradient flow of the parameters is $\dot{W}(t) = -H(t) W(t)$, where $H(t) = \frac{4}{m} \sum_{\mu=1}^m \Delta_\mu(t) A(x_\mu)$. By expanding the residual $\Delta_\mu(t) = \mathrm{Tr}[M(t) A(x_\mu)] - y_\mu$ and isolating the label-dependent term $Y$, we have:
\begin{equation}
H(t) = 4 \sum_{\mu=1}^m \frac{1}{m} \mathrm{Tr}[M(t) A(x_\mu)] A(x_\mu) - Y.
\end{equation}
We evaluate the first term element-wise:
\begin{equation}
\left( 4 \sum_{\mu=1}^m \frac{1}{m} \mathrm{Tr}[M(t) A(x_\mu)] A(x_\mu) \right)_{ij} = 4 \sum_{k=1}^p \sum_{l=1}^p \left( \frac{1}{m} \sum_{\mu=1}^m A_{ij}(x_\mu) A_{kl}(x_\mu) \right) M_{kl}(t).
\end{equation}
Applying Assumption \ref{assume:isotropic}, this expression simplifies to:
\begin{equation}
4 \sum_{k=1}^p \sum_{l=1}^p \frac{c}{2} (\delta_{ik}\delta_{jl} + \delta_{il}\delta_{jk}) M_{kl}(t) = 2c \left( M_{ij}(t) + M_{ji}(t) \right) = 4c M_{ij}(t),
\end{equation}
where we utilized $M(t) = M(t)^\top $. Thus, we have $H(t) = 4c M(t) - Y$.
Substituting $H(t)$ into $\dot{M}(t) = -H(t) M(t) - M(t) H(t)$, we obtain a matrix Riccati differential equation \cite{abou2012matrix}:
\begin{equation}
\label{eq:riccati}
\dot{M} = -(4cM - Y)M - M(4cM - Y) = Y M + M Y - 8c M^2.
\end{equation}
To solve this non-linear matrix equation, we introduce $M(t) = V(t) U(t)^{-1}$, where $V(0) = M(0)$ and $U(0) = I$. Differentiating this expression yields:
\begin{equation}
\dot{M} = \dot{V} U^{-1} - V U^{-1} \dot{U} U^{-1} = \dot{V} U^{-1} - M \dot{U} U^{-1}.
\end{equation}
By matching the terms with the Riccati equation $\dot{M} = Y V U^{-1} - M (-Y U + 8c V) U^{-1}$, the dynamics decouple into a system of two linear ordinary differential equations:
\begin{align}
\dot{V} &= Y V, \label{eq:V_dot} \\
\dot{U} &= -Y U + 8c V. \label{eq:U_dot}
\end{align}
The solution to \eqref{eq:V_dot} is given by $V(t) = \exp(Y t) M(0)$. Substituting $V(t)$ into \eqref{eq:U_dot} gives:
\begin{equation}
\dot{U} + Y U = 8c \exp(Y t) M(0).
\end{equation}
Multiplying both sides by the integrating factor $\exp(Y t)$ leads to $\frac{d}{dt} \left[ \exp(Y t) U(t) \right] = 8c \exp(2 Y t) M(0)$. Integrating from $0$ to $t$ and using $U(0) = I$, we get:
\begin{equation}
\exp(Y t) U(t) - I = 8c \left( \int_0^t\exp(2Y\tau) d\tau \right) M(0) = 8c \Phi(t) M(0).
\end{equation}
Solving for $U(t)$, we find $U(t) = \exp(-Y t) \left[ I + 8c \Phi(t) M(0) \right]$, which is always invertible because $\Phi(t) := \int_0^t  \exp(2Y\tau) d\tau$ is positive definite for $t > 0$, and $M(0) = W(0)W(0)^\top $ is positive semi-definite. Finally, reconstructing $M(t) = V(t) U(t)^{-1}$, we arrive at:
\begin{equation}
M(t) = \exp(Y t) M(0) \left[ I + 8c \Phi(t) M(0) \right]^{-1} \exp(Y t).
\label{eq:solution-matrix-ricci}
\end{equation}
Furthermore, suppose $Y$ and $M(0)$ are simultaneously diagonalizable. Then, there exists an orthogonal matrix $P$ such that $Y = P \Gamma P^\top $ and $M(0) = P Z(0) P^\top $, where $\Gamma$ and $Z(0)$ are diagonal matrices containing the eigenvalues $\gamma_k$ and $z_k(0)$, respectively. Since all matrices in the closed-form solution share the same eigenbasis, $M(t)$ remains diagonal in this basis: $M(t) = P Z(t) P^\top $.

For the $k$-th eigen-mode, the matrix functions reduce to scalar functions. The integral evaluates to $\Phi_k(t) = \frac{\exp(2\gamma_k t) - 1}{2\gamma_k}$ for $\gamma_k \neq 0$. Substituting these into \eqref{eq:solution-matrix-ricci} gives:
\begin{equation}
z_k(t) = \frac{\exp(2\gamma_k t) z_k(0)}{1 + 8c \left( \frac{\exp(2\gamma_k t) - 1}{2\gamma_k} \right) z_k(0)}.
\end{equation}
Rearranging the terms, we recover the standard generalized logistic growth curve:
\begin{equation}
z_k(t) = \frac{\gamma_k z_k(0)}{4c z_k(0) + (\gamma_k - 4c z_k(0)) \exp(-2\gamma_k t)}.
\end{equation}
For the case where $\gamma_k = 0$, the integral simplifies to $\Phi_k(t) = \int_0^t 1 d\tau = t$. Substituting this yields:
\begin{equation}
z_k(t) = \frac{z_k(0)}{1 + 8c z_k(0) t}.
\end{equation}
This completes the proof.

\subsection{Formal Statements of Section \ref{sec:saddle-to-saddle}}
\label{app:feature-descent}
\paragraph{Feature-wise Descent}
The first part of the formal statement is the following theorem.
\begin{theorem}
\label{theo:saddle_to_saddle}
Under the conditions of Theorems \ref{theo:ortho_features} (or Theorem~\ref{theo:isotropic_samples} with $M(0)$ and $Y$ simultaneous diagonalizable), let $\zeta_1 > \zeta_2 > \dots > \zeta_p > 0$, where we define $\zeta_k = r_k$ in Theorem \ref{theo:ortho_features}  ($\zeta_k=\gamma_k$ in Theorem~\ref{theo:isotropic_samples}).

Then, for any fixed constant $\rho \in (0, 1)$, let $t_k^*$ denote the time that $z_k(t_k^*)=\rho z_k(\infty)$ with $z_k(\infty) = \lim_{t \to \infty} z_k(t)$. Let the initialization scales satisfy $z_k(0) = \Theta(\epsilon)$ for all $k \in \{1, \dots, p\}$, then as $\epsilon \to 0^+$,
\begin{equation}
t_k^*\sim\frac{1}{a\zeta_k} \ln\frac{1}{\epsilon},
\end{equation}
where $a=1$ in Theorem \ref{theo:ortho_features} ($a=2$ in Theorem~\ref{theo:isotropic_samples}).
\end{theorem}
Consequently, the gaps between consecutive feature activations diverge in the limit of small initialization, which exhibits saddle-to-saddle dynamics: 
\begin{equation}
\lim_{\epsilon \to 0} (t_{k+1}^* - t_k^*) = \infty.
\end{equation}
\begin{proof}
 Let the initialization scale be $z_k(0) = c_k \epsilon$ for all $k$, where $c_k >0$ is a constant and $\epsilon \to 0$.

By Theorems \ref{theo:ortho_features} and \ref{theo:isotropic_samples}, $z_k$ evolves according to
\begin{equation}
z_k(t) = \frac{\zeta_k z_k(0)}{C_k z_k(0) + (\zeta_k - C_k z_k(0)) \exp(-a \zeta_k t)}
\label{eq:logistic}
\end{equation}
where $\zeta_k = r_k, a=1, C_k = C_{kk}$ for Theorem \ref{theo:ortho_features}, and $\zeta_k = \gamma_k, a=2, C_k = 4c$ for Theorem \ref{theo:isotropic_samples}. The non-zero steady-state saturation value is $z_k(\infty) = \zeta_k / C_k$.

Setting $z_k(t_k^*) = \rho z_k(\infty) = \rho \frac{\zeta_k}{C_k}$, we substitute the initial condition and solve for $t_k^*$:
\begin{equation}
\rho \frac{\zeta_k}{C_k} = \frac{\zeta_k c_k \epsilon}{C_k c_k \epsilon + (\zeta_k - C_k c_k \epsilon) \exp(-a \zeta_k t_k^*)}
\end{equation}
Rearranging to isolate the exponential term yields:
\begin{equation}
\exp(-a \zeta_k t_k^*) = \frac{C_k c_k \epsilon}{\zeta_k - C_k c_k \epsilon} \left( \frac{1 - \rho}{\rho} \right)
\end{equation}
Taking the logarithm and dividing by $-a \zeta_k$, we obtain the characteristic time:
\begin{equation}
t_k^* = \frac{1}{a \zeta_k} \ln\left( \frac{1}{\epsilon} \right) - \frac{1}{a \zeta_k} \ln\left( \frac{C_k c_k}{\zeta_k - C_k c_k \epsilon} \frac{1 - \rho}{\rho} \right)
\end{equation}
As $\epsilon \to 0$, the second term is $O(1)$, and the time scales as:
\begin{equation}
t_k^* = \frac{1}{a \zeta_k} \ln\left(\frac{1}{\epsilon}\right) + O(1) \sim \frac{1}{a\zeta_k} \ln\frac{1}{\epsilon}.
\end{equation}
This completes the proof.
\end{proof}

Recall that $r_k$ and $\gamma_k$ are essentially the correlation between input features and label, and $C_{kk}\geq 0$ is essentially the input variances. The second part of the formal statement is the following.

\begin{theorem}
\label{theo:feature-scaling-law}
Suppose the conditions of Theorem \ref{theo:saddle_to_saddle} hold. Consider the infinite-width limit $p \to \infty$\footnote{Note that under Theorem \ref{theo:isotropic_samples} this requires $m\to\infty$ simultaneously to satisfy Assumption \ref{assume:isotropic}} and assume the initialization satisfies $z_k(0) = \epsilon c_kk^{-\beta}$ where the constants $c_k$ are bounded away from zero: $0 < c_{\min} \le c_k \le c_{\max} < \infty$. Let $\mathcal{E}_\epsilon(t) := L(t) - L(\infty)$ denote the excess loss. Define $V_k := \frac{r_k^2}{8 C_{kk}}$ (under Theorem \ref{theo:ortho_features}) or $V_k := \frac{\gamma_k^2}{16c}$ (under Theorem \ref{theo:isotropic_samples}).

Assume that $\zeta_k$ and $V_k$ exhibit the following power-law decays:
\begin{align}
\zeta_k = c_\zeta k^{-\alpha_2}, \
V_k = c_w k^{-\alpha_1}
\end{align}
for positive constants $c_\zeta, c_w, \alpha_2 > 0$. Furthermore, assume that $\alpha_1 > 1$ \footnote{Note that under Theorem \ref{theo:isotropic_samples}, $V_k = \frac{\zeta_k^2}{16c}$, which enforces $\alpha_1 = 2\alpha_2$. Thus, we require $\alpha_2 > 0.5$ in that regime.} and $\beta \ge \alpha_1 - \alpha_2$  to ensure finite initial excess loss.

 Define the rescaled time $\tau := \frac{t}{\ln(1/\epsilon)}$. Then, under the small initialization limit $\epsilon \to 0^+$, the excess loss converges $\mathcal{E}(\tau) := \lim_{\epsilon \to 0^+} \mathcal{E}_\epsilon(\tau\ln(1/\epsilon))$ for almost every $\tau$ and $\mathcal{E}(\tau)$ has a power-law decay
\begin{equation}
\mathcal{E}(\tau) = \Theta\left( \tau^{-\frac{\alpha_1 - 1}{\alpha_2}} \right).
\end{equation}
\end{theorem}
\begin{proof}
We first evaluate the empirical loss $L(t)$ in terms of the features $z_k(t)$.

In Theorem \ref{theo:ortho_features}, the loss is $L(t) = \frac{1}{m} \sum_{\mu=1}^m \left( \sum_{k=1}^p \lambda_{\mu, k} z_k(t) - y_\mu \right)^2$. Expanding the square and applying the orthogonality condition $\sum_\mu \lambda_{\mu, j} \lambda_{\mu, l} = 0$ for $j \neq l$, the cross-terms vanish. Using the definitions $C_{kk} = \frac{8}{m} \sum_\mu \lambda_{\mu, k}^2$ and $r_k = \frac{8}{m} \sum_\mu \lambda_{\mu, k} y_\mu$, we obtain:
\begin{equation}
L(t) = \frac{1}{8} \sum_{k=1}^p C_{kk} \left( z_k(t) - \frac{r_k}{C_{kk}} \right)^2+\text{Const}.
\end{equation}
Since $z_k(\infty) = r_k / C_{kk}$, the excess loss is $\mathcal{E}_\epsilon(t) := L(t) - L(\infty) = \frac{1}{8} \sum_{k=1}^p C_{kk} \left( z_k(t) - z_k(\infty) \right)^2$.

Similarly, in Theorem \ref{theo:isotropic_samples}, expanding $L(t) = \frac{1}{m} \sum_{\mu=1}^m \left( \mathrm{Tr}[M(t) A(x_\mu)] - y_\mu \right)^2$ and utilizing Assumption \ref{assume:isotropic} yields $L(t) = c \mathrm{Tr}[M(t)^2] - \frac{1}{2} \mathrm{Tr}[M(t) Y]+\text{Const}$. Because $M(t)$ and $Y$ are simultaneously diagonalizable, this simplifies to $L(t) = c \sum_{k=1}^p \left( z_k(t) - \frac{\gamma_k}{4c} \right)^2+\text{Const}$. With $z_k(\infty) = \gamma_k / 4c$, the excess loss is $\mathcal{E}_\epsilon(t) = c \sum_{k=1}^p \left( z_k(t) - z_k(\infty) \right)^2$.

Therefore, in both regimes the excess loss at any time $t$ can be written as
\begin{equation}
\mathcal{E}_\epsilon(t) = \sum_{k=1}^\infty V_k \left( 1 - \frac{z_k(t)}{z_k(\infty)} \right)^2.
\label{eq:excess_loss_series}
\end{equation}
after taking $p\to\infty$. $\beta\ge\alpha_1-\alpha_2$ guarantees that $(1 - z_k(0)/z_k(\infty))^2 \to 1$ because
\begin{equation}
\frac{z_k(0)}{z_k(\infty)} = \frac{c_k' \epsilon k^{-\beta} \zeta_k}{V_k} = \frac{c_k' \epsilon k^{-\beta} c_\zeta k^{-\alpha_2}}{c_w k^{-\alpha_1}} = \frac{c_k' c_\zeta \epsilon}{c_w} k^{\alpha_1 - \alpha_2 - \beta},
\end{equation}
where $c_k':=\frac{c_k}{8}$ in Theorem \ref{theo:ortho_features} and $c_k':=\frac{c_k}{4}$ in Theorem \ref{theo:isotropic_samples}. 
Further $\alpha_1 > 1$ guarantees that at $t=0$ the excess loss $\sum_{k=1}^\infty V_k$ is finite.

We introduce the rescaled time $\tau = \frac{t}{\ln(1/\epsilon)}$, or equivalently $t = \tau \ln(1/\epsilon)$. The fraction of unlearned feature magnitude, denoted as $f_{k, \epsilon}(\tau) = 1 - \frac{z_k(\tau \ln(1/\epsilon))}{z_k(\infty)}$, can be evaluated by substituting \eqref{eq:logistic}:
\begin{equation}
\frac{z_k(t)}{z_k(\infty)} = \frac{\zeta_k c_k \epsilon k^{-\beta} / C_k}{\zeta_k c_k \epsilon k^{-\beta} / C_k + (\zeta_k / C_k - c_k \epsilon k^{-\beta}) \exp(-a \zeta_k t)} = \frac{1}{1 + \left( \frac{\zeta_k}{c_k k^{-\beta} C_k \epsilon} - 1 \right) \exp(-a \zeta_k \tau \ln(1/\epsilon))}.
\end{equation}
Let $\tilde{C}_k(\epsilon) := \frac{\zeta_k}{c_k k^{-\beta} C_k \epsilon} - 1$. For sufficiently small $\epsilon$, $\tilde{C}_k(\epsilon) > 0$. The fraction becomes:
\begin{equation}
f_{k, \epsilon}(\tau) = 1 - \frac{1}{1 + \tilde{C}_k(\epsilon) \epsilon^{a \zeta_k \tau}} = \frac{\tilde{C}_k(\epsilon) \epsilon^{a \zeta_k \tau}}{1 + \tilde{C}_k(\epsilon) \epsilon^{a \zeta_k \tau}}.
\end{equation}
Notice that the leading order behavior of the term $\tilde{C}_k(\epsilon) \epsilon^{a \zeta_k \tau}$ is governed by $\epsilon^{a \zeta_k \tau - 1}$.
Let $\tau_k^* = \frac{1}{a \zeta_k}$. We analyze the pointwise limit of $f_{k, \epsilon}(\tau)$ as $\epsilon \to 0^+$:
\begin{itemize}
\item If $\tau > \tau_k^*$, then $a \zeta_k \tau - 1 > 0$. As $\epsilon \to 0^+$, $\tilde{C}_k(\epsilon) \epsilon^{a \zeta_k \tau} \sim \epsilon^{a \zeta_k \tau - 1} \to 0$. Hence, $f_{k, \epsilon}(\tau) \to 0$.
\item If $\tau < \tau_k^*$, then $a \zeta_k \tau - 1 < 0$. As $\epsilon \to 0^+$, $\tilde{C}_k(\epsilon) \epsilon^{a \zeta_k \tau} \sim \epsilon^{a \zeta_k \tau - 1} \to \infty$. Hence, $f_{k, \epsilon}(\tau) \to 1$.
\end{itemize}
Therefore, for almost every $\tau$, $f_{k, \epsilon}$ converges to an indicator function: $\lim_{\epsilon \to 0^+} f_{k, \epsilon}(\tau) = \mathbb{I}(\tau < \tau_k^*)$.

Notice that for both Theorem \ref{theo:ortho_features} and Theorem \ref{theo:isotropic_samples}, the relationship $C_k = \Theta(\zeta_k^2 / V_k) = \Theta(k^{\alpha_1 - 2\alpha_2})$ holds (with $\alpha_1 = 2\alpha_2$ specifically under Theorem \ref{theo:isotropic_samples}). Substituting this scaling into $\tilde{C}_k(\epsilon)$ yields:
\begin{equation}
\tilde{C}_k(\epsilon) \ge \frac{\text{Const}}{\epsilon} k^{\beta + \alpha_2 - \alpha_1} - 1.
\end{equation}
Because we assumed $\beta \ge \alpha_1 - \alpha_2$, the exponent of $k$ is non-negative, and thus there exists a sufficiently small $\epsilon > 0$ such that $\tilde{C}_k(\epsilon) > 0$ uniformly for all $k \ge 1$. This ensures that $0 \le f_{k, \epsilon}(\tau) \le 1$ across all features.

Consequently, the terms in \eqref{eq:excess_loss_series} are bounded by $V_k f_{k, \epsilon}(\tau)^2 \le V_k$. Since $\sum_{k=1}^\infty V_k < \infty$, we can apply the Lebesgue dominated convergence theorem to interchange the limit and the summation:
\begin{equation}
\mathcal{E}(\tau) := \lim_{\epsilon \to 0^+} \mathcal{E}_\epsilon(\tau\ln(1/\epsilon)) = \sum_{k=1}^\infty V_k \left( \lim_{\epsilon \to 0^+} f_{k, \epsilon}(\tau) \right)^2 = \sum_{k=1}^\infty V_k \cdot \mathbb{I}(\tau < \tau_k^*).
\end{equation}
The condition $\tau < \tau_k^*$ translates to $\tau < \frac{1}{a c_\zeta k^{-\alpha_2}}$, which can be inverted to define a threshold:
\begin{equation}
k > (a c_\zeta \tau)^{\frac{1}{\alpha_2}} := k^*(\tau).
\end{equation}
Let $K(\tau) = \lfloor k^*(\tau) \rfloor + 1$ be the smallest integer greater than $k^*(\tau)$. The excess loss equals the tail sum:
\begin{equation}
\mathcal{E}(\tau) = \sum_{k=K(\tau)}^\infty c_w k^{-\alpha_1}.
\end{equation}
Because $c_w x^{-\alpha_1}$ is monotonically decreasing, the tail sum is bounded by the Riemann integrals:
\begin{equation}
\int_{K(\tau)}^\infty c_w x^{-\alpha_1} dx \le \sum_{k=K(\tau)}^\infty c_w k^{-\alpha_1} \le \int_{K(\tau)-1}^\infty c_w x^{-\alpha_1} dx.
\end{equation}
Evaluating the integrals yields:
\begin{equation}
\frac{c_w}{\alpha_1 - 1} \left( K(\tau) \right)^{-(\alpha_1 - 1)} \le \mathcal{E}(\tau) \le \frac{c_w}{\alpha_1 - 1} \left( K(\tau) - 1 \right)^{-(\alpha_1 - 1)}.
\end{equation}
As $\tau \to \infty$, the threshold scales as $K(\tau) \sim k^*(\tau) = (a c_\zeta \tau)^{1/\alpha_2}$. Substituting $k^*(\tau)$ into the bounds gives the asymptotics:
\begin{equation}
\mathcal{E}(\tau) = \Theta\left( \left( \tau^{\frac{1}{\alpha_2}} \right)^{-(\alpha_1 - 1)} \right) = \Theta\left( \tau^{-\frac{\alpha_1 - 1}{\alpha_2}} \right).
\end{equation}
This completes the proof.
\end{proof}

\paragraph{Sample-wise Descent}
The first part of the formal statement is the following theorem.
\begin{theorem}
\label{theo:saddle_to_saddle_samples}
Under the conditions of Theorem \ref{theo:orthogonal_samples}, for each sample $\mu \in \{1, \dots, m\}$, let $\lambda_{\mu, \max} = \max_k \lambda_{\mu, k}$ and assume that $y_\mu > 0$. Define the effective growth rate for each sample as:
\begin{equation}
\zeta_\mu := \frac{8}{m} \lambda_{\mu, \max} y_\mu
\end{equation}
Assume that $\zeta_1 > \zeta_2 > \dots > \zeta_m > 0$. Let $L(t) = \frac{1}{m} \sum_{\mu=1}^m \Delta_\mu(t)^2$ be the empirical MSE and the initialization satisfies $z_k(0) = \Theta(\epsilon)$ for all $k \in \{1, \dots, p\}$.

Define the rescaled time $\tau := \frac{t}{\ln(1/\epsilon)}$ and the characteristic timescale $\tau_\mu^* := \frac{1}{\zeta_\mu}$. Then, in the small initialization limit $\epsilon \to 0^+$, the empirical loss converges pointwise (for almost every $\tau$) to a saddle-to-saddle trajectory transitioning through $m$ plateaus:
\begin{equation}
\lim_{\epsilon \to 0^+} L\left(\tau \ln\frac{1}{\epsilon}\right) = \frac{1}{m} \sum_{\mu=1}^m y_\mu^2 \cdot \mathbb{I}(\tau < \tau_\mu^*),
\end{equation}
where $\mathbb{I}(\cdot)$ denotes the indicator function.
\end{theorem}
\begin{proof}
According to Theorem \ref{theo:orthogonal_samples}, for a specific sample $\mu$, we denote $z_k(t) := z_k(0) \exp\left( -\tilde{\lambda}_k \xi_\mu(t) \right)$ and $\tilde{\lambda}_k := \frac{8}{m} \lambda_{\mu, k}$.
Let $\mathcal{K}_\mu = \{ k \mid \lambda_{\mu, k} = \lambda_{\mu, \max} \}$ be the set of indices for the maximum eigenvalue. We define
\begin{equation}
Z(t) := \sum_{k \in \mathcal{K}_\mu} \lambda_{\mu, \max} z_k(t) = Z(0) \exp\left( -\tilde{\lambda}_{\max} \xi_\mu(t) \right)
\end{equation}
where $Z(0) := \sum_{k \in \mathcal{K}_\mu} \lambda_{\mu, \max} z_k(0) = \Theta(\epsilon)$.

We first prove that $\xi_\mu(t)\leq0$. By definition, $\xi_\mu(0) = 0$. The initial residual is $\Delta_\mu(0) = \sum_{k=1}^p \lambda_{\mu, k} z_k(0) - y_\mu = O(\epsilon) - y_\mu$. Since $y_\mu > 0$, for sufficiently small $\epsilon$, $\Delta_\mu(0) < 0$. Because $\dot{\xi}_\mu(t) = \Delta_\mu(t)$, $\xi_\mu(t)$ is initially decreasing. Suppose there exists a first time $t_1 > 0$ such that $\xi_\mu(t_1) = 0$. At this instant, $z_k(t_1) = z_k(0)$ for all $k$, which implies $\dot{\xi}_\mu(t_1) = \Delta_\mu(0) < 0$. Since $\xi_\mu(t) < 0$ for $t \in (0, t_1)$, reaching zero at $t_1$ from below requires $\dot{\xi}_\mu(t_1) \ge 0$, which contradicts $\dot{\xi}_\mu(t_1) < 0$.

Because $\xi_\mu(t) \le 0$, the evolution of any feature depends on the sign of its eigenvalue. 
Since $\lambda_{\mu, \max} > 0$, the exponent $-\tilde{\lambda}_{\max} \xi_\mu(t) \ge 0$, which guarantees $Z(t) \ge Z(0) = C\epsilon$ for all $t \ge 0$. For any sub-dominant positive feature ($k \notin \mathcal{K}_\mu, \lambda_{\mu, k} > 0$), we can express its dynamics in terms of $Z(t)$ by eliminating $\xi_\mu(t)$: 
\begin{equation}
z_k(t) = z_k(0) \left( \frac{Z(t)}{Z(0)} \right)^{\gamma_k},
\end{equation}
where $\gamma_k := \lambda_{\mu, k} / \lambda_{\mu, \max} \in (0, 1)$. Conversely, for any negative feature ($\lambda_{\mu, k} < 0$), the exponent $-\tilde{\lambda}_k \xi_\mu(t) \le 0$. This implies that $z_k(t) \le z_k(0)$ for all $t \ge 0$.

We define the residual $\Delta_\mu(t) := Z(t) + R(Z, \epsilon) - y_\mu$, where $R(Z, \epsilon) := R_+(Z, \epsilon) + R_-(\epsilon)$ and
\begin{equation}
R_+(Z, \epsilon) = \sum_{\lambda_{\mu, k} > 0, k \notin \mathcal{K}_\mu} \lambda_{\mu, k} c_k \epsilon \left( \frac{Z}{A \epsilon} \right)^{\gamma_k}, \quad R_-(\epsilon) = \sum_{\lambda_{\mu, k} < 0} \lambda_{\mu, k} z_k(t).
\end{equation}
Before proceeding, we establish a global bound for $Z(t)$. Since the dynamics follow a gradient flow, the total loss is non-increasing ($L(t) \le L(0)$), which guarantees that the residual $\Delta_\mu(t)$ is globally bounded, i.e., $|\Delta_\mu(t)| \le C_\Delta$ for some constant $C_\Delta > 0$.

Rearranging the residual equation yields: $Z(t) + R_+(Z, \epsilon) = \Delta_\mu(t) + y_\mu - R_-(\epsilon)$. We have already established that $z_k(t) \le z_k(0)$ for all features with $\lambda_{\mu, k} < 0$. This implies that the negative contribution $R_-(\epsilon)$ is bounded: $-R_-(\epsilon) \le O(\epsilon)$. Consequently, the right-hand side of the equation is bounded by a constant $M_Z := C_\Delta + y_\mu + O(\epsilon)$. Since both $Z(t)$ and $R_+(Z, \epsilon)$ are sums of non-negative terms (as $\lambda_{\mu,k} > 0$ and $z_k(t) \ge 0$), we must have $Z(t) \le M_Z$ for all $t \ge 0$.

Let $\delta := \min_{\lambda_{\mu,k}>0} (1 - \gamma_k) > 0$. Since $M_Z\ge Z \ge C\epsilon$, $|R_+(Z, \epsilon)| \le C_1 \epsilon^\delta$. Since $z_k(t) \le c_k \epsilon$ for negative eigenvalues, $|R_-(\epsilon)| \le C_2 \epsilon$. Combining these, for any bounded trajectory of $Z(t)$, we can define a uniform bound relative to $y_\mu$: $\kappa(\epsilon) := \sup_t |R(Z(t), \epsilon)| / y_\mu = O(\epsilon^{\min(\delta, 1)})$, which vanishes as $\epsilon \to 0^+$.

Now, summing \eqref{eq:GD-z} over $k\in\mathcal{K}_\mu$ yields:
\begin{equation}
\dot{Z}(t) = -\frac{8}{m} \lambda_{\mu, \max} \Delta_\mu(t) Z(t) = \zeta_\mu Z(t) \left( 1 - \frac{Z(t) + R(Z, \epsilon)}{y_\mu} \right)
\end{equation}
where $\zeta_\mu = \frac{8}{m} \lambda_{\mu, \max} y_\mu$. With the definition of $\kappa(\epsilon)$, for $Z \in [0, y_\mu]$ we have:
\begin{equation}
\zeta_\mu Z \left( 1 - \kappa(\epsilon) - \frac{Z}{y_\mu} \right) \le \dot{Z} \le \zeta_\mu Z \left( 1 + \kappa(\epsilon) - \frac{Z}{y_\mu} \right)
\end{equation}
Let $Z_{lower}(t)$ and $Z_{upper}(t)$ denote the solutions to these bounding logistic equations with the initial condition $Z_{lower}(0) = Z_{upper}(0) = Z(0) = A \epsilon$:
\begin{align}
Z_{upper}(t) &= \frac{y_\mu (1 + \kappa(\epsilon)) A \epsilon}{A \epsilon + (y_\mu (1 + \kappa(\epsilon)) - A \epsilon) \exp(-\zeta_\mu (1 + \kappa(\epsilon)) t)} \\
Z_{lower}(t) &= \frac{y_\mu (1 - \kappa(\epsilon)) A \epsilon}{A \epsilon + (y_\mu (1 - \kappa(\epsilon)) - A \epsilon) \exp(-\zeta_\mu (1 - \kappa(\epsilon)) t)}
\end{align}
We now apply the timescale change $t = \tau \ln(1/\epsilon)$ to evaluate the pointwise limit as $\epsilon \to 0^+$. For the upper bound, substituting the time yields the exponential term $\exp(-\zeta_\mu (1 + \kappa(\epsilon)) \tau \ln(1/\epsilon)) = \epsilon^{\zeta_\mu \tau (1 + \kappa(\epsilon))}$. Dividing the numerator and denominator by $\epsilon$ yields:
\begin{equation}
Z_{upper}\left(\tau \ln\frac{1}{\epsilon}\right) = \frac{y_\mu (1 + \kappa(\epsilon)) A}{A + (y_\mu (1 + \kappa(\epsilon)) - C\epsilon) \epsilon^{\zeta_\mu \tau (1 + \kappa(\epsilon)) - 1}}
\end{equation}
Let $\tau_\mu^* = 1/\zeta_\mu$. Because $\kappa(\epsilon) \to 0$ as $\epsilon \to 0^+$, for any fixed $\tau \neq \tau_\mu^*$, the sign of the exponent $\zeta_\mu \tau (1 + \kappa(\epsilon)) - 1$ is determined by $\zeta_\mu \tau - 1$ for sufficiently small $\epsilon$. If $\tau < \tau_\mu^*$, then the exponent is negative, meaning $\epsilon^{\zeta_\mu \tau (1 + \kappa(\epsilon)) - 1} \to \infty$ and $Z_{upper} \to 0$. If $\tau > \tau_\mu^*$, then the exponent is positive, meaning the term vanishes and $Z_{upper} \to y_\mu (1 + 0) = y_\mu$.

Applying the same time substitution to the lower bound gives the term $\epsilon^{\zeta_\mu(1-\kappa(\epsilon))\tau}$. Dividing by $\epsilon$ gives the denominator behavior governed by $\epsilon^{\zeta_\mu \tau - 1 - \zeta_\mu \tau \kappa(\epsilon)}$. Since $\kappa(\epsilon) \to 0$, for any fixed $\tau \neq \tau_\mu^*$, the sign of the exponent is determined by $\zeta_\mu \tau - 1$ for sufficiently small $\epsilon$. Thus, as with the upper bound, $Z_{lower} \to 0$ for $\tau < \tau_\mu^*$ and $Z_{lower} \to y_\mu (1 - \kappa(0)) = y_\mu$ for $\tau > \tau_\mu^*$.

Since $Z_{lower}(t) \le Z(t) \le Z_{upper}(t)$, the dominant feature subspace converges to a step function for $\tau\neq\tau_\mu^*$:
\begin{equation}
\lim_{\epsilon \to 0^+} Z\left(\tau \ln\frac{1}{\epsilon}\right) = y_\mu \cdot \mathbb{I}(\tau > \tau_\mu^*)
\end{equation}
Returning to the residual $\Delta_\mu(t) = Z(t) + R(Z, \epsilon) - y_\mu$. Because $Z \le y_\mu$ and $R(Z, \epsilon) = O(\epsilon^{\min(\delta, 1)}) \to 0$ uniformly for bounded $Z$, we have $\lim_{\epsilon \to 0^+} \Delta_\mu(t) = \lim_{\epsilon \to 0^+} Z(t) - y_\mu$. Squaring this residual yields:
\begin{equation}
\lim_{\epsilon \to 0^+} \Delta_\mu\left(\tau \ln\frac{1}{\epsilon}\right)^2 = (0 - y_\mu)^2 \cdot \mathbb{I}(\tau < \tau_\mu^*) + (y_\mu - y_\mu)^2 \cdot \mathbb{I}(\tau > \tau_\mu^*) = y_\mu^2 \cdot \mathbb{I}(\tau < \tau_\mu^*)
\end{equation}
Summing over $\mu$ completes the proof.
\end{proof}

The second part of the formal statement is the following.
\begin{theorem}
Suppose the conditions of Theorem \ref{theo:saddle_to_saddle_samples} hold, and consider the infinite-sample limit $m \to \infty$. 
Suppose that the $\zeta_\mu$ and $y_\mu$ exhibit the following power-law decay:
\begin{align}
\zeta_\mu = c_\zeta \mu^{-\gamma_2},\
y_\mu = c_y \mu^{-\gamma_1}
\end{align}
for some positive constants $c_\zeta, c_y, \gamma_2 > 0$. Let the empirical MSE loss be defined as $\tilde L(t) = \sum_{\mu=1}^m \Delta_\mu(t)^2$. We also assume $\sum_{\mu=1}^\infty \left( \sum_{k=1}^\infty \lambda_{\mu, k} z_k(0) \right)^2 < \infty$ 
and $\gamma_1 > 1/2$ such that the initial loss converges.

Define the rescaled time $\tau := \frac{t}{\ln(1/\epsilon)}$. The empirical loss converges $\mathcal{L}(\tau) := \lim_{\epsilon \to 0^+} \tilde L\left(\tau \ln \frac{1}{\epsilon} \right)$ for almost every $\tau$ and the limiting loss follows a power-law decay:
\begin{equation}
\mathcal{L}(\tau) = \Theta\left( \tau^{-\frac{2\gamma_1 - 1}{\gamma_2}} \right).
\end{equation}
\end{theorem}
\begin{proof}
According to Theorem \ref{theo:saddle_to_saddle_samples}, the residual for each sample converges for almost every $\tau$:
\begin{equation}
\lim_{\epsilon \to 0^+} \Delta_\mu\left(\tau \ln\frac{1}{\epsilon}\right)^2 = y_\mu^2 \cdot \mathbb{I}(\tau < \tau_\mu^*)
\end{equation}
where $\mathbb{I}(\cdot)$ is the indicator function and $\tau_\mu^* = \frac{1}{\zeta_\mu}$ is the characteristic timescale.

To interchange the limit and the infinite sum, we establish a uniform bound for the residuals. As shown in the proof of Theorem \ref{theo:saddle_to_saddle_samples}, for any trajectory of $Z(t)$, we have $Z(t) \le y_\mu(1 + \kappa(\epsilon))$ and $|R(Z, \epsilon)| \le y_\mu \kappa(\epsilon)$, where $\kappa(\epsilon) \to 0$ as $\epsilon \to 0^+$. Therefore, the absolute residual is bounded by:
\begin{equation}
|\Delta_\mu(t)| = |Z(t) + R(Z, \epsilon) - y_\mu| \le |Z(t)| + |R(Z, \epsilon)| + y_\mu \le y_\mu(2 + 2\kappa(\epsilon)).
\end{equation}
For sufficiently small $\epsilon$, we have $\kappa(\epsilon) \le 1/2$, yielding a uniform bound $\Delta_\mu(t)^2 \le 9 y_\mu^2$ for all $t \ge 0$. Under the assumption $\gamma_1 > 1/2$, the series $\sum_{\mu=1}^\infty 9 y_\mu^2 = 9 c_y^2 \sum_{\mu=1}^\infty \mu^{-2\gamma_1} < \infty$.

Now we can apply the Lebesgue dominated convergence theorem:
\begin{equation}
\mathcal{L}(\tau) = \lim_{\epsilon \to 0^+} \sum_{\mu=1}^\infty \Delta_\mu\left(\tau \ln\frac{1}{\epsilon}\right)^2 = \sum_{\mu=1}^\infty y_\mu^2 \cdot \mathbb{I}(\tau < \tau_\mu^*),
\end{equation}
We can then define a threshold on the index:
\begin{equation}
\mu^*(\tau):=\left( c_\zeta \tau \right)^{\frac{1}{\gamma_2}}.
\end{equation}
Let $M(\tau) = \lfloor \mu^*(\tau) \rfloor + 1$ be the smallest integer greater than $\mu^*(\tau)$. The limiting loss evaluates to the tail sum:
\begin{equation}
\mathcal{L}(\tau) = \sum_{\mu=M(\tau)}^\infty c_y^2 \mu^{-2\gamma_1}.
\end{equation}
Since $c_y^2 x^{-2\gamma_1}$ is decreasing, the sum is bounded by its corresponding Riemann integrals:
\begin{equation}
\int_{M(\tau)}^\infty c_y^2 x^{-2\gamma_1} dx \le \mathcal{L}(\tau) \le \int_{M(\tau)-1}^\infty c_y^2 x^{-2\gamma_1} dx
\end{equation}
Evaluating the integrals yields explicit upper and lower bounds:
\begin{equation}
\frac{c_y^2}{2\gamma_1 - 1} \left( M(\tau) \right)^{-(2\gamma_1 - 1)} \le \mathcal{L}(\tau) \le \frac{c_y^2}{2\gamma_1 - 1} \left( M(\tau) - 1 \right)^{-(2\gamma_1 - 1)}
\end{equation}
As $\tau \to \infty$, the threshold scales as $M(\tau) \sim \mu^*(\tau) = (c_\zeta \tau)^{\frac{1}{\gamma_2}}$. Substituting this asymptotic scaling into both bounds provides the decay rate:
\begin{equation}
\mathcal{L}(\tau) = \Theta\left( \left( \tau^{\frac{1}{\gamma_2}} \right)^{-(2\gamma_1 - 1)} \right) = \Theta\left( \tau^{-\frac{2\gamma_1 - 1}{\gamma_2}} \right).
\end{equation}
This completes the proof.
\end{proof}

\section{Additional Theory}
\subsection{Removing ZGZ Condition}
\label{app:remove_ZGZ}
\begin{theorem}
\label{theo:expansion}
Let $f_x: \mathbb{R}^{p \times d} \to \mathbb{R}$ be a three times continuously differentiable model with respect to its neurons $W = (w_1, \dots, w_d)$. Assume $f_x$ satisfies the permutation symmetry.
Then, the model can be approximated around the origin as:
\begin{equation}
f_x(W) = f_x(0) + \sum_{i=1}^d g(x)^\top  w_i + \sum_{i=1}^d w_i^\top  A(x) w_i + \sum_{i \neq j} w_i^\top  B(x) w_j + O(\|W\|^3),
\end{equation}
where $g(x) \in \mathbb{R}^p$, $A(x) \in \mathbb{R}^{p \times p}$ is a symmetric matrix representing the self-interaction of each neuron and $B(x) \in \mathbb{R}^{p \times p}$ is a symmetric matrix representing the cross-interaction between any pair of distinct neurons.
\end{theorem}
\begin{proof}
We begin with the standard multi-variable Taylor expansion of $f_x$ around $W=0$:
\begin{equation}
f_x(W) = f_x(0) + \sum_{i=1}^d \big(\nabla_{w_i} f_x(0)\big)^\top  w_i + \frac{1}{2} \sum_{i=1}^d \sum_{j=1}^d w_i^\top  H_{ij} w_j + O(\|W\|^3),
\end{equation}
where $H_{ij} = \frac{\partial^2 f_x}{\partial w_i \partial w_j} \big|_{W=0}$ are the $p \times p$ blocks of the Hessian matrix.

By permutation symmetry, all neurons are interchangeable. This implies that the first derivative evaluated at the symmetric origin $W=0$ must be identical for all neurons. Thus, there exists a vector $g(x)$ such that:
\begin{equation}
\nabla_{w_i} f_x(0) = g(x) \quad \text{for all } i.
\end{equation}
Similarly, the permutation symmetry severely restricts the blocks of the Hessian matrix evaluated at $W=0$. There can only be two types of blocks.

Diagonal blocks: $H_{ii}$ must be identical for all $i$. Let $H_{ii} = 2 A(x)$.

Off-diagonal blocks: $H_{ij}$ for all $i \neq j$ must be identical. Let $H_{ij} = 2 B(x)$.

Substituting these findings back into the Taylor expansion:
\begin{equation}
\begin{aligned}
f_x(W) &= f_x(0) + \sum_{i=1}^d g(x)^\top  w_i + \frac{1}{2} \sum_{i=1}^d w_i^\top  (2A(x)) w_i + \frac{1}{2} \sum_{i \neq j} w_i^\top  (2B(x)) w_j + O(\|W\|^3) \\
&= f_x(0) + \sum_{i=1}^d g(x)^\top  w_i + \sum_{i=1}^d w_i^\top  A(x) w_i + \sum_{i \neq j} w_i^\top  B(x) w_j + O(\|W\|^3).
\end{aligned}
\end{equation}
This concludes the proof.
\end{proof}

We can similarly expand the model to an arbitrary order.
\begin{theorem}
\label{theo:generalized_expansion}
Let $f_x: \mathbb{R}^{p \times d} \to \mathbb{R}$ be a $K+1$-times continuously differentiable model with respect to its neurons $W = (w_1, \dots, w_d)$. Assume $f_x$ satisfies the permutation symmetry.

Then, the $K$-th order Taylor expansion of the model around the origin $W=0$ can be expressed as:
\begin{equation}
f_x(W) = f_x(0) + \sum_{k=1}^K \sum_{\lambda \vdash k} \sum_{\substack{i_1, \dots, i_m \\ \text{distinct}}} \mathcal{T}_{\lambda}(x) \big[ w_{i_1}^{\otimes c_1}, w_{i_2}^{\otimes c_2}, \dots, w_{i_m}^{\otimes c_m} \big] + O(\|W\|^{K+1}),
\label{eq:Taylor-higher-order}
\end{equation}
where $\lambda = \{c_1, c_2, \dots, c_m\}$ represents an integer partition of $k$, satisfying $\sum_{r=1}^m c_r = k$ and $c_r \ge 1$. The innermost sum iterates over all $m$-tuples of mutually distinct indices $(i_1, \dots, i_m)$ drawn from $\{1, \dots, d\}$ and $\mathcal{T}_{\lambda}(x)$ is a symmetric tensor of order $k$, determined by $\lambda$, representing the partial derivatives and combinatorial constants evaluated at the origin.
\end{theorem}
\begin{proof}
We begin with the standard multi-variable Taylor expansion of $f_x$ around $W=0$ up to order $K$:
\begin{equation}
f_x(W) = f_x(0) + \sum_{k=1}^K \frac{1}{k!} \sum_{j_1=1}^d \dots \sum_{j_k=1}^d \left( \nabla^k_{w_{j_1} \dots w_{j_k}} f_x(0) \right) [w_{j_1}, \dots, w_{j_k}] + O(\|W\|^{K+1}).
\end{equation}
where $\nabla^k_{w_{j_1} \dots w_{j_k}} f_x(0)$ denotes the $k$-th order partial derivative tensor block evaluated at the origin.

For any fixed $k$, consider an arbitrary sequence of indices $(j_1, \dots, j_k)$. This sequence can be grouped by the multiplicity of each distinct index it contains. Let there be $m$ distinct indices $i_1, \dots, i_m$ appearing with multiplicities $c_1, \dots, c_m$ respectively, such that $\sum_{r=1}^m c_r = k$ and $c_r \ge 1$. This multiset of multiplicities defines an integer partition $\lambda \vdash k$.

By permutation symmetry, all neurons are interchangeable. At $W=0$, any permutation of the input neurons leaves the function and its derivative tensors unchanged. Consequently, the derivative $\nabla^k_{w_{j_1} \dots w_{j_k}} f_x(0)$ depends on the partition $\lambda$ of its indices, and is identical for any specific choice of neurons $i_1, \dots, i_m$.

Let $D_\lambda(x)$ denote this derivative corresponding to the partition $\lambda$. We can rewrite the sum over the raw indices $(j_1, \dots, j_k)$ by summing over all valid partitions $\lambda \vdash k$, and then summing over all possible choices of $m$ distinct indices $(i_1, \dots, i_m)$ from $\{1, \dots, d\}$:
\begin{equation}
\sum_{j_1=1}^d \dots \sum_{j_k=1}^d \left( \nabla^k_{w_{j_1} \dots w_{j_k}} f_x(0) \right) [w_{j_1}, \dots, w_{j_k}] = \sum_{\lambda \vdash k} \sum_{\substack{i_1, \dots, i_m \\ \text{distinct}}} C_\lambda D_\lambda(x) \big[ w_{i_1}^{\otimes c_1}, \dots, w_{i_m}^{\otimes c_m} \big].
\end{equation}
Here, $C_\lambda$ is a combinatorial coefficient counting the number of ways to arrange the sequence $(j_1, \dots, j_k)$ such that the distinct elements appear with the frequencies specified by $\lambda$.

Substituting it back into \eqref{eq:Taylor-higher-order}, we can define $
\mathcal{T}_\lambda(x) = \frac{C_\lambda}{k!} D_\lambda(x)$, which yields:
\begin{equation}
f_x(W) = f_x(0) + \sum_{k=1}^K \sum_{\lambda \vdash k} \sum_{\substack{i_1, \dots, i_m \\ \text{distinct}}} \mathcal{T}_{\lambda}(x) \big[ w_{i_1}^{\otimes c_1}, \dots, w_{i_m}^{\otimes c_m} \big] + O(\|W\|^{K+1}).
\end{equation}
This concludes the proof.
\end{proof}

\subsection{Multidimensional Output}
\label{app:multi-dimension}
\begin{corollary}
Let $F_x: \mathbb{R}^{p \times d} \to \mathbb{R}^k$ be a three times continuously differentiable model with respect to its neurons $W = (w_1, \dots, w_d)$, where $w_i \in \mathbb{R}^p$. Assume $F_x$ satisfies the permutation symmetry\footnote{$F_x(w_1, \dots, w_d) = F_x(w_{\sigma(1)}, \dots, w_{\sigma(d)})$ for any permutation $\sigma$.} and the ZGZ condition: the Jacobian matrix satisfies $
\mathcal{J}_{w_i} F_x(w_1, \dots, w_d) \big|_{w_i = 0} = \mathbf{0} \in \mathbb{R}^{k \times p}$ for any $i$.

Then, the model can be approximated around the origin as:
\begin{equation}
F_x(W) = F_x(0) + \sum_{i=1}^d \mathcal{A}(x)[w_i, w_i] + O(\|W\|^3),
\end{equation}
where $\mathcal{A}(x) \in \mathbb{R}^{k \times p \times p}$ is a 3rd-order tensor dependent only on the input $x$. The operation $\mathcal{A}(x)[w_i, w_i]$ yields a vector in $\mathbb{R}^k$ whose $m$-th component is $w_i^\top  A^{(m)}(x) w_i$, with each $A^{(m)}(x) \in \mathbb{R}^{p \times p}$ being a symmetric matrix.
\end{corollary}
\begin{proof}
Let $F_x(W) = \big[ f_x^{(1)}(W), f_x^{(2)}(W), \dots, f_x^{(k)}(W) \big]^\top $, where each component function $f_x^{(m)}: \mathbb{R}^{p \times d} \to \mathbb{R}$ is three times continuously differentiable.

By permutation symmetry, each component $f_x^{(m)}$ satisfies
\begin{equation}
f_x^{(m)}(w_1, \dots, w_d) = f_x^{(m)}(w_{\sigma(1)}, \dots, w_{\sigma(d)}).
\end{equation}
Furthermore, the ZGZ condition implies that for every output dimension $m$:
\begin{equation}
\nabla_{w_i} f_x^{(m)}(W) \big|_{w_i = 0} = \mathbf{0} \in \mathbb{R}^p.
\end{equation}
Thus each component function $f_x^{(m)}$ satisfies the conditions of Theorem \ref{theo:main}, and we can apply it to each dimension $m$:
\begin{equation}
f_x^{(m)}(W) = f_x^{(m)}(0) + \sum_{i=1}^d w_i^\top  A^{(m)}(x) w_i + O(\|W\|^3),
\end{equation}
where $A^{(m)}(x) = \frac{1}{2} H_{ii}^{(m)}$ is the Hessian blocks of the $m$-th output dimension. To stack the equations for all $k$ components, we define a 3rd-order tensor $\mathcal{A}(x) \in \mathbb{R}^{k \times p \times p}$ such that its $m$-th slice along the first dimension represents $A^{(m)}(x)$. Consequently, the multi-dimensional expansion can be written as:
\begin{equation}
F_x(W) = F_x(0) + \sum_{i=1}^d \mathcal{A}(x)[w_i, w_i] + O(\|W\|^3).
\end{equation}
This concludes the proof.
\end{proof}

\subsection{Learning Dynamics of Other Optimization Methods}
\label{app:linear_update}
\begin{corollary}
Consider any learning algorithm whose update rule for the weight matrix $W$ takes the form:
\begin{equation}
\Delta W_t = \sum_{k} \alpha_k \nabla_W \mathcal{L}_k + \beta_k W_k
\end{equation}
where $\alpha_k, \beta_k$ can be any scalar functions of the historical sufficient statistics $\{M_\tau, \mu_\tau\}_{\tau \le t}$. Then, the learning dynamics of the NQF are completely determined by $\{M_\tau, \mu_\tau\}_{\tau \le t}$. %, but excludes algorithms with element-wise non-linearities such as Adam.
\label{cor:linear_update}
\end{corollary}
\begin{proof}
From Theorem \ref{theo:master}, we know that the gradient of the loss with respect to the weight matrix $W$ at any step $k$ takes the form:
\begin{equation}
\nabla_W \mathcal{L}_k = v_k \mathbf{1}^\top  + H_k W_k.
\end{equation}
Given the update rule $\Delta W_t = \sum_{k \le t} (\alpha_k \nabla_W \mathcal{L}_k + \beta_k W_k)$, we substitute the gradient expression:
\begin{equation}
\Delta W_t = \sum_{k \le t} \big( \alpha_k v_k \mathbf{1}^\top  + (\alpha_k H_k + \beta_k I) W_k \big).
\end{equation}
We claim that for any step $t$, the weight matrix $W_t$ can be written in the following form:
\begin{equation}
W_t = P_t W_0 + Q_t \mathbf{1}^\top ,
\end{equation}
where $P_t \in \mathbb{R}^{p \times p}$ is a transformation matrix and $Q_t \in \mathbb{R}^p$ is a vector, dependent on the historical sequences $\{v_k, H_k, \alpha_k, \beta_k\}_{k < t}$.

We prove this by induction: For $t=0$, $W_0 = I \cdot W_0 + 0 \cdot \mathbf{1}^\top $. Thus, $P_0 = I$ and $Q_0 = 0$, which holds.

Assume the claim holds for all $k \le t$. For step $t+1$, we have:
\begin{equation}
W_{t+1} = W_t + \Delta W_t = W_t + \sum_{k \le t} \big( \alpha_k v_k \mathbf{1}^\top  + (\alpha_k H_k + \beta_k I) W_k \big).
\end{equation}
Substitute $W_k = P_k W_0 + Q_k \mathbf{1}^\top $ into the equation:
\begin{equation}
W_{t+1} = (P_t W_0 + Q_t \mathbf{1}^\top ) + \sum_{k \le t} \Big[ \alpha_k v_k \mathbf{1}^\top  + (\alpha_k H_k + \beta_k I) (P_k W_0 + Q_k \mathbf{1}^\top ) \Big].
\end{equation}
By grouping the terms associated with $W_0$ and $\mathbf{1}^\top $, we get:
\begin{equation}
W_{t+1} = \underbrace{\Big( P_t + \sum_{k \le t} (\alpha_k H_k + \beta_k I) P_k \Big)}_{:= P_{t+1}} W_0 + \underbrace{\Big( Q_t + \sum_{k \le t} [\alpha_k v_k + (\alpha_k H_k + \beta_k I) Q_k] \Big)}_{:= Q_{t+1}} \mathbf{1}^\top .
\end{equation}
This defines a recurrence relation for $P_{t+1}$ and $Q_{t+1}$ that depends only on variables from steps $k \le t$.

Now we show that $\mu_t$ and $M_t$ only depend on $P_t, Q_t$ and the initial statistics.
For the first moment $\mu_t$:
\begin{equation}
\mu_t = W_t \mathbf{1} = (P_t W_0 + Q_t \mathbf{1}^\top ) \mathbf{1} = P_t \mu_0 + d \cdot Q_t
\end{equation}
For the second moment $M_t$:
\begin{equation}
M_t = W_t W_t^\top  = (P_t W_0 + Q_t \mathbf{1}^\top )(P_t W_0 + Q_t \mathbf{1}^\top )^\top = P_t M_0 P_t^\top  + P_t \mu_0 Q_t^\top  + Q_t \mu_0^t  P_t^\top  + d \cdot Q_t Q_t^\top .
\end{equation}
Therefore, given the initial conditions $M_0$ and $\mu_0$, the subsequent values of $P_t$ and $Q_t$ are completely determined, which in turn dictate the values of $\mu_t$ and $M_t$ via the derived equations. This proves that for any linear-span learning algorithm, the dynamics of the NQF form a closed system over $\{M_\tau, \mu_\tau\}_{\tau \le t}$. 
\end{proof}

\subsection{Learning Dynamics of Multi-layer NQF}
\label{app:multi-layer}
Before stating the theorem, we first define an $L$-layer NQF.
\begin{definition}[Deep NQF]
Let $f(x)=h^{(L)}(x)$ be an $L$-layer NQF with hidden representation $h^{(l)}(x) \in \mathbb{R}^{D_l}$:
\begin{align}
h^{(0)}(x) &= x \\
h^{(l)}_k(x) &= C_k^{(l)} + \big(g_k^{(l)}\big)^\top \mu^{(l)} + \big(\mu^{(l)}\big)^\top B_k^{(l)} \mu^{(l)} + \mathrm{Tr}\left[M^{(l)} A_k^{(l)}\right], \quad k=1, \dots, D_l
\end{align}
where $M^{(l)} := W^{(l)}(W^{(l)})^\top \in \mathbb{R}^{p_l \times p_l}$ and $\mu^{(l)} := W^{(l)} \mathbf{1}_{d_l} \in \mathbb{R}^{p_l}$ are the order parameters. The structure components $\{C_k^{(l)}, g_k^{(l)}, B_k^{(l)}, A_k^{(l)}\}$ are functions that only depend on the previous layer's output $h^{(l-1)}(x)$. Without loss of generality, $A_k^{(l)}$ and $B_k^{(l)}$ are symmetric matrices.
\end{definition}
\begin{theorem}[Master Theorem for Deep NQF]
\label{theo:deep_master}
Under SGD
\begin{equation}
\Delta W^{(l)} = -\eta \sum_{x \in \mathcal{B}} \nabla_{W^{(l)}} \mathcal{L}(f(x))
\end{equation}
with identical data sampling, two models will have identical states $(\mu^{(l)}_a(t) = \mu^{(l)}_b(t), M^{(l)}_a(t) = M^{(l)}_b(t))$ for all layers $l$ and time $t$ if they share the same initial statistics:
\begin{equation}
\mu^{(l)}_a(0) = \mu^{(l)}_b(0), \quad M^{(l)}_a(0) = M^{(l)}_b(0) \quad \forall l \in \{1, \dots, L\}.
\end{equation}
The learning dynamics are closed, evolving as:
\begin{align}
\Delta \mu^{(l)} &= -\eta \big( d_l \cdot v^{(l)} + H^{(l)} \mu^{(l)} \big) \\
\Delta M^{(l)} &= -\eta \big( v^{(l)} (\mu^{(l)})^\top + \mu^{(l)} (v^{(l)})^\top + H^{(l)} M^{(l)} + M^{(l)} H^{(l)} \big) \notag \\
&\quad + \eta^2 \big( d_l \cdot v^{(l)} (v^{(l)})^\top + v^{(l)} (\mu^{(l)})^\top H^{(l)} + H^{(l)} \mu^{(l)} (v^{(l)})^\top + H^{(l)} M^{(l)} H^{(l)} \big),
\end{align}
where $d_l$ is the width (number of neurons) of layer $l$. The vector $v^{(l)} \in \mathbb{R}^{p_l}$ and effective Hessian $H^{(l)} \in \mathbb{R}^{p_l \times p_l}$ are defined as:
\begin{align}
v^{(l)} &:= \sum_{x \in \mathcal{B}} \sum_{k=1}^{D_l} \frac{\partial \mathcal{L}}{\partial h_k^{(l)}(x)} \Big( g_k^{(l)} + 2B_k^{(l)} \mu^{(l)} \Big) \\
H^{(l)} &:= \sum_{x \in \mathcal{B}} \sum_{k=1}^{D_l} 2 \frac{\partial \mathcal{L}}{\partial h_k^{(l)}(x)} A_k^{(l)}
\end{align}
\end{theorem}
\begin{proof}
By the chain rule, the gradient of the loss with respect to the weights $W^{(l)}$ is:
\begin{equation}
\nabla_{W^{(l)}} \mathcal{L} = \sum_{x \in \mathcal{B}} \sum_{k=1}^{D_l} \frac{\partial \mathcal{L}}{\partial h_k^{(l)}(x)} \nabla_{W^{(l)}} h_k^{(l)}(x).
\end{equation}
Notice that $\mu^{(l)} = W^{(l)} \mathbf{1}_{d_l}$, where $\mathbf{1}_{d_l} \in \mathbb{R}^{d_l}$ is a vector of ones. The derivatives of the terms in $h_k^{(l)}$ are:
\begin{align}
\nabla_{W^{(l)}} \big( (g_k^{(l)})^\top \mu^{(l)} \big) &= g_k^{(l)} \mathbf{1}_{d_l}^\top \\
\nabla_{W^{(l)}} \big( (\mu^{(l)})^\top B_k^{(l)} \mu^{(l)} \big) &= 2 B_k^{(l)} \mu^{(l)} \mathbf{1}_{d_l}^\top \\
\nabla_{W^{(l)}} \mathrm{Tr}\left[ M^{(l)} A_k^{(l)} \right] &= 2 A_k^{(l)} W^{(l)}
\end{align}
Collecting these terms, the gradient factorizes into:
\begin{equation}
\nabla_{W^{(l)}} \mathcal{L} = v^{(l)} \mathbf{1}_{d_l}^\top + H^{(l)} W^{(l)},
\end{equation}
where $v^{(l)}$ and $H^{(l)}$ are defined as in the theorem. The SGD update is therefore:
\begin{equation}
\Delta W^{(l)} = -\eta \big( v^{(l)} \mathbf{1}_{d_l}^\top + H^{(l)} W^{(l)} \big).
\end{equation}
For the mean vector $\mu^{(l)}$, we multiply the update by $\mathbf{1}_{d_l}$:
\begin{equation}
\Delta \mu^{(l)} = \Delta W^{(l)} \mathbf{1}_{d_l} = -\eta \big( v^{(l)} (\mathbf{1}_{d_l}^\top \mathbf{1}_{d_l}) + H^{(l)} W^{(l)} \mathbf{1}_{d_l} \big) = -\eta \big( d_l \cdot v^{(l)} + H^{(l)} \mu^{(l)} \big).
\end{equation}
For the second moment $M^{(l)}$, we expand $M^{(l)}_{t+1} = (W^{(l)} + \Delta W^{(l)})(W^{(l)} + \Delta W^{(l)})^\top$:
\begin{equation}
\Delta M^{(l)} = W^{(l)} (\Delta W^{(l)})^\top + \Delta W^{(l)} (W^{(l)})^\top + \Delta W^{(l)} (\Delta W^{(l)})^\top.
\end{equation}
Substitute $\Delta W^{(l)}$ and use the identity $W^{(l)} \mathbf{1}_{d_l} = \mu^{(l)}$:
\begin{equation}
W^{(l)} (\Delta W^{(l)})^\top = -\eta \big( W^{(l)} \mathbf{1}_{d_l} (v^{(l)})^\top + W^{(l)} (W^{(l)})^\top H^{(l)} \big) = -\eta \big( \mu^{(l)} (v^{(l)})^\top + M^{(l)} H^{(l)} \big).
\end{equation}
The symmetric counterpart is $-\eta \big( v^{(l)} (\mu^{(l)})^\top + H^{(l)} M^{(l)} \big)$.
For the $\eta^2$ term:
\begin{align}
\Delta W^{(l)} (\Delta W^{(l)})^\top &= \eta^2 \big( v^{(l)} \mathbf{1}_{d_l}^\top + H^{(l)} W^{(l)} \big) \big( \mathbf{1}_{d_l} (v^{(l)})^\top + (W^{(l)})^\top H^{(l)} \big) \notag \\
&= \eta^2 \big( v^{(l)} (\mathbf{1}_{d_l}^\top \mathbf{1}_{d_l}) (v^{(l)})^\top + v^{(l)} \mathbf{1}_{d_l}^\top (W^{(l)})^\top H^{(l)} + H^{(l)} W^{(l)} \mathbf{1}_{d_l} (v^{(l)})^\top + H^{(l)} W^{(l)} (W^{(l)})^\top H^{(l)} \big) \notag \\
&= \eta^2 \big( d_l \cdot v^{(l)} (v^{(l)})^\top + v^{(l)} (\mu^{(l)})^\top H^{(l)} + H^{(l)} \mu^{(l)} (v^{(l)})^\top + H^{(l)} M^{(l)} H^{(l)} \big).
\end{align}
Summing these gives the update rule for $\Delta M^{(l)}$.

Since the forward pass $h^{(l)}(x)$ depends only on $\{\mu^{(l)}, M^{(l)}\}$ and $h^{(l-1)}(x)$, all intermediate activations are functions of the order parameters. Consequently, the gradients $\frac{\partial \mathcal{L}}{\partial h^{(l)}(x)}$ also depend only on the sequence of order parameters. Thus, the dynamics are closed independent of $W^{(l)}$.
\end{proof}
Theorem \ref{theo:compressibility} can be generalized to multi-layer NQFs in the same way.
\begin{theorem}[Compressibility of Deep NQF]\label{theo:deep_compressibility}
Let an $L$-layer General Deep NQF have layer widths $d_1, \dots, d_L$, defined by functions $g_k^{(l)}, B_k^{(l)}, A_k^{(l)}$, learning rate $\eta$, and initial parameter statistics $\mu^{(l)}(0)$ and $M^{(l)}(0)$ for each layer $l \in \{1, \dots, L\}$.

Let $\mathcal{V}^{(l)} \subseteq \mathbb{R}^{p_l}$ be the joint subspace spanned by the network components across all training samples $x \in \mathcal{X}$ and all training steps $t \ge 0$ of the original model:
\begin{equation}
\mathcal{V}^{(l)} := \mathrm{span} \bigcup_{x \in \mathcal{X}, t \ge 0} \bigcup_{k=1}^{D_l} \Big( {g_k^{(l)}(h_t^{(l-1)})} \cup \mathrm{Col}(A_k^{(l)}(h_t^{(l-1)})) \cup \mathrm{Col}(B_k^{(l)}(h_t^{(l-1)})) \Big).
\end{equation}
Let $k_l := \dim(\mathcal{V}^{(l)}) \le p_l$. If $d_l > k_l + 1$ for all $l$, then there exists a smaller Deep NQF with layer widths $d'_l = k_l + 1$, characterized by $\tilde{A}_k^{(l)} = \frac{d'_l}{d_l} A_k^{(l)}$, rescaled learning rate $\tilde{\eta} = \frac{d_l}{d'_l} \eta$ (layer-specific learning rates), such that:
\begin{enumerate}[noitemsep,topsep=0pt, parsep=0pt,partopsep=0pt, leftmargin=13pt]
\item $\tilde{f}_x = f_x$ for all inputs $x\in\mathcal{X}$ at initialization;
\item $\tilde{\mu}^{(l)}_t = V_l V_l^\top \mu^{(l)}_t$ and $\tilde{M}^{(l)}_t = \frac{d_l}{d'_l} V_l V_l^\top M^{(l)}_t V_l V_l^\top$ for all layers $l$ and all training steps $t \ge 0$, where $V_l \in \mathbb{R}^{p_l \times k_l}$ is the orthogonal basis of $\mathcal{V}^{(l)}$.
\item The output and learning dynamics of the smaller model are identical to the original model: $\tilde{f}_x(t) = f_x(t)$ for all $t \ge 0$.
\end{enumerate}
\end{theorem}

\subsection{Compression Error}
\label{app:compression_error}
\begin{corollary}
\label{cor:error_bound}
For any original NQF with $d$ neurons and input dimension $p$, and any target width $d' \leq p$, there exists a compressed NQF with $d'$ neurons such that, assuming the SGD update dynamics are locally Lipschitz continuous with constant $L$, the deviation in its sufficient statistics at any training step $t$ is bounded by:
\begin{equation}
\left\| \frac{d'}{d}\tilde{M}_t - M_t \right\|_F + \left\| \tilde{\mu}_t - \mu_t \right\|_2 \le \left( \sqrt{\sum_{k=d'}^p \lambda_k^2} \right) \exp(\eta L t),
\end{equation}
where $\lambda_1 \ge \lambda_2 \ge \dots \ge \lambda_p \ge 0$ are the eigenvalues of the initial covariance matrix $C := \frac{d}{d'} \big( M(0) - \frac{1}{d}\mu(0)\mu(0)^\top  \big)$.
\end{corollary}
\begin{proof}
Let the discrepancy in the sufficient statistics at training step $t$ be denoted by the error states $E_M^{(t)} := \frac{d'}{d}\tilde{M}_t - M_t$ and $E_\mu^{(t)} := \tilde{\mu}_t - \mu_t$. Define the total error norm as $e_t := \| E_M^{(t)} \|_F + \| E_\mu^{(t)} \|_2$.

To minimize the initial error, we first match the first moment by setting $\tilde{\mu}(0) = \mu(0)$, which yields $E_\mu^{(0)} = 0$. For the second moment, following the construction in Theorem \ref{theo:compressibility}, finding $U = [u_1, \dots, u_{d'}] \in \mathbb{R}^{p \times d'}$ requires solving:
\begin{equation}
\sum_{i=1}^{d'} u_i = 0, \quad \text{and} \quad U U^\top  \approx C := \frac{d}{d'} \left( M(0) - \frac{1}{d}\mu(0)\mu(0)^\top  \right).
\end{equation}
Because the $d'$ vectors must sum to zero, they are confined to a $(d'-1)$-dimensional subspace. Consequently, the rank of the constructed covariance matrix $UU^\top $ can be at most $d'-1$. Since we assume the target width $d' \leq p$, perfect reconstruction of the $p \times p$ full-rank matrix $C$ is in general impossible.

By the Eckart-Young-Mirsky theorem, the optimal rank-$(d'-1)$ approximation $\tilde{C}$ that minimizes the Frobenius norm of the residual is given by the truncated eigendecomposition of $C$. Let $C = \sum_{k=1}^p \lambda_k q_k q_k^\top $ be the eigendecomposition with eigenvalues $\lambda_1 \ge \dots \ge \lambda_p \ge 0$. We construct $U$ such that $UU^\top  = \tilde{C} := \sum_{k=1}^{d'-1} \lambda_k q_k q_k^\top $. The residual matrix is $\Delta C = C - \tilde{C} = \sum_{k=d'}^p \lambda_k q_k q_k^\top $.

We initialize the compressed model such that its effective initial covariance is this optimal approximation. Thus, the initial matrix error is the residual:
\begin{equation}
E_M^{(0)} = \frac{d'}{d}\tilde{M}(0) - M(0) = \tilde{C} - C = -\Delta C.
\end{equation}
The initial total error is therefore the Frobenius norm of the truncated tail:
\begin{equation}
e_0 = \|\Delta C\|_F + 0 = \sqrt{\sum_{k=d'}^p \lambda_k^2}.
\label{eq:initial_error}
\end{equation}
For any step $t$, the dynamics of $M$ and $\mu$ are governed by the update rules derived in Theorem \ref{theo:master}. The increments $\Delta M$ and $\Delta \mu$ are functions of the current state variables $(M_t, \mu_t)$, the data sequence $A(x), B(x), g(x)$, and the loss derivative $\ell'_x$.

Because polynomials are locally Lipschitz continuous, and assuming the SGD trajectories and data samples remain bounded within the time horizon $t$, there exists a uniform Lipschitz constant $L > 0$ that bounds the divergence of the trajectories. Specifically, the difference in the updates between the uncompressed state $(M_t, \mu_t)$ and the scaled compressed state $(\frac{d'}{d}\tilde{M}_t, \tilde{\mu}_t)$ satisfies:
\begin{equation}
\left\| \frac{d'}{d}\Delta \tilde{M}_t - \Delta M_t \right\|_F + \left\| \Delta \tilde{\mu}_t - \Delta \mu_t \right\|_2 \le \eta L \Big( \| E_M^{(t)} \|_F + \| E_\mu^{(t)} \|_2 \Big) = \eta L e_t.
\end{equation}
By the triangle inequality, the error at the next step is bounded by:
\begin{equation}
e_{t+1} \le e_t + \left\| \frac{d'}{d}\Delta \tilde{M}_t - \Delta M_t \right\|_F + \left\| \Delta \tilde{\mu}_t - \Delta \mu_t \right\|_2 \le e_t + \eta L e_t = (1 + \eta L) e_t.
\end{equation}
Applying this discrete recurrence relation recursively from step $0$ to $t$, we obtain:
\begin{equation}
e_t \le e_0 (1 + \eta L)^t .
\end{equation}
Using the standard exponential inequality $(1 + x)^t  \le \exp(xt)$ for all $x \ge 0$, we establish the upper bound for the error dynamics:
\begin{equation}
e_t \le e_0 \exp(\eta L t) = \left( \sqrt{\sum_{k=d'}^p \lambda_k^2} \right) \exp(\eta L t).
\end{equation}
This completes the proof.
\end{proof}

\subsection{NTK and Feature Learning}\label{app:ntk}
\begin{proposition}
\label{prop:ntk}
Under the same notations as Theorem \ref{theo:master}, the empirical Neural Tangent Kernel (NTK) between two data points $x$ and $x'$, defined as $\Theta(x, x') := \langle \nabla_W f_x(W), \nabla_W f_{x'}(W) \rangle$, is determined by the summary statistics $\mu$ and $M$:
\begin{equation}
\Theta(x, x') = d \cdot u(x)^\top  u(x') + 2 u(x)^\top  A(x') \mu + 2 \mu^\top  A(x) u(x') + 4 \Tr\big(A(x) A(x') M\big)
\end{equation}
where $u(x) := g(x) + 2B(x)\mu$.
\end{proposition}
\begin{proof}
By definition, the empirical NTK is the Frobenius inner product of the gradients with respect to the weight matrix $W$:
\begin{equation}
\Theta(x, x') = \Tr\big(\nabla_W f_x(W)^\top  \nabla_W f_{x'}(W)\big).
\end{equation}
From the proof of Theorem \ref{theo:master}, we know the sample-wise gradient can be written as:
\begin{equation}
\nabla_W f_x(W) = \big(g(x) + 2B(x)\mu\big) \mathbf{1}^\top  + 2A(x)W = u(x)\mathbf{1}^\top  + 2A(x)W.
\end{equation}
Substituting this into the NTK definition yields:
\begin{equation}
\Theta(x, x') = \Tr\Big( \big(\mathbf{1} u(x)^\top  + 2W^\top  A(x)\big) \big(u(x')\mathbf{1}^\top  + 2A(x')W\big) \Big).
\end{equation}
We expand this expression to obtain:
\begin{equation}
\Theta(x, x') = d \cdot u(x)^\top  u(x') + 2 u(x)^\top  A(x') \mu + 2 \mu^\top  A(x) u(x') + 4 \Tr\big(A(x) A(x') M\big).
\end{equation}
This demonstrates that the empirical NTK at any point in the training trajectory can be evaluated without knowing the individual parameter states in $W$, as it relies solely on the sufficient statistics $\mu$ and $M$.
\end{proof}

Assume the network width is $d$. Let the NQF coefficients scale with $d$ as $g(x) = \Theta(d^{-\alpha_g})$, $A(x) = \Theta(d^{-\alpha_A})$, and $B(x) = \Theta(d^{-\alpha_B})$. To prevent the outputs and gradients from diverging at initialization, we require $\alpha_A \ge 0$ and $\alpha_B \ge 1/2$. Assume standard variance scaling at initialization (e.g., $1/d$) such that the summary statistics scale as $\mu(0) = \mathcal{O}(1)$ and $M(0) = \mathcal{O}(1)$.

Based on Proposition \ref{prop:ntk}, the initial magnitude of the NTK scales as $\Theta(x, x') = \mathcal{O}\big(d^{1-2\alpha_g} + d^{1-2\alpha_B} + d^{-2\alpha_A}\big)$. We assume that the initial NTK is $\Theta(x, x') = \Theta(1)$ by setting $\alpha_g = 1/2$. 

\begin{proposition}[NTK Evolution and Feature Learning]
\label{prop:feature_learning}
For the gradient flow, the learning dynamics have two regimes:
\begin{enumerate}
    \item Lazy Training (Constant NTK): If $\alpha_A > 0$ and $\alpha_B > 1$, then $\lim_{d \to \infty} \dot{\Theta} = 0$. The NTK remains invariant while the loss decreases by $\mathcal{O}(1)$. The NQF effectively behaves as a linear model.
    \item Feature Learning (Evolving NTK): If $\alpha_A = 0$ or $\frac{1}{2}\leq\alpha_B \le 1$, then $\dot{\Theta} = \mathcal{O}(1)$. The NTK changes on the same timescale as the loss, allowing the model to learn data-dependent representations.
\end{enumerate}
\end{proposition}
\begin{proof}
By the chain rule under gradient flow $\dot{W} = -\nabla_W \mathcal{L}$, the dynamics of the function output are $\dot{f}_x = \langle \nabla_W f_x, \dot{W} \rangle = - \sum_{x' \in \mathcal{B}} \Theta(x, x') \ell'_{x'}$. 

From Proposition \ref{prop:ntk}, the NTK is $\Theta = d \cdot u^\top  u + 2u^\top  A \mu + 2\mu^\top  A u + 4 \Tr(A^2 M)$, where $u = g + 2B\mu$. With $\alpha_g = 1/2$, the leading term $d \cdot g^\top  g = \Theta(1)$, ensuring $\Theta(x, x') = \Theta(1)$. Consequently, $\dot{f}_x = \mathcal{O}(1)$, guaranteeing that the loss changes by an $\mathcal{O}(1)$ amount in an $\mathcal{O}(1)$ time.

We now evaluate the change in NTK during this $\mathcal{O}(1)$ learning process. The evolution of the sufficient statistics from Theorem \ref{theo:master} are:
\begin{equation}
\dot{\mu} = - (d \cdot v + H \mu), \quad \dot{M} = - (v \mu^\top  + \mu v^\top  + H M + M H)
\end{equation}
where $v = \ell'_x u(x)$ and $H = 2 \ell'_x A(x)$. For magnitude analysis, we drop the constant scalar $\ell'_x$. Taking the time derivative of $\Theta$ yields:
\begin{equation}
\dot{\Theta} = 2d \cdot u^\top  \dot{u} + 4 \dot{u}^\top  A \mu + 4 u^\top  A \dot{\mu} + 4 \Tr(A^2 \dot{M})
\end{equation}
Substituting $\dot{u} = 2B\dot{\mu}$ and $\dot{\mu} \approx - d \cdot u$ (since $H\mu$ is dominated), the two leading-order terms dictating the macroscopic change $\dot{\Theta}$ are:
\begin{enumerate}
    \item Term 1: $d \cdot u^\top  B \dot{\mu} \approx - d^2 u^\top  B u$. Since $\alpha_g = 1/2$ and $\alpha_B \ge 1/2$, we have $u = g + 2B\mu = \mathcal{O}(d^{-1/2})$. Thus, this term scales as $d^2 (d^{-1/2})^2 d^{-\alpha_B} = \mathcal{O}(d^{1-\alpha_B})$.
    \item Term 2: $-d \cdot u^\top  A u$. This term scales as $d (d^{-1/2})^2 d^{-\alpha_A} = \mathcal{O}(d^{-\alpha_A})$.
\end{enumerate}
Other terms, such as $\Tr(A^2 \dot{M})$, scale as $d^{-3\alpha_A}$ or lower and are dominated. Therefore, the overall drift in NTK is bounded by $\dot{\Theta} = \mathcal{O}(d^{1-\alpha_B} + d^{-\alpha_A})$.

For the NTK to remain constant (while the loss decreases by $\mathcal{O}(1)$), the derivative $\dot{\Theta}$ must vanish as $d \to \infty$. This requires $1-\alpha_B < 0 \implies \alpha_B > 1$ and $-\alpha_A < 0 \implies \alpha_A > 0$. If these conditions are violated, the NTK shifts by an $\mathcal{O}(1)$ magnitude during the $\mathcal{O}(1)$ training time.
\end{proof}

\subsection{A general exact solution}
\label{app:exact_solution}
The following theorem unifies a class of exact solutions in which the learning dynamics reduce to independent logistic equations. It contains Theorem \ref{theo:ortho_features} and
the simultaneously diagonalizable special case of Theorem \ref{theo:isotropic_samples}, as well as the exact solutions of linear networks in \cite{saxe2014exact}.

\begin{theorem}
\label{theo:compressed_projector_modes}
Define the empirical covariance
\begin{equation}
\mathcal K(X):=\frac{4}{m}\sum_{\mu=1}^m\langle A(x_\mu),X\rangle_F A(x_\mu),
\qquad Y:=\frac{4}{m}\sum_{\mu=1}^m y_\mu A(x_\mu).
\label{eq:K-and-Y}
\end{equation}
Suppose that there exist pairwise orthogonal symmetric projectors
$\{\Pi_a\}_{a=1}^r$ and constants $\gamma_a\in\mathbb R$ and $\kappa_a\geq 0$ such that
$\Pi_a\Pi_b=\delta_{ab}\Pi_a$, and for every $a,b\in\{1,\ldots,r\}$,
\begin{equation}
Y\Pi_a+\Pi_aY=2\gamma_a\Pi_a,
\label{eq:compressed-Y-condition}
\end{equation}
\begin{equation}
\mathcal K(\Pi_b)\Pi_a+\Pi_a\mathcal K(\Pi_b)=2\kappa_a\delta_{ab}\Pi_a.
\label{eq:compressed-K-condition}
\end{equation}
If $M(0)=\sum_{a=1}^r z_a(0)\Pi_a$ with $z_a(0)\geq0$, then the solution can be written as:
\begin{equation}
M(t)=\sum_{a=1}^r z_a(t)\Pi_a.
\label{eq:aligned-M-trajectory}
\end{equation}
If $\gamma_a\neq0$,
\begin{equation}
z_a(t)=\frac{\gamma_a z_a(0)}{\kappa_a z_a(0)+\left(\gamma_a-\kappa_a z_a(0)\right)e^{-2\gamma_a t}},
\label{eq:compressed-logistic-solution}
\end{equation}
whereas, if $\gamma_a=0$,
\begin{equation}
z_a(t)=\frac{z_a(0)}{1+2\kappa_a z_a(0)t}.
\label{eq:compressed-zero-mode-solution}
\end{equation}
\end{theorem}

\begin{proof}
The dynamics of $M(t)=W(t)W(t)^\top$ can be written as
\begin{equation}
\dot M
=
YM+MY-\mathcal K(M)M-M\mathcal K(M).
\label{eq:M-operator-dynamics}
\end{equation}
Consider a matrix in the subspace spanned by the projectors $M=\sum_{a=1}^r z_a\Pi_a$.
The linear part of  \eqref{eq:M-operator-dynamics} is
\begin{align}
YM+MY=\sum_{a=1}^rz_a\left(Y\Pi_a+\Pi_aY\right)=2\sum_{a=1}^r\gamma_a z_a\Pi_a,
\label{eq:linear-projector-part}
\end{align}
where we used  \eqref{eq:compressed-Y-condition}.
For the nonlinear part, we have
\begin{align}
\mathcal K(M)M+M\mathcal K(M)&=\sum_{a,b=1}^rz_a z_b\left[\mathcal K(\Pi_b)\Pi_a+\Pi_a\mathcal K(\Pi_b)\right]=2\sum_{a,b=1}^r\kappa_a z_az_b\delta_{ab}\Pi_a=2\sum_{a=1}^r\kappa_a z_a^2\Pi_a,
\label{eq:nonlinear-projector-part}
\end{align}
where  \eqref{eq:compressed-K-condition} was used in the second equality.

Substitution of
\eqref{eq:linear-projector-part} and
\eqref{eq:nonlinear-projector-part}
into  \eqref{eq:M-operator-dynamics} gives
\begin{equation}
\dot M=2\sum_{a=1}^r
z_a(\gamma_a-\kappa_a z_a)\Pi_a.
\label{eq:closed-projector-vector-field}
\end{equation}
Thus each mode evolves independently according to
\begin{equation}
\dot z_a(t)=2z_a(t)\left(\gamma_a-\kappa_a z_a(t)\right).
\label{eq:compressed-logistic-dynamics}
\end{equation}
For $\gamma_a\neq0$, separation of variables gives
\begin{equation}
\int \frac{dz_a}{z_a(\gamma_a-\kappa_a z_a)}=
2t+\mathrm{constant},
\end{equation}
from which  \eqref{eq:compressed-logistic-solution} follows.
When $\gamma_a=0$, the equation reduces to
\begin{equation}
\dot z_a=-2\kappa_a z_a^2,
\end{equation}
whose solution is
 \eqref{eq:compressed-zero-mode-solution}.
\end{proof}

\begin{corollary}
\label{cor:balanced-linear-network}
Consider an NQF with structure matrices
\begin{equation}
A(x_\mu)=\frac{1}{2}
\begin{bmatrix}
0 & S_\mu\\
S_\mu^\top & 0
\end{bmatrix},
\qquad
S_\mu\in\mathbb R^{n_{\mathrm{out}}\times n_{\mathrm{in}}}.
\label{eq:linear-network-A}
\end{equation}
Writing $W=\begin{bmatrix}
U\\
V
\end{bmatrix}$ with $U\in\mathbb R^{n_{\mathrm{out}}\times d}$ and $
V\in\mathbb R^{n_{\mathrm{in}}\times d}$.
Define
\begin{equation}
\mathcal C(Q):=\frac{1}{m}\sum_{\mu=1}^m
\langle S_\mu,Q\rangle_F S_\mu,
\qquad
T:=\frac{1}{m}\sum_{\mu=1}^m y_\mu S_\mu.
\label{eq:rectangular-covariance}
\end{equation}
Suppose that $T$ possesses singular triplets
$\{(\tau_a,u_a,v_a)\}_{a=1}^r$ satisfying
\begin{equation}
Tv_a=\tau_a u_a,
\qquad
T^\top u_a=\tau_a v_a,
\label{eq:task-singular-modes}
\end{equation}
where $\{u_a\}_{a=1}^r$ and $\{v_a\}_{a=1}^r$ are orthonormal.
Assume additionally that the corresponding rank-one singular modes are eigenmatrices of $\mathcal C$:
\begin{equation}
\mathcal C(u_av_a^\top)=c_a u_av_a^\top,
\qquad c_a\geq0.
\label{eq:rectangular-mode-alignment}
\end{equation}
Let $\{r_a\}_{a=1}^r\subseteq\mathbb R^d$ be orthonormal directions, and assume the balanced and aligned initialization
\begin{equation}
U(0)=\sum_{a=1}^r\sqrt{\rho_a(0)}\,u_ar_a^\top,
\qquad
V(0)=\sum_{a=1}^r\sqrt{\rho_a(0)}\,v_ar_a^\top,
\label{eq:balanced-linear-initialization}
\end{equation}
where $\rho_a(0)\geq0$.
Then the trajectory remains balanced and aligned:
\begin{equation}
U(t)=\sum_{a=1}^r\sqrt{\rho_a(t)}\,u_ar_a^\top,
\qquad V(t)=\sum_{a=1}^r\sqrt{\rho_a(t)}\,v_ar_a^\top,
\label{eq:balanced-linear-trajectory}
\end{equation}
where each singular mode evolves independently. If $\tau_a\neq0$, the solution is
\begin{equation}
\rho_a(t)=\frac{\tau_a\rho_a(0)}{c_a\rho_a(0)+\left(\tau_a-c_a\rho_a(0)\right)e^{-4\tau_at}},
\label{eq:linear-mode-solution}
\end{equation}
whereas, if $\tau_a=0$,
\begin{equation}
\rho_a(t)=\frac{\rho_a(0)}{1+4c_a\rho_a(0)t}.
\label{eq:linear-zero-mode-solution}
\end{equation}
\end{corollary}

\begin{proof}
Introduce the symmetric block-matrix notation
\begin{equation}
\mathcal B(Q):=\begin{bmatrix}
0&Q\\
Q^\top&0
\end{bmatrix}.
\end{equation}
Then
\begin{equation}
A(x_\mu)=\frac12\mathcal B(S_\mu).
\end{equation}
From  \eqref{eq:K-and-Y}, the target matrix $Y$ is
\begin{equation}
Y=\frac{4}{m}\sum_{\mu=1}^my_\mu A(x_\mu)=2\mathcal B(T).
\label{eq:linear-Y}
\end{equation}
For each singular mode, define
\begin{equation}
\psi_a:=\frac{1}{\sqrt{2}}
\begin{bmatrix}
u_a\\
v_a
\end{bmatrix},
\qquad\Pi_a:=\psi_a\psi_a^\top.
\label{eq:linear-mode-projector}
\end{equation}
The orthonormality of the singular vectors implies
\begin{equation}
\psi_a^\top\psi_b=\frac12
\left(u_a^\top u_b+v_a^\top v_b\right)=\delta_{ab},
\end{equation}
and hence $\Pi_a\Pi_b=\delta_{ab}\Pi_a$.
Using  \eqref{eq:task-singular-modes},
\begin{align}
Y\psi_a=\frac{2}{\sqrt{2}}
\begin{bmatrix}
Tv_a\\
T^\top u_a
\end{bmatrix}
=2\tau_a\psi_a.
\end{align}
It follows that
\begin{equation}
Y\Pi_a+\Pi_aY=4\tau_a\Pi_a.
\label{eq:linear-compressed-Y}
\end{equation}
Thus  \eqref{eq:compressed-Y-condition} is satisfied with $\gamma_a=2\tau_a$.

We next evaluate the empirical covariance on $\Pi_b$.
The off-diagonal block of $\Pi_b$ is
$\frac12u_bv_b^\top$, and therefore
\begin{align}
\mathcal K(\Pi_b)=\frac{4}{m}\sum_{\mu=1}^m
\langle A(x_\mu),\Pi_b\rangle_F A(x_\mu)=
\mathcal B\left(\mathcal C(u_bv_b^\top)\right)=
c_b\mathcal B(u_bv_b^\top).
\label{eq:linear-K-projector}
\end{align}
For any $a,b$,
\begin{align}
\mathcal B(u_bv_b^\top)\psi_a=
\frac{1}{\sqrt{2}}
\begin{bmatrix}
u_bv_b^\top v_a\\
v_bu_b^\top u_a
\end{bmatrix}
=\delta_{ab}\psi_a.
\label{eq:block-mode-action}
\end{align}
Consequently,
\begin{equation}
\mathcal K(\Pi_b)\Pi_a+\Pi_a\mathcal K(\Pi_b)=2c_a\delta_{ab}\Pi_a.
\label{eq:linear-compressed-K}
\end{equation}
Thus the condition \eqref{eq:compressed-K-condition} holds with $\kappa_a=c_a$.

Under the initialization
\eqref{eq:balanced-linear-initialization}, we have
\begin{align}
M(0)=W(0)W(0)^\top=\sum_{a=1}^r2\rho_a(0)\Pi_a.
\label{eq:linear-M-initialization}
\end{align}
Theorem \ref{theo:compressed_projector_modes} therefore applies
with $z_a(t)=2\rho_a(t), \gamma_a=2\tau_a, \kappa_a=c_a$.
Substitution into
\eqref{eq:compressed-logistic-dynamics} yields
\begin{equation}
\dot\rho_a=4\rho_a(\tau_a-c_a\rho_a).
\label{eq:linear-mode-logistic}
\end{equation}
\eqref{eq:linear-mode-solution} and
\eqref{eq:linear-zero-mode-solution} follow directly from
\eqref{eq:compressed-logistic-solution} and
\eqref{eq:compressed-zero-mode-solution}.
Finally, factorizing
\begin{equation}
M(t)=\sum_a2\rho_a(t)\Pi_a
\end{equation}
using the fixed orthonormal directions $\{r_a\}$ gives \eqref{eq:balanced-linear-trajectory}.
\end{proof}

\begin{remark}[Recovery of the solution in \cite{saxe2014exact}]
Consider the standard two-layer linear network
\begin{equation}
\widehat y(x)=UV^\top x
\end{equation}
trained on samples $\{(x_\mu,y_\mu)\}_{\mu=1}^m$. Define $\Sigma_x=\frac1m\sum_{\mu=1}^m x_\mu x_\mu^\top$ and $T=\frac1m\sum_{\mu=1}^m y_\mu x_\mu^\top.$
The corresponding operator is $\mathcal C(Q)=Q\Sigma_x.$
Hence the condition
 \eqref{eq:rectangular-mode-alignment} holds whenever
\begin{equation}
\Sigma_xv_a=c_av_a
\end{equation}
for the right singular vectors of $T$.
In particular, for whitened inputs, $\Sigma_x=I$,
we have $c_a=1$ for every mode. Then the solution is
\begin{equation}
\rho_a(t)=\frac{\tau_a\rho_a(0)}{\rho_a(0)+\left(\tau_a-\rho_a(0)\right)e^{-4\tau_at}}.
\end{equation}
Up to a rescaling of time, this is precisely the balanced solution of deep linear networks derived by \cite{saxe2014exact}.
\end{remark}

\section{Experimental Details and Additional Experiments}
\label{app:exp-details}
\begin{figure}[t]
    \centering
    \includegraphics[width=0.35\linewidth]{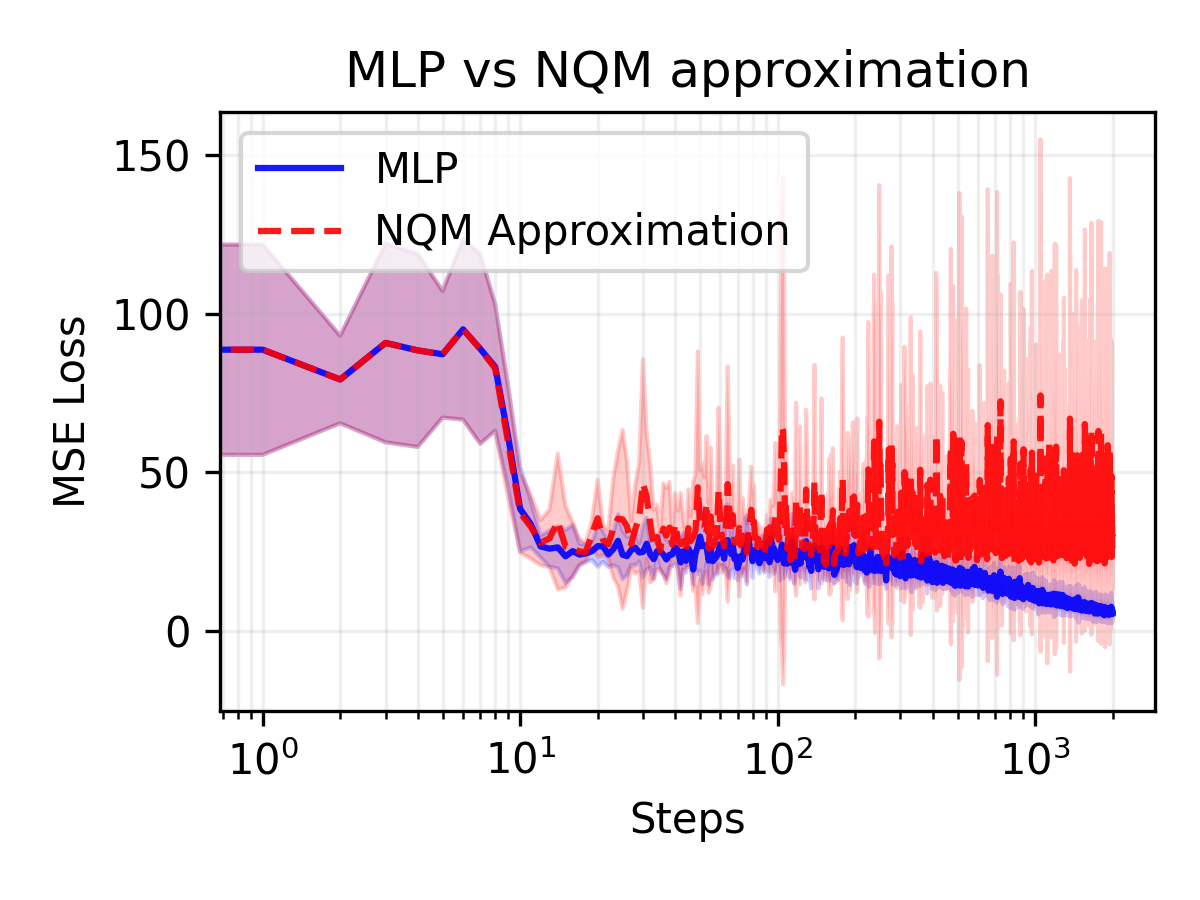}
    \includegraphics[width=0.35\linewidth]{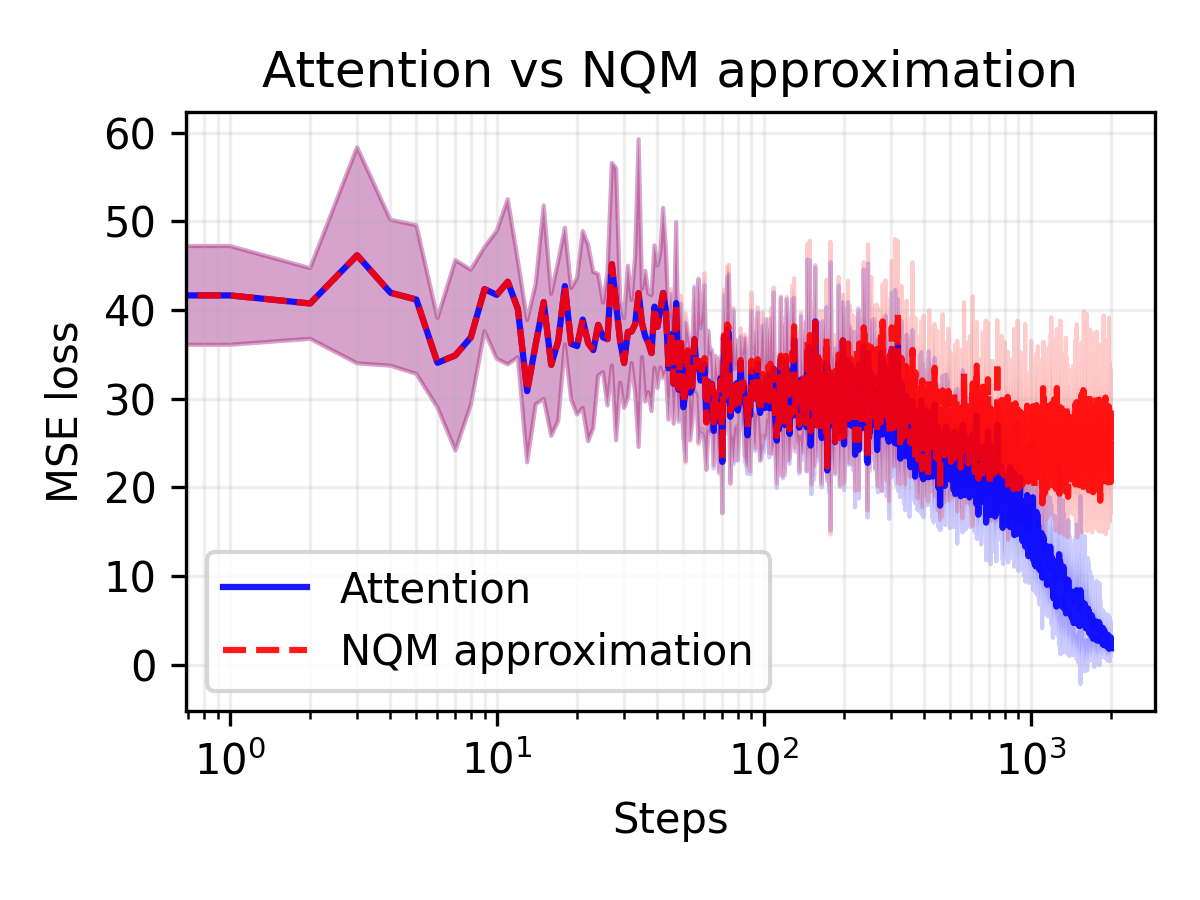}
    \caption{\textbf{Left}: Comparison between a two-layer MLP and its NQF approximation (Proposition \ref{prop:MLP}) under a teacher-student setting.  \textbf{Right}: Comparison between a query-key-only attention model and its NQF approximation (Proposition \ref{prop:MHA_variant}). Models are trained via online SGD (learning rate $0.05$, batch size $64$) over $2,000$ iterations. The results are averaged across $5$ independent runs, with shaded regions representing $\pm 1$ standard deviation.}
    \label{fig:NQF}
\end{figure}
\begin{figure}[t]
    \centering
    \includegraphics[width=0.9\linewidth]{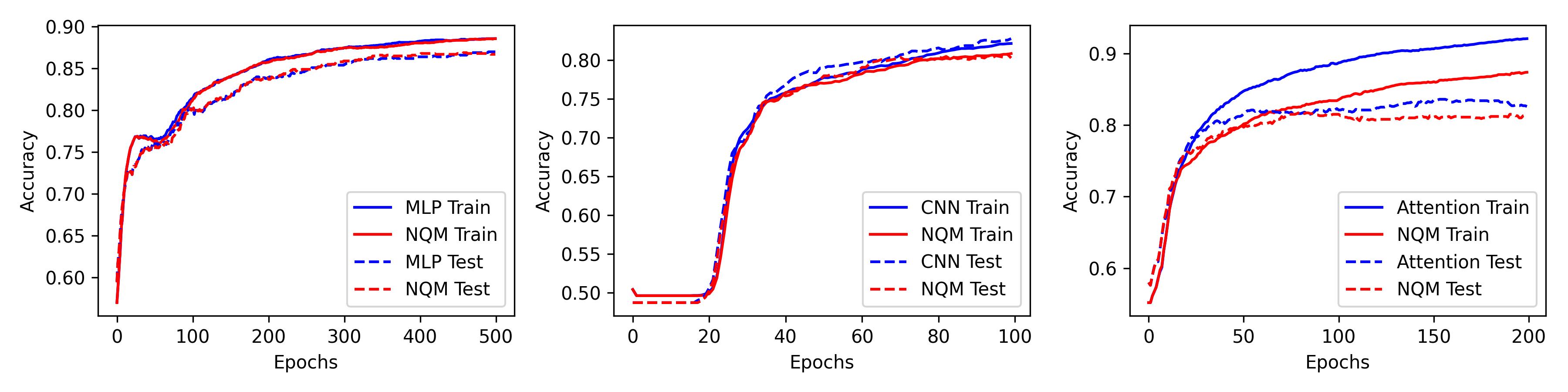}
    \caption{Comparison of training dynamics between a two-layer MLP (\textbf{Left}), a one-layer CNN (\textbf{Middle}), a one-layer self-attention (\textbf{Right}) and their corresponding NQFs on the MNIST odd vs. even classification task. In all cases, the NQFs (in red) achieve similar performance as the original model (in blue).}
    \label{fig:MNIST}
\end{figure}

\begin{figure}[t]
\centering
\includegraphics[width=1.0\linewidth]{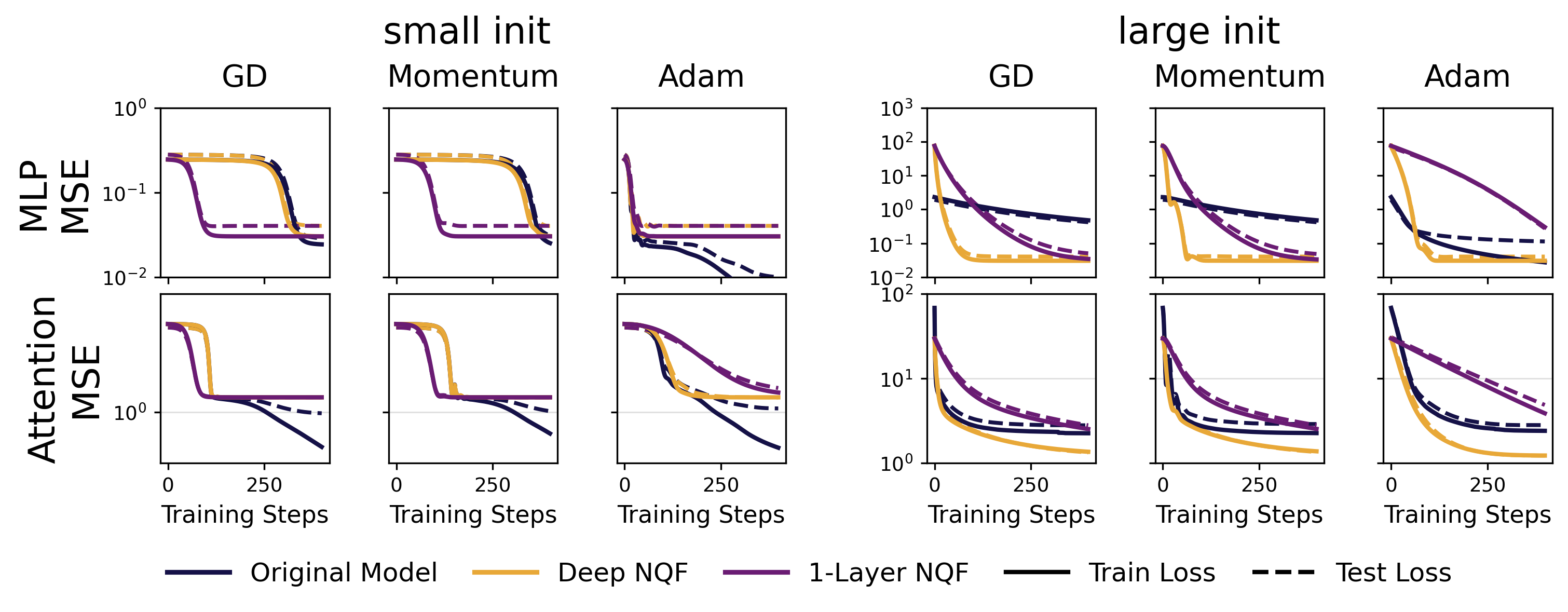}
\caption{Training dynamics of $4-$layer MLPs and $2-$layer attention models compared to their NQF approximations under different optimizers and initialization scales. The black line represents the original model. The orange line shows the 2-layer NQF approximations. The purple line represents an independently trained 1-layer NQF, initialized to match the original model at step zero. Solid lines represent the training loss and dashed lines represent the test loss. Under small initialization, the dynamics overlap, confirming the validity of deep NQF approximations.}
\label{fig:2layer_NQF}
\end{figure}

\paragraph{Figure \ref{fig:NQF_approx}} In Figure \ref{fig:NQF_approx}, all the models use output dimension $8$ and the MSE loss. The two-layer MLP has an input dimension $20$, hidden dimension $64$ and the $\tanh$ activation without any biases. The single-layer CNN takes a single-channel input of dimension $5\times 7$, employing $64$ kernels of dimension $4\times5$ (with stride 1 and no padding), followed by global sum-pooling. Both the single-head and the multi-head attention models have an input dimension $20$, a hidden dimension of $8$ and sequence length $16$. The multi-head model further has four heads with $d_v=20$. For all the experiments, we do full batch learning with $600$ samples. ‘‘Momentum’’ refers to GD with momentum $0.9$. Learning rates differ in each setting.

\paragraph{Additional Experiments on NQF Approximation}
On the left side of Figure \ref{fig:NQF}, we train a two-layer bias-free MLP with a $\tanh$ activation function in the teacher-student framework. The teacher weights are standard Gaussian, and we train two student models: an MLP of the same structure and its corresponding NQF approximation. Both models are optimized using online SGD with a learning rate of $0.05$ and a batch size of $64$ over $2,000$ steps. Both the teacher and the student MLP have the input dimension $20$, the hidden dimension $100$ and the output dimension $1$. Figure \ref{fig:NQF} shows that the trajectories of two student models match for the first $100$ steps and then they become different. 

For the right side of Figure \ref{fig:NQF}, we compare a query-key-only attention model (Proposition \ref{prop:MHA_variant}) with its NQF approximation and observe similarly that the NQF tracks the attention model accurately for the first $100$ steps. The attention model computes a scalar output from a query token $x \in \mathbb{R}^D$ and a context sequence $X \in \mathbb{R}^{D \times N}$, with a fixed Gaussian readout vector $c \in \mathbb{R}^D$. We set $D=16$, $N=10$, $d_k=8$, and use $H=4$ attention heads. Both the attention model and the NQF are optimized using online SGD with a learning rate of $0.02$ and a batch size of $64$ over $2,000$ steps. For both the MLP and the attention model, the teacher has standard Gaussian weights and student networks are initialized with identical Gaussian weights drawn from $\mathcal{N}(0, \epsilon^2)$ where $\epsilon = 10^{-3}$. At each step, the input data is sampled i.i.d. from a Gaussian distribution. 

In Figure \ref{fig:MNIST}, we construct a two-layer bias-free MLP equipped with tanh activations, a one-layer CNN and a one-layer self-attention and their NQF counterparts, where the NQFs are initialized with identical small Gaussian weights ($\mathcal{N}(0, \epsilon^2)$ with $\epsilon=0.01$) as the original models. All the models are trained with Adam (full batch), MSE loss and learning rate $0.001$. We construct a binary classification task: distinguishing odd from even digits on a subset of the MNIST dataset with $4000$ samples. As illustrated in Figure \ref{fig:MNIST}, the original models and their corresponding NQFs exhibit indistinguishable trajectories in both training loss and test accuracy. The two-layer MLP uses the hidden dimension $20$, $\tanh$ activation and no biases. The one-layer CNN uses $50$ channels with kernel size $15\times15$. The attention model uses a fixed readout vector, $10$ heads and $d_k=32$. The input $28\times 28$ image is regarded as $28$ tokens with embedding dimension $28$.

In Figure \ref{fig:2layer_NQF}, We formulate a standard teacher-student regression task. We examine two architectures: a 4-layer MLP and a 2-layer single-head attention. The inputs are drawn from a normal distribution, and the targets are generated by a corresponding teacher model with fixed Gaussian weights. According to Propositions \ref{prop:MLP} and \ref{prop:attention}, both models can be approximated by 2-layer NQFs. This is validated by Figure \ref{fig:2layer_NQF}, where under small initialization, the training dynamics of both models are well approximated by 2-layer NQFs but not 1-layer NQFs. The MLP has the input dimension $8$, the hidden dimension $16$ and the output dimension $1$. The initial weights are Gaussian with variance $0.05$ under small initialization and $0.5$ under large initialization. The 2-layer attention model has the input and hidden dimension $8$, input length $10$ and the output dimension $1$. The initial weights are Gaussian with variance $0.1$ under small initialization and $1$ under large initialization. The training set has $200$ samples and the test set has $1000$ samples. When the initialization scale and the model are fixed, three models (original one, 2-layer NQF and 1-layer NQF) share the same initialization and the same learning rate.

\paragraph{Figure \ref{fig:feature-descent}}
For the left and middle side of Figure \ref{fig:feature-descent}, we choose $m=5$, $r = \{1.0, 0.5, 0.25, 0.125, 0.0625\}$ and $C_{kk}=1$ for all $k$. The initialization scale is $\epsilon = 10^{-5}$. $\{A_\mu\}_{\mu=1}^m$ are chosen to be diagonal matrices of size $m\times m$. To satisfy the condition of Theorem \ref{theo:ortho_features}, we randomly generate a $m\times m$ orthogonal matrix and use the $\mu-$th column as the diagonal elements of $A_\mu$. We train the model using full-batch GD with learning rate $0.01$ to approximate the gradient flow and use the initialization $W(0) = \sqrt{\epsilon} I$ such that $z_k(0) = \epsilon$. Because $W$ remains diagonal, $z_k(t)=[W(t)W^\top (t)]_{kk}$ represents the eigenvalues of $WW^\top $.

For the right side of Figure \ref{fig:feature-descent}, we choose $\{A_\mu\}_{\mu=1}^{1000}$ such that each matrix only has one non-zero element. Specifically, we first set $r_k=k^{-\alpha_2}$, $V_k=k^{-\alpha_1}$ for $k=1,2,\cdots,1000$ and $\alpha_1=2$, $\alpha_2=0.8$. Then we can calculate $C_{kk} = r_k^2 / 8V_k$. We choose $A_{k,kk}=\sqrt{C_{kk}m/8}$ and other elements to be $0$. The labels are chosen to be $y_k = mr_k / (8A_{k,kk})$. The initialization is chosen to be $z_k(0)=10^{-6}k^{-1.2}$ for $k=1,2,\cdots,1000$.

\paragraph{Figure \ref{fig:sample-descent}} 
To validate Theorem \ref{theo:ortho_features}, we construct a synthetic dataset utilizing block-diagonal data matrices. We consider a system with $m=4$ samples and a parameter dimension of $p = m \times K$. $\{A(x_\mu)\}_{\mu=1}^m$ are configured as diagonal matrices with disjoint non-zero supports. Specifically, for any given sample $\mu \in \{1, \dots, m\}$, only the diagonal entries within the index block $[\mu K, \mu K + K - 1]$ are non-zero. This ensures that $A(x_\mu) A(x_\nu) = 0$ for all $\mu \neq \nu$, satisfying the condition required by Theorem \ref{theo:orthogonal_samples}. We choose $K=3$ and the three non-zero elements to be $1,1.5,2$. In this case \eqref{eq:implicit-solution} remains irreducible and lacks a closed-form solution. Consequently, the theoretical trajectories presented in Figure \ref{fig:sample-descent} are obtained by numerically integrating the associated ODE for $\{\xi_\mu(t)\}_{\mu=1}^4$ using the standard Runge-Kutta method.

\paragraph{Finite-initialization Effect}
On the right side of Figure \ref{fig:feature-descent} and  Figure \ref{fig:power-law-MLP}, one can notice that there is a slight difference between the empirical slope and the theoretical prediction. This is a finite-initialization effect.

Under the feature-wise descent setting in Section \ref{sec:saddle-to-saddle}, we can approximate the excess loss through
\begin{equation} 
\mathcal{E}(k) \approx \int_k^\infty V_x \, dx \propto k^{-(\alpha_1 - 1)}, 
\end{equation}
where $k$ as a function of $t$ is determined by the characteristic timescale
\begin{equation} t(k) \approx \frac{1}{\zeta_k} \ln \left( \frac{\zeta_k}{C_k \epsilon} \right) \label{eq:t_k} \end{equation}
with $\zeta_k / C_k \propto k^{\alpha_2 - \alpha_1}$. We can define the effective slope as $k_{\text{eff}} = -\frac{d \ln \mathcal{E} / dk}{d \ln t / dk}$, where the derivatives are $ \frac{d \ln \mathcal{E}}{d k} = -\frac{\alpha_1 - 1}{k} $ and 
\begin{equation} 
\frac{d \ln t}{d k} = \frac{\alpha_2}{k} + \frac{\alpha_2 - \alpha_1}{k \left[ \ln\left(\frac{\zeta_k}{C_k}\right) + \ln\left(\frac{1}{\epsilon}\right) \right]} = \frac{1}{k} \left( \alpha_2 - \frac{\alpha_1 - \alpha_2}{\ln \left( \frac{\zeta_k}{C_k \epsilon} \right)} \right). \end{equation}
Combining them, we obtain
\begin{equation} k_{\text{eff}}(k,\epsilon) = \frac{\alpha_1 - 1}{\alpha_2 - \frac{\alpha_1 - \alpha_2}{\ln \left( \frac{\zeta_k}{C_k \epsilon} \right)}}. \label{eq:k_eff} 
\end{equation}
For the experiments in Figure \ref{fig:feature-descent} (right) and Figure \ref{fig:power-law-MLP}, we use $\alpha_2 < \alpha_1$, and thus the effective slope is larger than the theoretical prediction for finite initialization ($\epsilon>0$). 

We can also write $k_{\text{eff}}$ as a function of $t$. From \eqref{eq:t_k} we have
\begin{equation} 
\zeta_k(t) t = \ln(1/\epsilon) + \ln\left(\frac{\zeta_k}{C_k}\right) \approx \ln(1/\epsilon) + \alpha \ln \zeta_k(t) + C, 
\end{equation}
where we denote $\alpha := \frac{\alpha_1 - \alpha_2}{\alpha_2}$ and $C$ is a constant independent of $\epsilon, t$. At the leading order of $t$ this gives
\begin{equation} 
\zeta_k(t) \cdot t \approx \ln(1/\epsilon) + \alpha \ln \ln(1/\epsilon) - \alpha \ln t + C. \end{equation}
Taking it into \eqref{eq:k_eff} we obtain
\begin{equation} k_{\text{eff}}(t) \approx \frac{\alpha_1 - 1}{\alpha_2} \left[ 1 + \frac{\alpha}{\ln(1/\epsilon) + \alpha \ln \ln(1/\epsilon) + C - \alpha \ln t} \right], \end{equation}
which increases as $t$ increases, as we observe in Figure \ref{fig:feature-descent} (right) and Figure \ref{fig:power-law-MLP}.

\paragraph{Figure \ref{fig:power-law-MLP}} For Figure \ref{fig:power-law-MLP}, we use $\beta = 1.5$, $P = 64$, $M = 256$ (thus $256$ data points) and run full batch GD with learning rate $0.02$ for $30000$ steps. The standard MLP has $256$ hidden units. All models are initialized with $\mathcal{N}(0, \epsilon^2)$ where $\epsilon = 10^{-4}$.

According to Proposition \ref{prop:MLP}, \eqref{eq:Fourier-MLP} corresponds to an NQF with
\begin{equation}
A_k(x) = \frac{s_k}{2\sqrt{P}} \begin{bmatrix} 0 & \psi_k(x) \\ \psi_k(x) & 0 \end{bmatrix}
\end{equation}
for the $k$-th hidden unit. Now we verify that the conditions in Section \ref{sec:saddle-to-saddle} (feature-wise descent) are satisfied. Since $A(x)$ is block-diagonal with elements $\{A_k(x)\}_{k=1}^P$ and $A_{k,\mu} A_{k,\nu}=A_{k,\nu} A_{k,\mu}$, we have $A_\mu A_\nu=A_\nu A_\mu$ and thus Assumption \ref{assume:commutation} is satisfied. 

Given the block-diagonal structure, the eigenvalues associated with the $k$-th unit for sample $\mu$ are $\lambda_{\mu, k} = \pm \frac{s_k}{2\sqrt{P}} \psi_k(x_\mu)$. Because $\psi_k(x)$ are Fourier modes uniformly sampled on the grid, they are orthogonal, leading to $\sum_{\mu} \lambda_{\mu, k} \lambda_{\mu, j} = 0$ for $k \neq j$, satisfying the orthogonal condition in Theorem \ref{theo:ortho_features}.

Finally, we map the experimental variables to that in Section \ref{sec:saddle-to-saddle}:
\begin{equation}
C_{kk} = \frac{8}{m} \sum_{\mu} \lambda_{\mu, k}^2 \propto s_k^2 = k^{-2\theta},\ r_k = \frac{8}{m} \sum_{\mu} \lambda_{\mu, k} y_\mu \propto s_k b_k = k^{-(\theta+\beta)}.
\end{equation}
Therefore, the $r_k$ decays with exponent $\alpha_2 = \theta + \beta$. $V_k$ is given by $V_k = \frac{r_k^2}{8C_{kk}} \propto \frac{k^{-2(\theta+\beta)}}{k^{-2\theta}} = k^{-2\beta}$, identifying $\alpha_1 = 2\beta$\footnote{Theorem \ref{theo:feature-scaling-law} assumes an initialization $z_k(0)=\epsilon c_k k^{-\beta_0}$ with $\beta_0\ge\alpha_1-\alpha_2$, whereas we initialize all weights from the same $\mathcal{N}(0,\epsilon^2)$. This does not lead to a significant difference because $P=64$ is very small.}. Plugging $\alpha_1$ and $\alpha_2$ into Section \ref{sec:saddle-to-saddle}, the excess loss $\mathcal{E}(t)$ decays asymptotically as:
\begin{equation}
\mathcal{E}(t) = \Theta\left( t^{-\frac{\alpha_1 - 1}{\alpha_2}} \right) = \Theta\left( t^{-\frac{2\beta - 1}{\theta + \beta}} \right)
\end{equation}
This provides the prediction for the scaling law in Figure \ref{fig:power-law-MLP}.

\paragraph{Power Law on One-hot Data}
We can also observe power laws from one-hot data. We choose each input sample $x_\mu \in \mathbb{R}^p$ (for $\mu \in \{1, \dots, m\}$) to be $x_\mu = \lambda_\mu e_\mu$, where $e_\mu$ is the basis vector (with $1$ at the $\mu$-th coordinate and $0$ elsewhere). We choose the scaling to be $\lambda_\mu \propto \mu^{-\left(\alpha_2 - \frac{\alpha_1}{2}\right)}$ and the label to be $y_\mu \propto \mu^{-\frac{\alpha_1}{2}}$.
Then we consider a quadratic network
\begin{equation}
f(x) = x^\top  W W^\top  x=\Tr[WW^\top A(x)],
\end{equation}
where $A(x)=xx^\top =\lambda_\mu^2 E_{\mu\mu}$, where $E_{\mu\mu}$ is the matrix with $1$ at the $(\mu, \mu)$ entry and $0$ elsewhere. Then we can find that this is the identical setting of Figure \ref{fig:feature-descent} (right).

% As a comparison, we use the same $\{x_\mu,y_\mu\}_{\mu=1}^m$ and train a two-layer MLP with tanh activation and $1000$ hidden units instead. The results are shown in Figure \ref{fig:MLP-powerlaw}, where we can see that the power law does not exist any more. This suggests that the power law we observe in Figure \ref{fig:power-law} is architecture specific.
\end{document}